\documentclass[twoside,11pt]{article}

\usepackage{jmlr2e}
\usepackage{shortcuts}
\usepackage{algorithm2e}

\usepackage{bold-extra}

\usepackage{bold-extra}
\usepackage{courier}
 \usepackage{booktabs}

 \usepackage{slantsc}
 \usepackage{lmodern}

\usepackage{scrextend}

\usepackage{multirow}

\usepackage{listings}
\usepackage{xcolor}

\newcommand\minus{%
  \setbox0=\hbox{-}%
  \vcenter{%
    \hrule width\wd0 height \the\fontdimen8\textfont3%
  }%
}

\newbox\flinebox
\newbox\slinebox
\newbox\mlinebox
\def\duplines{\setlength\parindent{0pt}
  \setbox\flinebox\lastbox
  \ifvoid\flinebox\relax
  \else
  \setbox\slinebox\hbox{\copy\flinebox}
  \setbox\mlinebox\hbox{\copy\flinebox}
  \unskip\unpenalty
  {\duplines}
  \color{black!95}\box\flinebox\vspace*{-2.8ex}
  \box\slinebox \fi
}

\newcommand*{\medcap}{\mathbin{\scalebox{1.2}{\ensuremath{\cap}}}}%

\jmlrheading{17}{2016}{1-42}{11/16}{--/--}{C\'{e}dric Malherbe and Nicolas Vayatis}

\ShortHeadings{A ranking approach to global optimization}{Malherbe and Vayatis}

\newtheorem{condition}{Condition}

\newcommand\blankfootnote[1]{
  \let\thefootnote\relax\footnotetext{#1}
  \let\thefootnote\svthefootnote
}

\begin{document}


\title{A Ranking Approach to Global Optimization \vspace{1em}}

\author{\name C\'{e}dric Malherbe \email malherbe@cmla.ens-cachan.fr
  \AND \name Nicolas Vayatis \email vayatis@cmla.ens-cachan.fr
  \AND \addr CMLA - Ecole Normale Sup\'{e}rieure de Cachan\\
  CNRS - Universit\'{e} Paris-Saclay\\
  94 235 Cachan cedex, France}

\editor{-}

\maketitle

\vspace{-0.7em}
\begin{abstract}
  In this paper, we consider the problem of maximizing an
  {unknown} and potentially {nonconvex} function $f$ over a compact and convex
  set $\X \subset \R^d$ using as few observations $f(x)$ as possible.
  We observe that the optimization of the function $f$ essentially relies
  on learning the induced bipartite ranking rule of $f$.
  Based on this idea, we relate global optimization to bipartite
  ranking which allows to address problems with high dimensional input space,
  as well as cases of functions with weak regularity properties.
  The paper introduces novel meta-algorithms for global optimization
  which rely on the choice of any bipartite ranking method.
  Theoretical properties are provided as well as convergence guarantees
  and equivalences between various optimization methods are obtained as a byproduct.
  Eventually, numerical evidence is provided to show that the main
  algorithm of the paper which adapts empirically to the underlying ranking structure
  is efficient in practice and displays competitive results
  with regards to the existing state-of-the-art global optimization methods
  over a wide range of usual benchmarks. 
\end{abstract}


\section{Introduction}

  In many applications such as complex system design or hyperparameter
  calibration for learning systems, the goal is to optimize some output value
  of a non-explicit function with as few evaluations as possible.
  Indeed, in such contexts, one has access to the function values
  only through numerical evaluations by simulation or cross-validation
  with significant computational cost. Moreover,
  the operational constraints generally impose a sequential exploration
  of the solution space with small samples.
  This generic problem of sequentially optimizing the output
  of an unknown and potentially { non-convex} function
  is often referred to as { global} optimization (\cite{pinter1991global}),
  black-box optimization (\cite{jones1998efficient})
  or derivative-free optimization (\cite{rios2013derivative}).
  In particular, there are several algorithms based on various heuristics
  which have been introduced in order to address complicated
  optimization problems with limited regularity assumptions, such as
  genetic algorithms, Bayesian methods, multi-start algorithms, etc.

  This paper follows the line of the approaches recently considered
  in the machine learning literature
  (\cite{bull2011convergence, munosmono, sergeyev2013introduction}).
  These approaches extend the seminal work on Lipschitz optimization of
  \cite{hansen1992global} and \cite{jones1993lipschitzian}
  and they led to significant relaxations of  the conditions
  required for convergence, {\it e.g.},~only
  the existence of a local {\it smoothness} around the optimum is required
  (\cite{munosmono, grill2015black}).
  More precisely, in the work of \cite{bull2011convergence} and \cite{munosmono},
  specific conditions have been identified to derive
  a finite-time analysis
  of the algorithms. However, these guarantees do not hold whenever the unknown
  function is not assumed to be locally
  smooth  around (one of) its optimum.
  In the present work, we propose to explore concepts
  from ranking theory based on overlaying
  estimated level sets (\cite{clemenccon2008ranking})
  in order to develop global optimization algorithms
  that do not rely on the smoothness of the function.
  The idea behind this approach is simple: even if
  the unknown function presents arbitrary large variations,
  most of the information required to identify its
  optimum may be contained in its induced ranking rule,
  {\it i.e.}~how the level sets of the function are included one in another.
  To exploit this idea, we introduce a novel optimization scheme where
  the complexity of the function is characterized by the underlying
  pairwise ranking which it defines. Our contribution is twofold:
  first, we introduce two novel  global optimization algorithms
  that learn the ranking rule induced by the unknown function
  with a sequential scheme, and second,
  we provide mathematical results in terms of statistical consistency
  and convergence to the optimum. Moreover, the algorithms proposed
  lead to efficient implementation and display good performance
  on the classical benchmarks for global
  optimization as shown at the end of the paper.

  This paper is structured as follows.
  In Section~\ref{sec:setup}, we introduce the framework and the main definitions.
  In Section \ref{sec:rankopt}, we introduce and analyze the {\sc RankOpt} algorithm
  which requires the knowledge of a ranking structure underlying
  the unknown function.
  In Section \ref{sec:adarank}, an adaptive version of the algorithm is presented.
  Companion results which establish the equivalence between
  learning algorithms and optimization procedures are discussed
  in Section \ref{sec:equivalence} as they support implementation choices.
  Finally, the adaptive version of the algorithm is compared
  to other global optimization algorithms
  in Section \ref{sec:xps}.
  All proofs are postponed to the Appendix section.

\vspace{-0.5em}
\section{Global optimization and ranking structure}
\label{sec:setup}

\subsection{Setup and notations}

  \textbf{Setup.} Let $\X \subset \R^d$ be a compact and convex set
  and let $f:\X \rightarrow \R$ be an unknown
  function which is only supposed to admit a global maximum 
  over its domain $\X$.
  The goal in global optimization consists in finding some point 
  \[
    x^{\star} \in \underset{x \in \X}{\arg \max}~ f(x) 
  \]
  with a minimal amount of function evaluations.
  More precisely, we wish to set up a sequential procedure which
  starts by evaluating the function at an initial point $X_1\in \X$
  and then selects at each step $t\geq 1$ an evaluation point $X_{t+1}\in\X$
  which depends on the previous evaluations $\{(X_i, f(X_i)) \}_{i=1}^t$ and recieves
  the evaluation of the unknown function $f(X_{t+1})$ at this point.
  After $n$ iterations, we consider that the algorithm returns the argument
  of the highest evaluation observed so far: 
  \[
  X_{\hatin} \text{~~~where~~~} \hatin \in  \underset{i= 1 \ldots  n}{\arg \max} ~f(X_i). 
  \]
  The analysis provided in the paper considers that the number $n$ of
  evaluation points is not fixed and it is assumed that function
  evaluations are noiseless.\\

  \noindent \textbf{Notations.} For all $x=(x_1, \ldots, x_d ) \in \R^d$,
  we define the standard $\ell_2$-norm as
  $\norm{x}^2_2= \sum_{i=1}^d x_i^2$,
  we denote by $\inner{\cdot,\cdot}$ the corresponding inner product and
  we denote by $B(x,r)=\{x' \in \R^d: \norm{x-x'}_2 \leq r \}$
  the  $\ell_2$-ball centered in $x$ of radius $r\geq 0$ .
  For any bounded set $\X \subset \R^d$, we define
  its inner-radius as
  $\textrm{rad}(\X)=\max \{r>0: \exists x\in \X \textrm{~such that~} B(x,r)\subseteq \X \}$,
  its diameter as
  $\diam{\X}=\max_{(x,x')\in \X^2}\norm{x-x'}_2$
  and we denote by $\mu(\X)$ its volume
  where $\mu$ stands for the Lebesgue measure.
  We denote by $\mathcal{C}^0(\X, \R)$ the set of continuous functions
  defined on $\X$ taking values in $\R$, we denote by
  $\mathcal{P}_{k}(\X, \R)$ the set of (multivariate)
  polynomial functions of degree $k\geq 1$ defined on $\X$,
  and for any function $f: \X \to \R$,
  we denote by $\textrm{Im}(f)=\{f(x) :x \in \X \}$ its image.
  Finally,  we denote by
  $\mathcal{U}(\mathcal{A})$
  the uniform distribution over a bounded measurable domain $\mathcal{A}$,
  we denote by $\indic{\cdot}$ the indicator function taking values in $\{0,1 \}$
  and  we denote by  $\textrm{sgn}(\cdot)$ 
  the standard sign function defined on $\R$ and taking values in $\{-1, 0, 1 \}$.

\subsection{The ranking structure of a real-valued function}
\label{rankingoptimization}

  In this section, we introduce the ranking structure as a characterization of the complexity
  for a general real-valued function to be optimized.
  First, we observe that real-valued functions induce an order relation
  over the input space $\X$, and the underlying ordering induces a
  ranking rule which records pairwise comparisons between evaluation points.

  \begin{definition} {\sc (Induced ranking rule)}
  The ranking rule $r_f: \X \times \X \rightarrow \{-1, 0, 1\}$
  induced by a function $f:\X \rightarrow \R$
  is defined by:
    \begin{equation*}
      r_f(x,x')=
      \begin{cases}
	+1 & \ \ \   \text{if}\ \   \ f(x)>f(x') \\
	\ \ ~\!\! 0 & \   \ \ \text{if}\  \ \ f(x)=f(x')\\
	-1 & \ \ \   \text{if}\ \  \ f(x)<f(x')
      \end{cases}
    \end{equation*}
    for all $(x,x')\in\X^2$.
  \end{definition}

  \noindent The key argument of the paper is that the optimization
  of any weakly regular real-valued function only depends
  on the nested structure of its level sets.
  Hence there is an equivalence class of real-valued functions
  that share the same induced ranking rule as shown by the following proposition.

  \begin{proposition} {\sc (Ranking rule equivalence)}
  \label{prop:rankingequivalence}
    Let $h \in \mathcal{C}^{0}(\X, \R)$ be any continuous function.
    Then, a function $f: \X \rightarrow \R$ shares the
    same induced ranking rule with $h$ {\normalfont(}i.e.~$\forall (x,x')\in\X^2$,
    $r_f(x,x')=r_{h}(x,x')${\normalfont)}
    if and only if
    there exists a strictly increasing, but not necessarily continuous
    function
    $\psi : \R \rightarrow \R$ such that $h = \psi \circ f$.
  \end{proposition}
  \vspace{-0.7em}
  \begin{figure}[!h]
  \label{pics:sameranking}
    \begin{center}$
	\begin{array}{lcccr}
	\includegraphics[width=35mm, height=18mm]{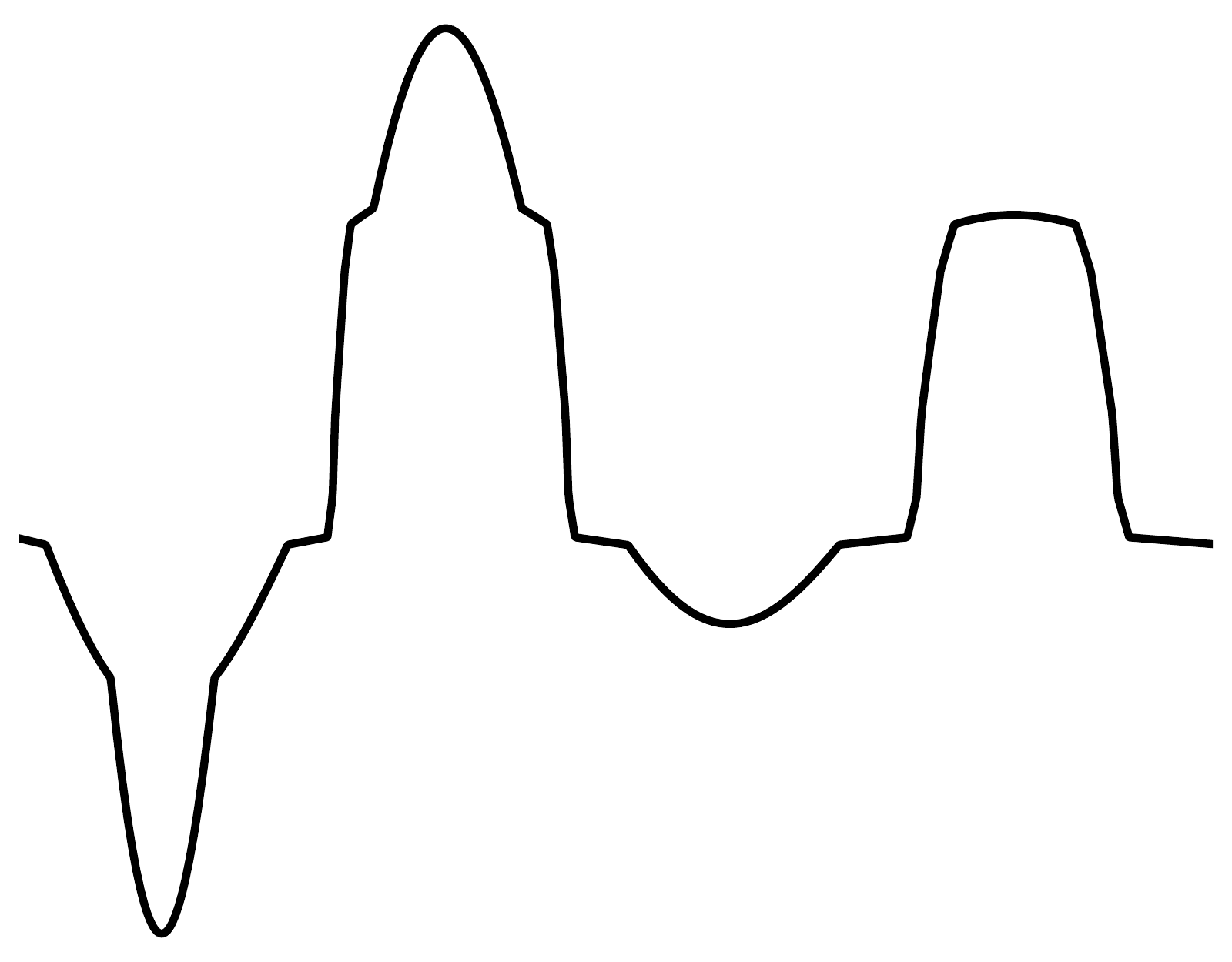}&~~~\!&
	\includegraphics[width=35mm, height=20mm]{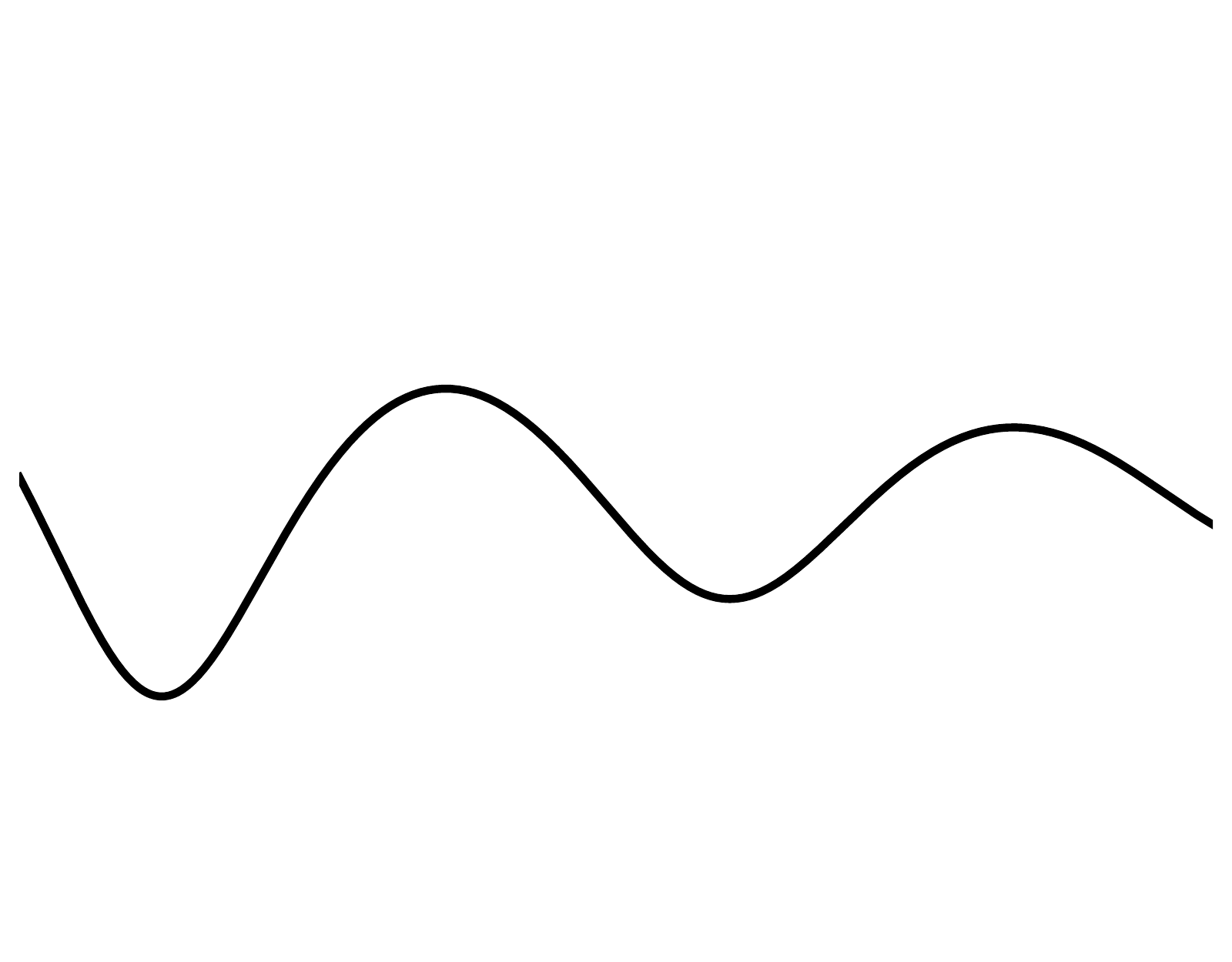}&~~~\!&
	\includegraphics[width=35mm, height=20mm]{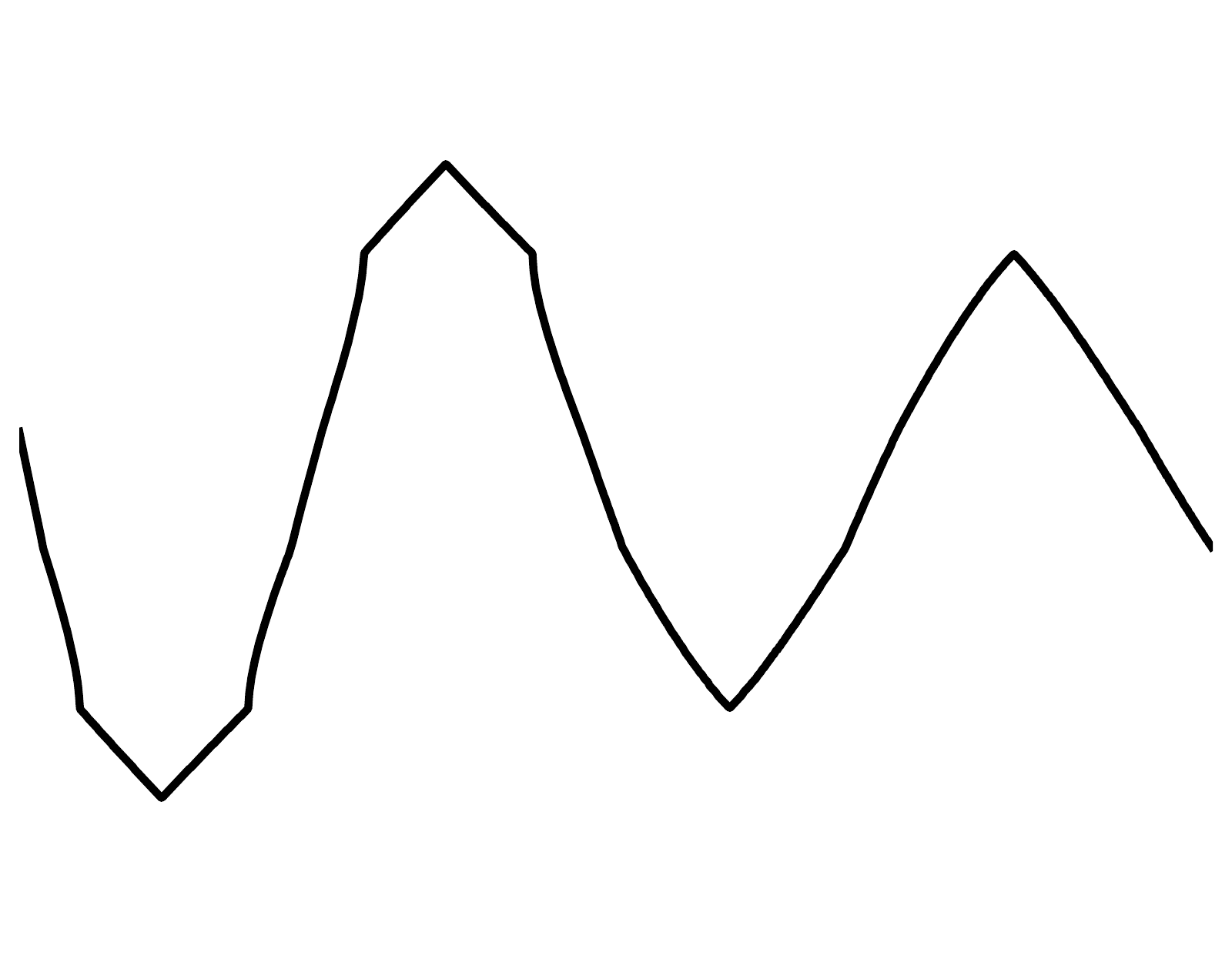}
	\end{array}$
    \end{center}
  \vspace{-0.8em}
  \caption{Three functions that share the same ranking rule}
  \end{figure}
  The previous proposition
  states that
  even if the unknown function
  $f$ admits noncontinuous or large  variations,
  up to a transformation $\psi$,
  there might exist a simpler function $h=\psi \circ f$
  that shares the same induced ranking rule.
  Figure \ref{pics:sameranking} gives an example of three functions that share
  the same ranking rule while they display highly different regularity properties.
  As a second example, we may consider the problem of maximizing the function
  $f(x)=1-1/\abs{\ln(x)}$
  if $x \neq 0$ and $1$ otherwise
  over $\X=[0,1/2]$.
  In this case, the unknown function $f$ is not \emph{'smooth'}
  around its unique global maximizer $x^{\star}=0$
  but shares the same induced ranking rule with $h(x)=-x$ over $\X$.\\

  A ranking structure is a collection
  of ranking rules. The approach developed in this paper consists 
  of seeing the ranking structure as a characterization of the complexity
  of the target function $f$ through the complexity of its induced ranking rule.
  We first introduce a very large class of ranking rules.

  \begin{definition} {\sc (Continuous Ranking Structure and Continuous Ranking Rules)}
    We say that a real-valued function $f: \X \to\R$ has a  continuous ranking rule
    if  $r_f \in \Rank_{\infty}$ where
    $\Rank_{\infty} := \{ r_{h} ~|~ h \in \mathcal{C}^{0}(\X, \R) \}$
    denotes the set of continuous ranking rules
    (i.e.~the set of ranking rules induced by continuous functions).
  \end{definition}
  In the continuation of this definition,
  we further introduce three examples of more stringent ranking structures.
  \begin{definition}{\sc (Polynomial Ranking Rules)}
    The set of polynomial ranking rules of degree $\normalfont{k} \ge 1$ is defined as
    \[
    \Rank_{\mathcal{P}_{\normalfont{k}}}
    := \{r_{h}: (x,x') \mapsto
    \normalfont{\text{sgn}}(h(x) -h(x')) ~| ~ h \in \mathcal{P}_{k}(\X, \R) \}.
    \]
  \end{definition}
  We point out that even a polynomial function of degree $k>1$
  may admit a lower degree polynomial ranking rule.
  For example, consider the polynomial function $f(x)=(x^2-3x+1)^{9}$.
  Since $f(x)= \psi  (x^2-3x)$  where $\psi : x \mapsto (x+1)^{9}$
  is a strictly increasing function,
  the ranking rule induced by  $f$ is a polynomial ranking rule of (at most) degree $2$.
  We may now introduce our second class of ranking structures
  which is an extension of the set of polynomial
  ranking rules.
  \begin{definition}{\sc (Sinusoidal Ranking Rules)}
    The set of sinusoidal ranking rules of degree $k \ge 1$ is defined as
    \[
    \Rank_{\mathcal{S}_k}
    := \{r_{h}:(x,x') \mapsto
    \normalfont{\text{sgn}} \left( 
    \left( h(\cos(2\pi x)) -h(\cos(2\pi x')) \right) \right)
    ~| ~ h \in \mathcal{P}_k(\X, \R) \}
    \]
    where the cosine function is vectorized, {\it i.e.~}$\forall x \in \R^d$,
    $cos(x) = \{cos(x_1), \dots, cos(x_d) \}$.
  \end{definition}
  The last class of ranking structures we introduce is a class of non-parametric ranking
  rules.
  \begin{definition}{\sc (Convex Ranking Rules)}
  The set of convex ranking rules of degree $k \geq 1$ is defined as
    \begin{multline*}
      \Rank_{\mathcal{C}_k} := \{r \in \Rank_{\infty}~
      \textrm{such that}~ \textrm{~for any~} x' \in \X,
      \textrm{~the set~}\\
      \{x \in \X: r(x,x')\geq0 \}\textrm{~is a union of~} k \textrm{~convex sets}\}.
    \end{multline*}
  \end{definition}
  It is easy to see that the ranking rule of a function $f$
  is a convex ranking rule of degree $k$
  if and only all the level sets of the function $f$
  are unions of at most $k$ convex sets.

\subsection{Identifiability and regularity}

  We now state two conditions that will be used in the theoretical analysis:
  the first condition is about the identifiability of the maximum of the 
  function
  and the second  is about the regularity of the function around its maximum.
  \begin{condition}{\sc (Identifiability)}
  \label{cond:id}
    The maximum of a function $f: \X \rightarrow \R$
    is said to be identifiable if for any $\varepsilon>0$ arbitrarily small,
      \[
      \mu(\{x \in \X: f(x) \geq \max_{x \in \X}f(x) - \varepsilon \})>0.
    \]
  \end{condition}
  This condition prevents the function from having a spike on its maximum.
  It will be useful to state  asymptotic results of the type
  $f(X_{\hatin}) \rightarrow \max_{x \in \X} f(x)$
  when $n \rightarrow + \infty$.

  \begin{condition}{\sc (Regularity of the level sets)}
  \label{cond:levelset}
    A function $f: \X \rightarrow \R$
    has $(c_{\alpha}, \alpha)$-regular level sets
    for some $c_{\alpha}>0$, $\alpha \ge 0$ if:
    \begin{enumerate}
      \item The global optimizer $x^{\star}\in\X$ is unique;
      \item For any $y \in$ \normalfont{Im}$(f)$,
	the iso-level set $f^{-1}(y)=\{x \in \X: f(x) = y \}$ satisfies
	\[
	\max_{x \in f^{-1}(y)} \norm{x^{\star}-x}_2 \leq
	c_{\alpha} \cdot \min_{x \in f^{-1}(y)} \norm{x^{\star}-x}^{1/(1+\alpha)}_2.
	\]
    \end{enumerate}
  \end{condition}
  Condition \ref{cond:levelset} guarantees that the points
  associated with high evaluations
  are close to the unique optimizer with respect to the Euclidean distance.
  Note however that
  for any iso-level set $f^{-1}(y)$ with finite distance
  to the optimum, the condition
  is satisfied with $\alpha=0$ and
  $c_{\alpha}=\diam{\X}/\min_{x \in f^{-1}(y)} \norm{x^{\star}-x}_2$.
  Thus, this condition concerns the local behavior of the level sets 
  when
  $\min_{x \in f^{-1}(y)}\norm{x^{\star}-x}_2 \rightarrow 0$.
  As an example, the iso-level sets of three simple functions satisfying
  the condition with different values of $\alpha$
  are shown in Figure \ref{pics:lvlset}.
  \begin{figure}[!h]
    \begin{center}$
      \begin{array}{lcccr}
      \includegraphics[width=20mm]{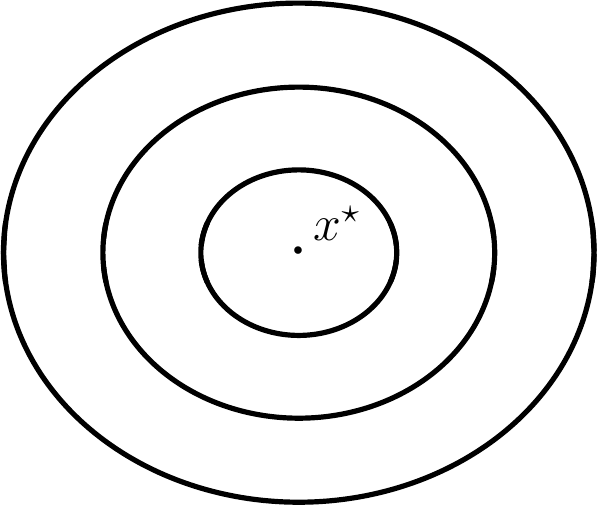}&~~~~~~&
      \includegraphics[width=20mm]{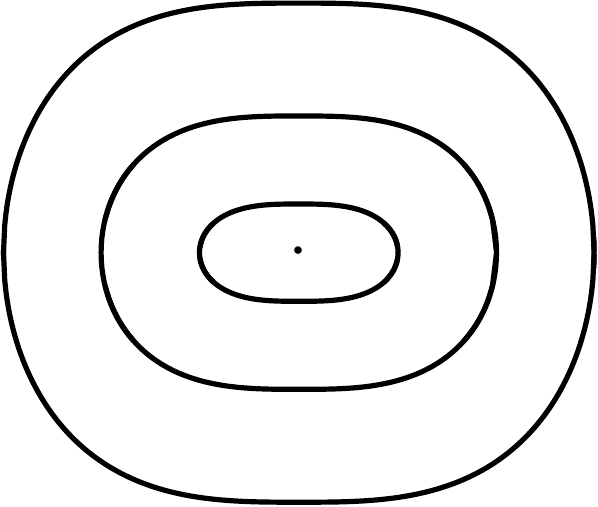}&~~~~~~&
      \includegraphics[width=20mm]{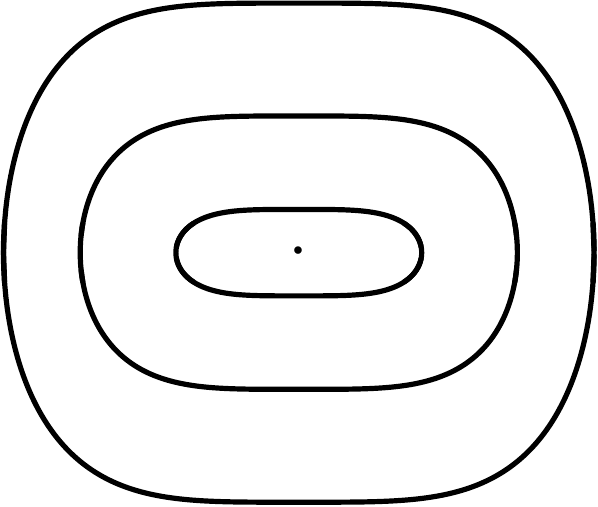}
      \end{array}$
    \end{center}
   \vspace{-0.4cm}
   \caption{Illustration of the regularity of the level sets
   on three simple functions. {\it Left:}
   $f(x_1,x_2)=-x_1^2 -1.4x_2^2$  where $\alpha=0$.
   {\it Middle:} $f(x_1,x_2)=\exp{(-\abs{x_1}^3 -1.4x_2^2)}$  where $\alpha=1/2$.
   {\it Right:} $f(x_1,x_2)=-x_1^4-1.4x_2^2$ where $\alpha=1$.}
  \label{pics:lvlset}
  \end{figure}

\section{Optimization with fixed ranking structure}
\label{sec:rankopt}
  In this section, we consider the problem of optimizing
  an unknown function $f$ given the knowledge that its induced ranking rule
  $r_f$ belongs to a given ranking structure $\Rank \subseteq \Rank_{\infty}$.

\subsection{The \textbf{\textsc{RankOpt}} algorithm}

  {\bf Definitions.} In order to properly set up the algorithm,
  we first introduce some key
  concepts that will be at the core of its strategy
  and used throughout the paper.
  We start with the definition of the empirical ranking loss.
  
  \begin{definition}
   {\sc (Empirical ranking loss)}
   The empirical ranking loss computed over a sample
   $(X_1, f(X_1)), \dots, (X_t, f(X_t))$
   of $t\geq 2$ function evaluations
   is defined for all $r: \X \times \X \to \{-1, 0, 1 \}$ by
  \begin{equation*}
  L_t(r) := \frac{2}{t(t-1)} \sum_{1 \leq i<j \leq t}
  \indic{ r(X_i,X_j)~ \! \neq ~ \! r_f(X_i,X_j) }
  \end{equation*}
  where $r_f(X_i,X_j)=\normalfont{\sgn{f(X_i) - f(X_j)}}$ for all $(i,j)\in\{1, \dots, t \}^2$.
  \end{definition}
  Based on this definition, 
  one might then 
  recover among a ranking structure $\Rank$ 
  the subset of ranking rules $r$ 
  which are consistent 
  with the ranking rule $r_f$ induced by the unknown function over a sample
  of function evaluations.

  \begin{definition}
  \sloppy
   {\sc (Active subset of consistent ranking rules)} 
   The active subset of a ranking structure $\Rank$
   which contains the ranking rules consistent with $r_f$ over 
   a sample $(X_1, f(X_1)), \dots, (X_t, f(X_t))$ 
   of $t\geq 2$ function evaluations  is defined by
   \[
    \Rank_t : = \{r \in \Rank : L_t(r) =0 \}
   \]
   where $L_t(\cdot)$ denotes the empirical ranking loss defined above.
  \end{definition}
  We may now introduce the optimization algorithm.\\
%

\newcommand{\nosemic}{\renewcommand{\@endalgocfline}{\relax}}
\newcommand{\dosemic}{\renewcommand{\@endalgocfline}{\algocf@endline}}
\newcommand{\pushline}{\Indp}
\newcommand{\popline}{\Indm\dosemic}
\let\oldnl\nl
\newcommand{\nonl}{\renewcommand{\nl}{\let\nl\oldnl}}
\SetNlSty{textbf}{}{.}

\RestyleAlgo{boxed}
\begin{figure}[b!]
  \begin{algorithm}[H]
  \vspace{0.5em}
  \textbf{1. Initialization:} Let $X_1 \sim \mathcal{U}(\X)$\\
  \nonl \pushline  Evaluate $f(X_1)$, $t \leftarrow 1$\\
  \nonl $\Rank_1 \leftarrow \Rank$, $\hatiun \leftarrow 1$\\
  \vspace{0.5em}
  \nl \Indm \textbf{2. Iterations:} Repeat while $t <n$: \\
  \pushline \nonl Let $X_{t+1} \sim \mathcal{U}(\X)$\\
  \nonl If  there exists $r \in \Rank_t$ such that
  $ r(X_{t+1},X_{\hatit}) \geq 0$~\textcolor{black!75}{\tt \{Decision rule\}}\\
  \pushline \nonl Evaluate $f(X_{t+1})$, $t \leftarrow t+1$\\
  \nonl $\mathcal{R}_t \leftarrow \{r \in \mathcal{R}: L_t(r)=0 \}$\\
  \nonl $\hatit \in \arg\max_{i=1 \ldots t}f(X_i)$ \\
  \vspace{0.5em}
  \Indm \Indm \textbf{3. Output:} Return $X_{\hatin}$
  \vspace{0.5em}
  \end{algorithm}
  \vspace{-0.5em}
  \caption{The {\sc RankOpt}$(n,f,\X,\Rank)$ algorithm}
  \label{fig:rankopt}
\end{figure}

    \noindent {\bf Algorithm description.}
  The input of the {\sc RankOpt} algorithm (displayed in Figure \ref{fig:rankopt})
  are a number $n$ of iterations,
  the unknown function $f$,  a compact and convex set $\X \subset \R^d$
  and a ranking structure $\Rank \subseteq \Rank_{\infty}$.
  At each iteration $t  <n $, a point $X_{t+1}$ is sampled uniformly over $\X$ and
  the algorithm decides, whether or not, to evaluate the function at this point.
  The decision rule involves the active subset $\Rank_t$
  which contains the ranking rules that are consistent with
  the ranking rule induced by $f$ over the points sampled so far.
  More precisely the decision rule operates as follows:
  if there does not exist  any ranking rule $r \in \Rank_t$
  which satisfies $r(X_{t+1}, X_{\hatit}) \geq 0$, then we know from
  the definition of $\Rank_t$ that
  $r_f(X_{t+1}, X_{\hatit})=-1$ which necessarily means that 
  $f(X_{t+1})< f(X_{\hatit})$.
  Thus, the algorithm never evaluates the function at a point
  that will not return certainly an evaluation at least equal to the highest evaluation
  $f(X_{\hatit})$ observed so far.

  \begin{remark}{\sc (Connection with active learning)}
    Although the problem considered in this paper is 
    very different,
    the \textsl{\textsc{RankOpt}} algorithm
    might be seen as an extension to ranking of
    the baseline
    {active learning}
    algorithm introduced in \cite{cohn1994improving} and further analyzed
    by \cite{hanneke2011rates}.
    However, the main difference with this algorithm
    lies in the fact that in active learning, one estimates a binary classifier
    $h: \X \rightarrow \{0, 1\}$
    where the goal in global optimization is to estimate the winner
    of a tournament deriving from the ranking rule
    $r_f: \X \times \X \rightarrow \{-1, 0, 1 \}$
    and not the ranking rule itself.
  \end{remark}

  \begin{remark}{\sc (Adaptation to noisy evaluations)}
  \label{rem:noisy}
  It is noteworthy that the proposed optimization scheme 
  could be extended to settings with noisy evaluations by slightly
  adapting the ideas developed in \cite{dasgupta2011two} and \cite{hanneke2011rates}.
  More precisely,
  a straightforward strategy would consist in using a relaxed
  version of the active subset $\Rank_{\delta, t} := \{r \in \Rank: L_t(r)
    \leq \min_{r \in \Rank} L_t(r) + UB_{\delta,t } \}$ where the term
    $UB_{\delta,t }$ comes out of some standard generalization bound on
    $\abs{L_t(r_f) - \min_{r \in \Rank}L_t(r)}$
    (see, {\it e.g.,} \cite{clemenccon2008ranking}).
  \end{remark}

  \begin{remark}
   {\sc (Computational aspects)} 
   Due to the theoretical nature and the genericity of the algorithm, 
   several questions remain to be addressed in order to derive a practical implementation.
   In particular, the crucial steps
   of (i) identifying the set of ranking rules
   which minimize the empirical ranking loss and (ii) 
   simulating
   the next evaluation points $X_{t+1}$ 
   with the rejection method 
   might not be trivial in practice.
   Nevertheless, we point out that, under specific conditions,
   a complete implementation of the algorithm can be proposed
   (see Section \ref{sec:equivalence} for further discussions on these aspects).
%
  \end{remark}

\subsection{Convergence analysis}

  We state here some convergence properties of the {\sc RankOpt} algorithm.
  The results are stated in a probabilistic framework.
  Recall however that the source of randomness comes from
  the random variables generated by the algorithm and not from the evaluations
  which are assumed noiseless.
  We start by casting an intermediate result that will be important in order
  to formulate the consistency property of the algorithm and the upper bound
  on the convergence rate.

  \begin{proposition}
  \label{th:fasterprs}
    Let $\X \subset \R^d$ be any compact and convex set
    with non-empty interior,
    let $\Rank$ be any continuous ranking structure 
    and let $f: \X \to \R$ be any function such that $r_f \in \Rank$.
    Then, for any $n \in \mathbb{N}^{\star}$ and 
    all $ y \in \R$, we have that
    \[
    \P( f(X_{\hatin}) \geq y  )
      \geq \P( ~\! \!  \textstyle{\max_{i = 1 \ldots n} f(X'_i)}  \geq y )
    \]
    where $X_{\hatin}$ denotes the output of
    the 
    \textsl{\textsc{RankOpt}}$(n, f, \X, \Rank)$  algorithm and
    $\{ X'_i \}_{i=1}^n$ is a sequence of $n$ independent random variables
    uniformly distributed over $\X$.
  \end{proposition}
One can then easily derive the next asymptotic result
by combining Proposition \ref{th:fasterprs} with the identifiability condition.

  \begin{corollary} {\sc (Consistency)}
  \label{coro:consistencyrank}
  Consider the same assumptions as in Proposition \ref{th:fasterprs}.
    Then, under Condition \ref{cond:id},
    we have that
    \[
      f(X_{\hatin}) ~\! \! \overset{\mathbb{P}}{\longrightarrow} ~\! \! \! \max_{x \in \X}f(x).
    \]
  \end{corollary}
  
  Now we focus on the nonasymptotic performance of the algorithm.
  The next result provides our first finite-sample bound on the distance
  between the exact solution and the estimate
  provided by {\sc RankOpt}.

  \begin{theorem} {\sc (Upper bound)}
  \label{th:upperbound}
    Suppose that the assumptions of Proposition \ref{th:fasterprs}
    hold true.
    Then, under Condition \ref{cond:levelset},
    for any $n \in \mathbb{N}^{\star}$ and $\delta \in (0,1)$,
    we have with probability at least $1-\delta$,
    \begin{equation*}
      \norm{x^{\star} - X_{\hatin}}_2 \leq
      C_1 \cdot
      \left(  \frac{\ln(1/\delta)}{n}\right)^{\frac{1}{d(1+\alpha)^2}}
    \end{equation*}
    where 
    $C_{1}= c_{\alpha}^{(2+\alpha)/(1+\alpha)}
    \diam{\X} ^{ 1/ (1+\alpha)^{2}}$.
  \end{theorem}
  More surprisingly, a lower bound can also be derived
  by connecting the {\sc RankOpt} algorithm
  to a
  theoretical algorithm defined below
  which uses the knowledge of the level
  sets of the unknown function.
  \begin{definition}
  \label{def:PAS}
   {\sc (Pure Adaptive Search}\normalfont{, from \cite{zabinsky1992pure}}{\sc)}
   {\it We say that a sequence $\{X^{\star}_i \}_{i=1}^n$
   is distributed as a Pure Adaptive Search indexed by $f$ over $\X$
   if it follows the Markov process defined by:
    \begin{equation*}
      \begin{cases}
	  X^{\star}_1 \sim \mathcal{U}(\X) \\
	  X^{\star}_{t+1}|~ \! X^{\star}_t \sim \mathcal{U}(\X^{\star}_t)
	  \textrm{~~~~~~~} \forall t \in \{1 \ldots n-1 \}
	\end{cases}
    \end{equation*}
    where at each step $t\geq 1$ the next evaluation point $X^{\star}_{t+1}$
    is sampled  uniformly over the level set of the previous evaluation
    $\X^{\star}_t := \{x \in \X: f(x) \geq f(X^{\star}_{t}) \}$.}
  \end{definition}
  Precisely, the next result shows 
  that the value of the highest evaluation observed by a  {\sc Pure Adaptive Search}
  is superior or equal, in the usual stochastic ordering sense, to
  the one 
  observed by the {\sc RankOpt} algorithm tuned with the same
  number of function evaluations.
%
  \begin{proposition}
  \label{prop:slowerpas}
    Consider the same assumptions as in  Proposition \ref{th:fasterprs}.
    Then, for any $n\in \mathbb{N}^{\star}$ and all $y \in \R$, we have that
    \[
    \P( f(X_{\hatin}) \geq y  )
    \leq   \P( f(X^{\star}_{n}) \geq y )
    \]
    where $X_{\hatin}$ denotes the output of 
    the \textsl{\textsc{RankOpt}} algorithm
    after $n$ iterations and
    $\{X^{\star}_i \}_{i=1}^n$
    is a sequence of $n$ evaluation points
    distributed as a Pure Adaptive Search indexed by $f$ over $\X$.
  \end{proposition}
  As the performance of the algorithm can now
  be controlled by Proposition \ref{prop:slowerpas},
  it is then possible to establish a second finite-time bound on the distance
  between the exact solution and its approximation.

  \begin{theorem} {\sc (Lower bound)}
  \label{th:lowerbound}
    Suppose that the assumptions of Proposition \ref{th:fasterprs} hold true.
    Then, under Condition \ref{cond:levelset},
    for any $n \in \mathbb{N}^{\star}$ and $\delta \in (0,1)$,
    we have with probability at least $1-\delta$,
      \[
	C_{2} \cdot
	e^{ -\frac{(1+\alpha)^2}{d} \left( n+ \sqrt{2n\ln(1/\delta)} +\ln(1/\delta) \right) }
	\leq \norm{x^{\star}-X_{\hatin}}_2
      \]
      where $C_{2} = c_{\alpha}^{-(1+\alpha)(2+\alpha)}
      \rad{\X}^{(1+\alpha)^2}$.
  \end{theorem}
  Note that, in addition to the following remarks,
  a complete discussion on the theoretical results obtained in this paper
  can be found in the next section where an adaptive version of the algorithm 
  is presented.

  \begin{remark}
   {\sc (Tightness of the bounds)} We stress that
   the RankOpt algorithm does achieve,
   for specific choices of  ranking structures $\Rank$ and functions $f$,
   the polynomial and exponential rates exhibited in Theorems
   \ref{th:upperbound} and \ref{th:lowerbound}.
   Indeed, 
   noticing that the algorithm is equivalent to
   a pure random search when $\Rank$ is set to  $\Rank_{\infty}$,
   it can be easily shown by means of covering arguments that
   $\norm{X_{\hatin} - x^{\star}}_2 = \Omega_{\P}(n^{-1/d})$
   as soon as $f$ admits a unique maximum and $\Rank=\Rank_{\infty}$.
   Similarly, observing that the algorithm is equivalent to a pure adaptive search 
   when $\Rank$ is set to $\{ r_f\}$,
   one can also show
   by reproducing the same steps as in the proof of the lower bound 
   that $\norm{X_{\hatin} - x^{\star}}_2 = O_{\P} (e^{-n/d(1+\alpha+\epsilon)})$ 
   for any $\epsilon>0$
   when 
   $\Rank=\{ r_f\}$ and 
   $f$ has regular level sets but no flat parts
   {\normalfont (}{\it i.e.~}$\mu(\{x: f(x)=y\})=0$ for 
   all $y\in \normalfont{\text{Im}}(f)${\normalfont)}.
   As these bounds actually match the one
   reported above
   when $\alpha=0$,
   we then deduce that the RankOpt algorithm
   does indeed achieve
   the near-optimal exponential and polynomial rates
   of $\Theta^{*}_{\P}(e^{-n/d})$ and $\Theta_{\P}(n^{-1/d})$ 
   on any function $f$ with $(1,0)$-regular level sets and no flat parts
   when the ranking structure $\Rank$ is respectively set to $\{ r_f\}$ and $\Rank_{\infty}$.
  \end{remark}

  \begin{remark}
  \label{remark:gap}
   {\sc (Gap between the bounds)}
   As a direct consequence of the previous remark, 
   we underline that whereas the upper and lower bounds reported
   in Theorems \ref{th:upperbound} and \ref{th:lowerbound} display
   display very different convergence rates, 
   this gap can not be significantly reduced without imposing
   further conditions on both the function $f$ and
   the ranking structure $\Rank$ set as input.
   Indeed, observe that
   since the algorithm can achieve both the
   rates of $\Theta^{*}_{\P}(e^{-n/d})$ and $\Theta_{\P}(n^{-1/d})$
   on the same function $f$ depending on choice of the ranking structure $\Rank$,
   then
   the gap between any generic lower and upper bounds on the convergence rate
   will necessarily be at least of the order of $[e^{-n/d}, n^{-1/d}]$
   as long as it is only assumed 
   that $\Rank$ is a continuous ranking structure 
   and that $f$ has regular level sets.
   
  \end{remark}

  \begin{remark}
   {\sc (Choice of the ranking structure)}
   Finally,  we point out that
   some simple indications on how to choose in practice the ranking structure 
   $\Rank$ set as input can be deduced from the previous remarks.
   Recall indeed that since
   the algorithm achieves its best performance when $\Rank$ is set to $\{r_f \}$, 
   then an ideal but realistic ranking structure $\Rank$ should be (i) large enough so that 
   it actually has a chance to contain $r_f$ and (ii) 
   as small as possible in order to obtain similar performances as 
   when $\Rank = \{r_f \}$.\!
   Even though these indications seem to serve opposite goals,
   we will however see how to carefully combine  them in the next section
   in order to derive an adaptive version of the algorithm which automatically
   selects the ranking structure $\Rank$ among a series
   of ranking structures of different complexities.

   
  \end{remark}

\section{ Adaptive algorithm and stopping time analysis}
\label{sec:adarank}

  We consider here the problem of optimizing an unknown function $f$ when
  no information is available on its induced ranking rule $r_f$.

\subsection{The \textbf{\textsc{AdaRankOpt}} algorithm}

  The {\sc AdaRankOpt} algorithm (shown in Figure \ref{fig:adarank}) is an extension of
  the {\sc RankOpt} algorithm which involves model selection.
  We consider a parameter $p \in (0,1)$ and
  a nested sequence of ranking structures $\{\Rank_k \}_{k \in \mathbb{N}^{\star}}$
  satisfying
  \begin{equation}
    \label{eq:nested}
    \mathcal{R}_1 \subset \mathcal{R}_2 \subset \dots \subset \Rank_{\infty}.
  \end{equation}

\RestyleAlgo{boxed}
\begin{figure}[t!]
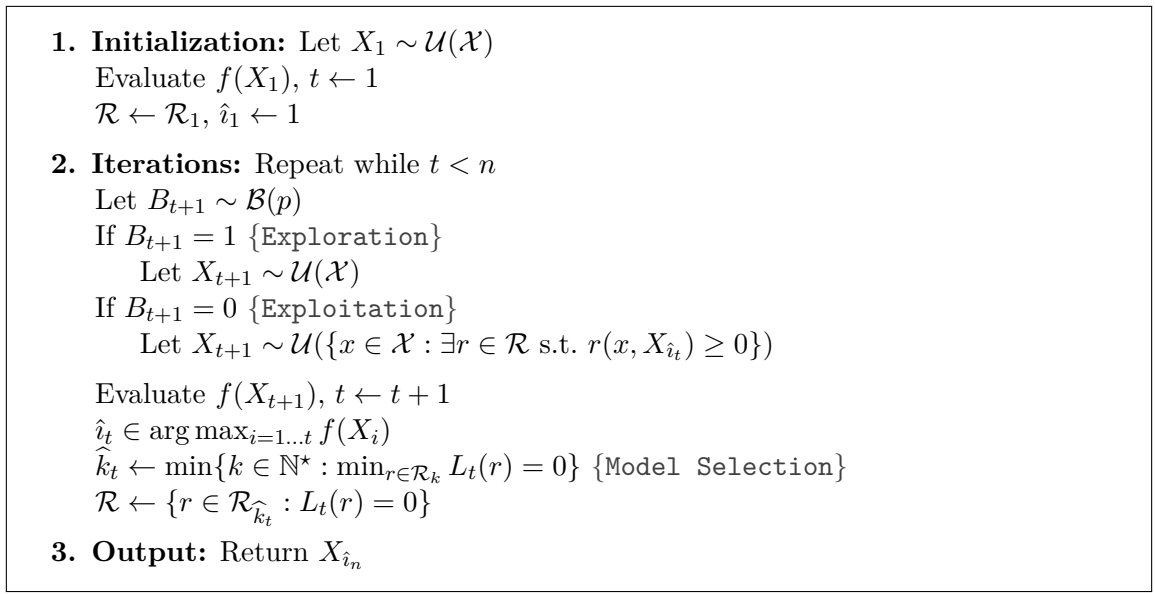

  \begin{algorithm}[H]
  \vspace{0.5em}
  {\bf 1. Initialization:} Let $X_1 \sim \mathcal{U}(\X)$\\
  \nonl \pushline  Evaluate $f(X_1)$, $t \leftarrow1$ \\
  \nonl $\Rank \leftarrow \Rank_1$, $\hatiun \leftarrow 1$\\
  \vspace{0.5em}
  \nl \Indm {\bf 2. Iterations:} Repeat while $t <n$ \\
  \pushline \nonl Let $B_{t+1} \sim \mathcal{B}(p)$\\
  \nonl If $B_{t+1}=1$  \textcolor{black!75}{\tt \{Exploration\}} \\
  \pushline \nonl Let $X_{t+1}\sim\mathcal{U}(\X)$ \\
  \nonl \Indm If $B_{t+1}=0$ \textcolor{black!75}{\tt \{Exploitation\}} \\
  \pushline \nonl Let $X_{t+1} \sim \mathcal{U}(\{x \in \X:
  \exists r \in \Rank \textrm{~s.t.~} r(x,X_{\hatit}) \geq 0\} )$\\
  \vspace{0.5em}
  \Indm Evaluate $f(X_{t+1})$, $t \leftarrow t+1$ \\
   $\hatit \in \arg\max_{i=1 \ldots t}f(X_i)$ \\
   $\widehat{k}_t \leftarrow \min \{k \in \mathbb{N}^{\star}:
   \min_{r \in \Rank_k}L_t(r)=0 \} $ \textcolor{black!75}{\tt \{Model Selection\}}\\
   $\Rank \leftarrow \{r \in \Rank_{\widehat{k}_t}: L_t(r)=0 \}$\\
  \vspace{0.5em}
  \Indm {\bf 3. Output:} Return $X_{\hatin}$
  \vspace{0.5em}
  \end{algorithm}
  \vspace{-0.5em}
  \caption{The
	  $\textsc{AdaRankOpt}(n,f, \X, p, \{\mathcal{R}_k \}_{k \in \mathbb{N}^{\star}})$
	  algorithm}
\label{fig:adarank}
\end{figure}

  \noindent The algorithm is initialized by
  evaluating the function at a point $X_1$ uniformly distributed over $\X$
  and by considering
  the smallest ranking structure $\Rank_1$ of the sequence.
  At each iteration $t <n$, a  Bernoulli random variable $B_{t+1}$
  of parameter $p$ is sampled.
  If $B_{t+1}=1$, the algorithm {\em explores} the space by evaluating
  the function at a point uniformly sampled over $\X$.
  If $B_{t+1}=0$, the algorithm {\em exploits} the previous evaluations
  by making an iteration of the {\sc RankOpt} algorithm
  with the smallest ranking structure $\Rank_{\hat{k}_t}$ of the sequence
  that probably contains the true ranking $r_f$.
  Once a new evaluation $f(X_{t+1})$ has been made,
  the index 
  $
   \widehat{k}_{t} := \min\{ k \in \mathbb{N}^{\star} : \min_{r \in \Rank_k} L_t(r)=0 \}
  $
  of the smallest ranking structure of the sequence 
  $\{\Rank_k \}_{k \in \mathbb{N}^{\star}}$
  which contains a ranking rule consistent with the sample is updated.
  Hence the parameter $p$ drives the trade-off between the exploitation phase and
  the exploration phase which prevents the algorithm from getting stuck in a local maximum.

\subsection{Theoretical properties of \textbf{\textsc{AdaRankOpt}}}

  We start by casting the consistency result for the algorithm.

  \begin{proposition} {\sc (Consistency)}
  \label{prop:cons_adarank}
    Fix any $p \in (0,1)$ and
    let $\{\mathcal{R}_k\}_{k \in \mathbb{N}^{\star}}$
    be any nested sequence of ranking structures.
    Then, under Condition \ref{cond:id}, 
    we have that
    \[
      f(X_{\hatin}) ~\! \! \overset{\mathbb{P}}{\longrightarrow}
      ~\! \! \! \max_{x \in \X}f(x)
    \]
    where $X_{\hatin}$ denotes the output of
    \textsl{\textsc{AdaRankOpt}}$(n, f, \X, p, \{\mathcal{R}_k \}_{k \in \mathbb{N}^{\star}})$.
  \end{proposition}
  The previous result reveals however that
  the adaptive version of the algorithm
  remains consistent over the same set of identifiable functions
  regardless 
  of its tuning
  ({\em e.g.}~the choice of the sequence of ranking structures and the value of $p$).
  We thus have to examine its nonasymptotic performance 
  in order to fully understand the impact of these parameters on its behavior.\\
  

  
  We begin the finite-time analysis
  by investigating
  the number of iterations required to identify a
  ranking structure which contains the ranking rule $r_f$
  induced by the unknown function.

  \begin{definition}{\sc (Stopping Time)}
    Let $k^{\star}=\min\{k \in \mathbb{N}^{\star}: r_f \in \Rank_k \}$
    be the index of the smallest ranking structure of the sequence which contains
    $r_f$ and let $\{ \widehat{k}_t \}_{t \in \mathbb{N}^{\star}}$
    be the sequence of random variables driving the model selection
    defined in the algorithm.
    We define the stopping time which corresponds to the number
    of iterations required to identify the index $k^{\star}$ as
    \[
      \tau_{k^{\star}} := \min \{ t \in \mathbb{N}^{\star}: \widehat{k}_t=k^{\star} \}.
    \]
  \end{definition}
  In order to bound $\tau_{k^{\star}}$,
  we need to control the complexity of the sequence of ranking structures
  $\{\Rank_k\}_{k \in \mathbb{N}^{\star}}$.
  Let us denote by $L(r)=\P(r(X,X') \neq  r_f(X,X')  )$ the
  true ranking loss where $(X,X')$ is a couple of independent random
  variables uniformly distributed over $\X$ and define the Rademacher average
  of a ranking structure $\Rank$ given $r_f$ as
  \[
    R_n(\Rank) :=  \sup_{r \in \Rank} \frac{1}{\lfloor n/2 \rfloor}
    \abs{ \sum_{i=1}^{\floor{n/2}} \epsilon_i \cdot
    \indic{
    r(X_i, X_{\left\lfloor n/2 \right\rfloor+i})
    ~\! \neq ~\!
    r_f(X_i, X_{\left\lfloor n/2 \right\rfloor+i}) }
    }
  \]
  where $\{ X_i\}_{i=1}^n$ are $n$ independent copies of $X \sim \mathcal{U}(\X)$
  and $\{ \epsilon_i \}^{\floor{n/2}}_{i=1}$ are $\floor{n/2}$
  independent Rademacher random variables
  ({\em i.e.,}~random symmetric sign variables), also independent
  of $\{X_i\}_{i=1}^n$.

  \begin{proposition} {\sc (Stopping Time Upper Bound)}
  \label{prop:model}
    Assume that  the index $ k^{\star}> 1$ is finite,
    assume that  $ \inf_{r \in \Rank_{k^{\star} \! -1}} L(r)>0 $ and
    assume that there exists a constant $K>0$ such that $\forall n \in \mathbb{N}^{\star}$,
    the Rademacher complexity of $\Rank_{k^{\star} \minus 1}$
    satisfies $\esp{R_n  (\Rank_{k^{\star} \minus1}) } \leq \sqrt{K/n}$.
    Then, for any $\delta \in (0,1)$, we have with probability at least $1-\delta$,
    \[
      \tau_{k^{\star}} \leq  \frac{10}{p} \cdot
      \left(  \frac{K+ \ln(2/\delta)}{\inf_{r \in \Rank_{k^{\star} \minus 1} }L(r)^2} \right).
    \]
  \end{proposition}
  In the situation described above, the smallest ranking structure
  of sequence which contains the true ranking rule
  can be identified in a finite number of iterations.
  One can then recover 
  an upper bound similar to the one of Theorem \ref{th:upperbound}
  where a ranking structure containing $r_f$
  is assumed to be known.

  \begin{theorem}{\sc (Upper Bound)}
  \label{th: upper_ada}
    Suppose that the assumptions of Proposition \ref{prop:model} hold ture.
    Then, under Condition \ref{cond:levelset},
    for any $\delta \in (0,1)$ and
    $n \in\mathbb{N}^{\star}$,
    we have with probability at least $1-\delta$,
    $$
    \norm{X_{\hatin} - x^{\star}  }_2 \leq
    C_{1} \cdot
    \left( \frac{11(K+\ln(4/\delta)) }{p\inf_{r \in \Rank_{k^{\star}\minus1}}L(r)^2 }
   \right)
    \cdot
    \left( \frac{\ln(2/\delta)}{ n } \right)^{\frac{1}{d(1+\alpha)^2}}
    $$
    \sloppy
    where 
    $C_{1}$ is the same constant as in Theorem \ref{th:upperbound}
    and $X_{\hatin}$ denotes the  output of
    \textsl{\textsc{AdaRankOpt}}$(n, f,\X,  p,  \{\mathcal{R}_k \}_{k \in \mathbb{N}^{\star}})$.
  \end{theorem}

  \noindent The following remarks provide some insights on the different 
  conditions and quantities
  involved in the  theorem.

  \begin{remark}\sloppy {\sc (On the complexity assumption)}
    As pointed out in \cite{clemencccon2011u} (see Remark 2 therein),
    standard VC-type arguments can be used in order to bound 
    $\mathbb{E}[R_n(\Rank_{k^{\star}\minus1})]$.
    More specifically, if the set of functions
    $\mathcal{F}_{k^{\star}\minus1}= \{(x,x') \in \X^2 \mapsto
    \indic{r(x,x')\neq r_f(x,x')}~\!|~\! r \in \Rank_{k^{\star}\minus1}) \}$
    is a VC major class with finite VC dimension $V$, then
    $\esp{R_n(\Rank_{k^{\star}\minus1}))} \leq c \cdot \sqrt{V/n}$
    for a universal constant $c>0$.
    This covers, in particular,
    the classes of
    polynomial and sinusoidal ranking rules
    of any degree $k^{\star}>1$.
  \end{remark}

  \begin{remark}
  \sloppy
   {\sc (On the infimum ranking loss)}
   In order to grasp the meaning of 
   the term
   $\inf_{r \in \Rank_{k^{\star}\minus1}} L(r)$, 
   observe first that since the function 
   $\rho : (r,r')
   \mapsto \P_{X,X'\sim \mathcal{U}(\X)}(r(X,X') \neq r'(X,X'))$
   defines a metric over the product space $\Rank_{\infty}\times \Rank_{\infty}$,
   then  the infimum ranking loss
   $\inf_{r \in \Rank_{k^{\star}\minus1}}L(r) 
   = \inf_{r \in \Rank_{k^{\star}\minus1}}\rho(r,r_f)$
   can  be  interpreted as a measure of the distance between the ranking rule $r_f$ and 
   the ranking structure $\Rank_{k^{\star}\minus1}$.
   As a consequence of this observation, we point out 
   that the condition $\inf_{r \in \Rank_{k^{\star}\minus1}}L(r)>0$ 
   can then be easily checked to be
   fulfilled for the sequences of polynomial and sinusoidal ranking rules
   whenever
   $r_f \in \Rank_{k^{\star}}$ for some $k^{\star}>1$ by combining 
   their parametric representation
   with the definition of the metric $\rho(\cdot, \cdot)$.

  \end{remark}

  \subsection{Comparison with previous works}
  \label{sec:comparison}
  
  Our interest here is to compare the theoretical results 
  obtained in this paper to existing results of the global optimization literature.
  We consider three different types of algorithms.\\

    \noindent {\bf DIRECT}  {\bf and}
    {\bf SOO} (\cite{jones1993lipschitzian} and \cite{munosmono}).
    These algorithms use a splitting technique of the search space and
    sequentially evaluate the function on subdivisions of the space
    that have recorded the highest evaluation among
    all the subdivisions of similar size.
    To the best of our knowledge, there is no finite-time analysis of the
    DIRECT algorithm (only the consistency was proven by \cite{finkel2004convergence}).
    However, \cite{munosmono} identified some local smoothness conditions allowing 
    to derive a finite-time analysis of the algorithms.
    Precisely, assuming
    there exists $ x^{\star}\in\X$,
    $\eta, c_1, c_2, \nu >0$ and $\alpha \geq 0$
    such that $\forall x \in B(x^{\star}, \eta)$, 
    $
     c_1 \norm{x^{\star} - x}^{\nu}
    \leq f(x^{\star}) - f(x) \leq c_2 \norm{x^{\star} - x}^{\nu/(1+\alpha)}
    $
    for some norm  $\norm{\cdot}$ ({\it e.g.,} $\ell_2$, $\ell_{\infty}$),
    the author reports for the SOO algorithm 
     a polynomial upper bound on the difference between the maximum and its estimation
    of $\max_{x \in \X}f(x) - f(X_{\hatin}) = {\it O}(n^{- \nu/\alpha d})$
    when $\alpha > 0$ 
    and an exponential decay of ${\it O}(e^{-c \nu \sqrt{n}})$
    for some $c>0$ when $\alpha = 0$.
    As a comparison, 
    we obtain for AdaRankOpt a polynomial bound of
    $\max_{x \in \X}f(x) - f(X_{\hatin})= {\it O}_{\P}(n^{- \nu/ d})$ for all $\alpha \geq 0$
    by assuming that both the conditions of Proposition \ref{prop:model} 
    and the local smoothness condition are fulfilled.
    Hence the bound we obtain turns out to be better 
    when $\alpha> 1$ and worse for $\alpha<1$,
    which is consistent with the fact that
    while the asymmetry in the smoothness of the function around its maximum (captured here 
    by $\alpha$) strongly impacts the performance of SOO in both ways, it does not impact
    AdaRankOpt which remains invariant to the variations of the local smoothness
    of the unknown function around its maximum.\\

    \noindent {\bf Evolution Strategies} (\cite{eigen1973ingo}).
    We now consider the class of $(\mu,\lambda)$-Evolution Strategies
    which use mutation, recombination, and selection
    in order to iteratively evolve the set of evaluation points.
    As far as we know, no consistency results or
    generic upper bounds have been proven for these algorithms.
    However, \cite{teytaud2008lower} were able to
    derive exponential lower bounds for several extensions of the $(\mu, \lambda)$-ES
    using the VC-dimension $V$ of the level sets of the unknown function. 
    Precisely, they  showed that if $V$ is finite, then
    $\norm{X_{\hatin}-x^{\star}}_2=\Omega_{\P}(e^{-c(V)n/d})$ 
    where $c(V)$ is a constant that depends on both the extension under consideration and $V$.
    Moreover \cite{auger2005convergence} also analyzed the convergence of
    the $(1,\lambda)$-SA-ES algorithm on 
    the simple sphere function $f(x) = -\norm{x}^2_2$ 
    and proved specific conditions on the parameters of the algorithm
    in order to ensure that $\ln(\norm{X_{\hatin}}_2)/n \overset{a.s.}{\longrightarrow} c$ 
    for some constant $c\in\R$.   
    However, 
    since the sign of the limit of
    $\ln(\norm{X_{\hatin}}_2)/n$  remains unknown,
    this result only proves the exponential convergence or divergence
    of the algorithm and
    can therefore not be cast into our framework.
    More specifically, we point out 
    all the results reported in those works can
    not be directly compared to the
    one obtained in this paper, as they are opposed by nature.
    Indeed, recall that while we were able to derive 
    (i) a generic upper bound for AdaRankOpt and 
    (ii) a lower bound for its nonadaptive version,
    they obtained on the contrary (i) generic lower bounds for various extensions 
    of the $(\mu, \lambda)$-ES and (ii) an asymptotic upper bound which might be
    only valid for a specific version of the algorithm in the case where $f$ is the
    sphere function.\\

    \sloppy
    \noindent {\bf Expected Improvement Strategy} (\cite{movckus1975bayesian}).  
    The last algorithm we consider is 
    a Bayesian optimization strategy which selects
    at each step $t\geq2$ an evaluation point $x_{t+1} \in \arg\max_{x \in \X} 
    \mathbb{E}_{f \sim \pi}[ \max(f(x) - \max_{i=1\dots t}f(x_i),0)| 
    \{(x_i, f(x_i)) \}_{i=1}^t ]$
    where $f$ is assumed to be drawn from a law $\pi$ set as input.
    \cite{vazquez2010convergence}  showed that when $\pi$ is a fixed Gaussian process 
    prior with a finite smoothness,
    the EI strategy converges on the maximum of any function $f$
    of the reproducing kernel Hilbert space $\mathcal{H}$
    canonically attached to $\pi$.
    Moreover
    \cite{bull2011convergence} went on to prove that an
    adaptive version of the EI algorithm they define
    could achieve a near-optimal polynomial bound of 
    $\max_{x \in \X}f(x) - f(X_{\hatin}) = {\it O}_{\P}(n^{-\nu/d})$
    for all $f \in \mathcal{H}$ when 
    $\pi$ is a prior of smoothness $\nu$.
    As a comparison, 
    considering that both the conditions of Proposition \ref{prop:model} are fulfilled
    and that $f \in \mathcal{H}$,
    we obtain for AdaRankOpt the exact same polynomial bound of
    $\max_{x \in \X}f(x) - f(X_{\hatin})  = 
    {\it O}_{\P}(n^{-\nu/d})$. 
    But we point out that this similarity 
    simply comes from the fact that the author also used
    a very similar---and potentially suboptimal---covering argument
    of the search space in order to prove their result.\\

  \noindent These comparisons suggest that although the upper bounds provided
  in this paper
  are generic, 
  they could certainly be
  improved in order to obtain the exponentially decreasing loss 
  exhibited in 
  Theorem \ref{th:lowerbound} and 
  observed in \cite{munosmono}.
  Nonetheless, as detailed in Remark \ref{remark:gap},
  such an analysis would require a refinement
  of the characterization of a real-valued function 
  with regards to a ranking structure and is therefore left as future work.

\section{Implementation and computational aspects}
\label{sec:equivalence}

  In this section, we discuss some technical aspects involved in the
  practical implementation of {\sc AdaRankOpt}.
  In particular, we provide some equivalences
  that can be used in order to implement the algorithm 
  for the classes of ranking structures introduced in Section \ref{sec:setup}
  without explicitly maintaining the active subset of
  consistent ranking rules.

\subsection{Notations}

  We collect here the specific notations used in this section.
  For any sample $\{(X_i, f(X_i)) \}_{i=1}^{t+1}$ of $t+1$ 
  function evaluations with distinct values ({\it i.e.}~any 
  sample such that $f(X_i) \neq f(X_j)$ for all $i\neq j$),
  we denote by {\small ${(1), (2),\dots, (t+1)}$} the indexes corresponding
  to the strictly increasing reordering:
  $
    f(X_{(1)}) < f(  X_{(2)} )  < \dots <f(X_{(t+1)}).
  $
  %
  For any dimension $d\geq 1$, we respectively denote by 
  $
   \vec{0}=(0,\dots,0) \in \R^{d}$ and by
   $\vec{1}=(1,\dots,1) \in \R^{d}$
  the zero and the unit vector of $\R^{d}$.
  The notation $x\succeq x'$ corresponds to the 
  component-wise inequality ({\it i.e.~}$
  x \succeq x' \in \R^d
  ~\Leftrightarrow~ \forall i \in \{1 \ldots d \},~x_i \geq x'_i
  $) and
  we denote by ConvHull$\{x_i\}_{i=1}^t$ the convex hull of any set
  $\{x_i\}_{i=1}^t$ of $t\geq 1$ points in $\R^d$.
  For any degree $k\geq 1$, 
  the function that maps $\R^d$
  into the corresponding polynomial feature space of degree $k$
  is denoted by
  $\Phi_{k}: \R^d \rightarrow \R^{\textrm{dim}(\Phi_{k})}$
   where
  $\textrm{dim}({\Phi_{k}})=\binom{k+d}{d}-1$. 
  For instance, in the case where $k=d=2$,
  we have that 
  $\Phi_2(x)= (x_1, x_2, x_1 x_2, x_1^2, x_2^2 ) \in \R^5$
  for all $x=(x_1,x_2)\in\R^2$.
  Finally, we denote by $\normalfont{\text{M}^{\Phi_k}_{t}}
  =[C_1 \mid \cdots \mid C_t]$ the
  $(\normalfont{\text{dim}}({\Phi_{\normalfont{k}}}),t)$-matrix
  with its $i$-th  column $C_i$ equal to
  $( \Phi_{\normalfont{k}}(X_{(i+1)}) 
  - \Phi_{\normalfont{k}}(X_{(i)}) )^\mathsf{T}$ and 
  we denote for all $i\leq t+1$
  by $~\normalfont{\text{M}}_i=[ C_1 \mid \cdots \mid C_i]$ the $(d,i)$-matrix
  where its $j$-th column $C_j$ is equal to $X_{(t+2-j)}^{\mathsf{T}}$.
  
\subsection{General ranking structures}

  Suppose now that we have collected a sample
  $\{(X_i, f(X_i)) \}_{i=1}^t$ of $t\geq2$ observations generated
  by AdaRankOpt tuned with any nested sequence of ranking structures
  $\{\Rank_k \}_{k \in \mathbb{N}^{\star}}$.
  We address here the questions of 
  {\bf (i)} sampling the next evaluation point $X_{t+1}$
  and {\bf (ii)} updating the index $\widehat{k}_{t+1}$ of the model selection
  once $f(X_{t+1})$ has been evaluated.\\
  
  \noindent {\bf (i)} We first consider the problem of 
  sampling the next evaluation point $X_{t+1} \sim \mathcal{U}(\X_t)$
   over the non-trivial subset
  $\X_t := \{x \in \X: \exists r \in \Rank_{\widehat{k}_t}$ $
  \textrm{~such that~}L_t(r)=0 \textrm{~and~}r(x,X_{\hatit}) \geq 0 \}$.
  To do so, we propose to use the rejection which consists in
  sampling $X' \sim \mathcal{U}(\X)$ until $X' \in \X_t$.
  We thus need to set up a procedure 
  that tests if any point $X' \in \X$
  belongs to $\X_t$.
  By definition of $\X_t$, we know that $X' \in \X_t$ if and only if
  there exists a ranking rule $r$ in $\Rank_{\widehat{k}_t}$
  which satisfies $L_t(r)=0$ and $r(X', X_{\hatit}) = 0$ {\it or} $1$.
  Therefore, we obtain by
  rewriting the previous statement in terms of minimal error
  that  $X' \in \X_t$ if and only if: 
    \begin{itemize}
    \item[-] either $\min_{r \in \Rank_{\widehat{k}_t}} L_{t+1}(r) = 0$  where
    the empirical ranking loss is taken over
    the sample $\{(X_i, f(X_i)) \}_{i=1}^t \cup  (X', f(X_{\hatit}) )$
    (case $r(X', X_{\hatit}) = 0$);
    \item[-] or $\min_{r \in \Rank_{\widehat{k}_t}} L_{t+1}(r) = 0$ where $L_{t+1}(\cdot)$
    is taken over the sample $\{(X_i, f(X_i)) \}_{i=1}^t \cup  (X', f(X_{\hatit})+c )$
    where $c>0$ is any positive constant
    (case $r(X',X_{\hatit}) = 1$).
  \end{itemize}
   Hence $X_{t+1}$ can be generated by sequentially sampling $X' \sim \mathcal{U}(\X)$
   until there exists a ranking rule $r$ in $\Rank_{\widehat{k}_t}$ that perfectly ranks 
   the initial set of $t$ observations 
   where we added
   a supplementary ghost evaluation
   $\{\{(X_i, f(X_i)) \}_{i=1}^{t} \cup (X',f(X_{\hatin})+c )\}$ for some $c\geq 0$.\\
 
  \noindent {\bf (ii)} We now consider the problem of updating the index $\widehat{k}_{t+1}$
  of the model selection once $f(X_{t+1})$ has been evaluated.
  Since $\{ \Rank_k \}_{k \in \mathbb{N}^{\star}}$ 
  forms, 
  by assumption,
  a nested sequence, 
  it necessarily follows that
  the sequence of indexes
  $\{\widehat{k}_t \}_{t \in \mathbb{N}^{\star}}$ is also increasing.
  One can thus write that
  $\textstyle{
  \widehat{k}_{t+1}= \widehat{k}_t  + 
  \min \{ i \in \mathbb{N}^{\star}: \min_{r \in \Rank_{\widehat{k}_t + i}}L_{t+1}(r) =0 \}
  }
  $
  where the empirical ranking loss 
  $L_{t+1}(\cdot)$
  is computed over the sample $\{( X_i, f(X_i) )\}_{i=1}^{t+1}$.
  Hence, the index $\widehat{k}_{t+1}$ can be updated by
  sequentially testing  if
  $\min_{r \in \Rank_{\widehat{k}_t+i}}L_{t+1}(r) =0$ for $i=0,1,2,\ldots$\\

   \noindent As shown above, both the steps {\bf(i)} and {\bf(ii)}
  can be done using  a single generic procedure that determines if
  $
   \min_{r \in \Rank_k} L_{t+1}(r) = 0
   $
  holds true for any ranking structure $\Rank_k$ of the sequence with $k\geq1$
  and where the empirical ranking loss $L_{t+1}(\cdot)$ is computed over
  any sample of $t+1$ function evaluations.
  In the next subsections, we provide some equivalences that can be used in order
  to design such a procedure for the classes of
  ranking structures introduced in Section \ref{sec:setup}.
  For simplicity,
  we will consider in the sequel that all the function evaluations of the sample
  have distinct values.

\subsection{Polynomial and sinusoidal ranking rules}

  We consider here the sequence of polynomial ranking rules
  $\{ \Rank_{\mathcal{P}_{k}} \}_{k \in \mathbb{N}^{\star}}$
  and we recall that $\Phi_k(\cdot)$ denotes the function that maps
  $\R^d$ into the corresponding polynomial feature space of degree $k$.
  However, we point out that the results stated below can easily be adapted for the 
  sequence of sinusoidal ranking rules by considering the adequate feature space.
  The first result we establish relates
  the existence of a consistent polynomial ranking rule to
  the linear separability of a sample-dependent set of points
  which belong to 
  the corresponding feature space.

  \begin{proposition}{\sc (Separability)}
  \label{prop:binary}
  \sloppy
    Let 
    $\{(X_i, f(X_i)) \}_{i=1}^{t+1}$ 
    be any sample of $t+1$  function evaluations with distinct values.
    Then, there exists a polynomial ranking rule of degree $\normalfont{k} \geq 1$
    that perfectly ranks the sample {\normalfont(}{\it i.e.}~$\min_{r \in
    \Rank_{\mathcal{P}_k}}L_{t+1}(r)=0${\normalfont)}
    if and only if there exists
    an axis $\omega \in \R^{\normalfont{\text{dim}}({\Phi_{k}})}$
    satisfying:
      \[
      \inner{\omega,\Phi_{\normalfont{k}}(X_{(i+1)}) 
      - \Phi_{\normalfont{k}}(X_{(i)})}>0,
      \textrm{~}\forall i \in \{1 \ldots t \}.
      \]
     where {\small$(1),(2),\dots(t+1)$} denote the indexes of the strictly increasing
     reordering of the sample.
  \end{proposition}
  Unfortunately, the equivalence exhibited in Proposition \ref{prop:binary}
  might not be always convenient in practice 
  since the computational cost of estimating such an axis $\omega \in \R^{\text{dim}(\Phi_k)}$ 
  can be prohibitive
  when the dimensionality of the feature space dim$(\Phi_k)$ is large.
  Nonetheless, as the previous result only makes the link with
  the {\em existence} of a separating axis,
  one can then use the following lemma presented in the generic 
  framework of binary classification 
  and illustrated in Figure \ref{fig:separability}
  in order  get an equivalence generally easier to check in practice.

  \begin{lemma}
  \label{lem:zero}
    Let $\{ (x_i, y_i)\}_{i=1}^t$ be any set of binary classification
    samples where $(x_i, y_i)\in \R^d \times \{-1, +1\}$.
    Then, there exists a separating axis $\omega \in \R^d$
    satisfying
    \[
      y_i \cdot \inner{\omega, x_i} >0, \textrm{~} \forall i \in \{1 \dots t \}
    \]
    if and only if
    \[
      \vec{0} \notin \textsc{ConvHull}\{y_i \cdot x_i \}_{i=1}^t.
    \]
\end{lemma}
  \begin{figure}[!t]
    \begin{center}$
	\begin{array}{lcccr}
	\includegraphics[width=30mm]{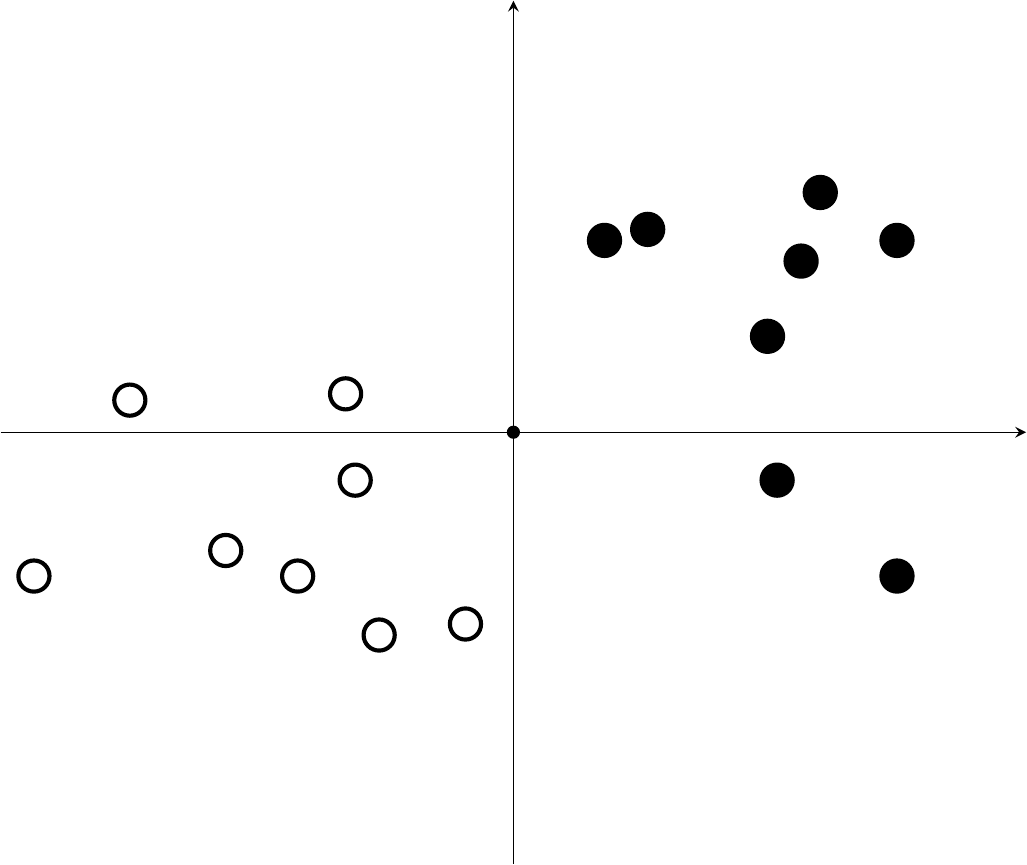}&~~~~~&
	\includegraphics[width=30mm]{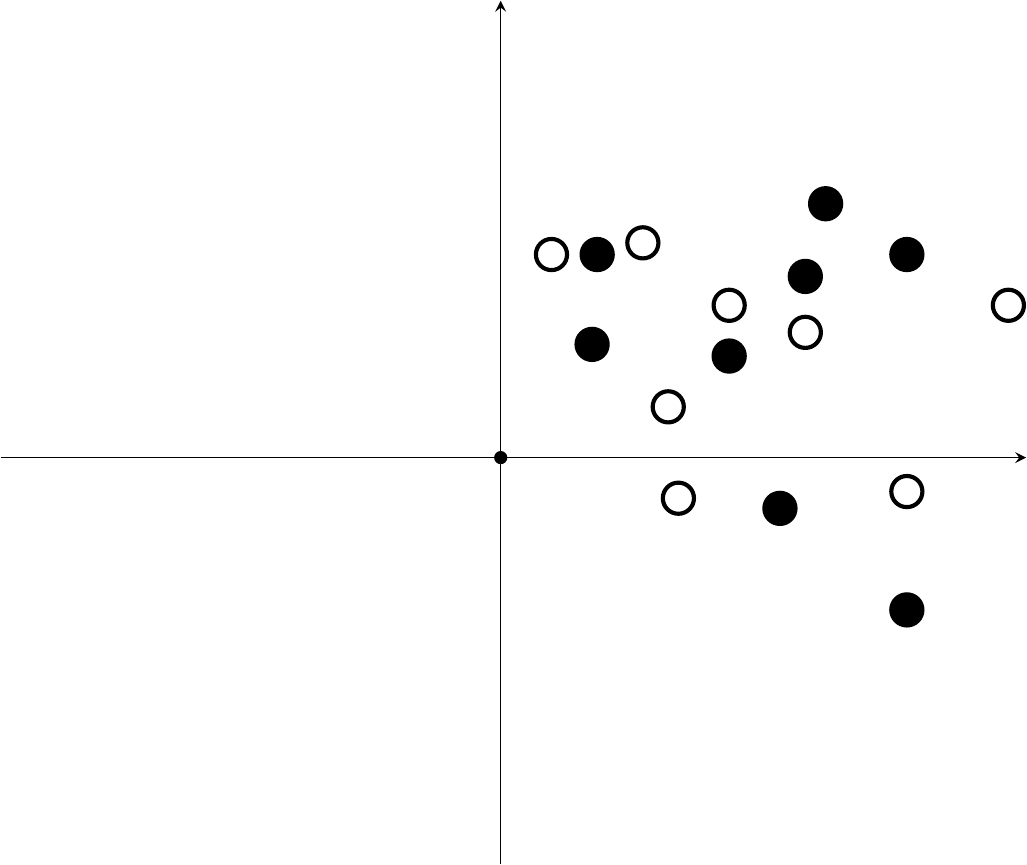}&~~~~~&
	\includegraphics[width=30mm]{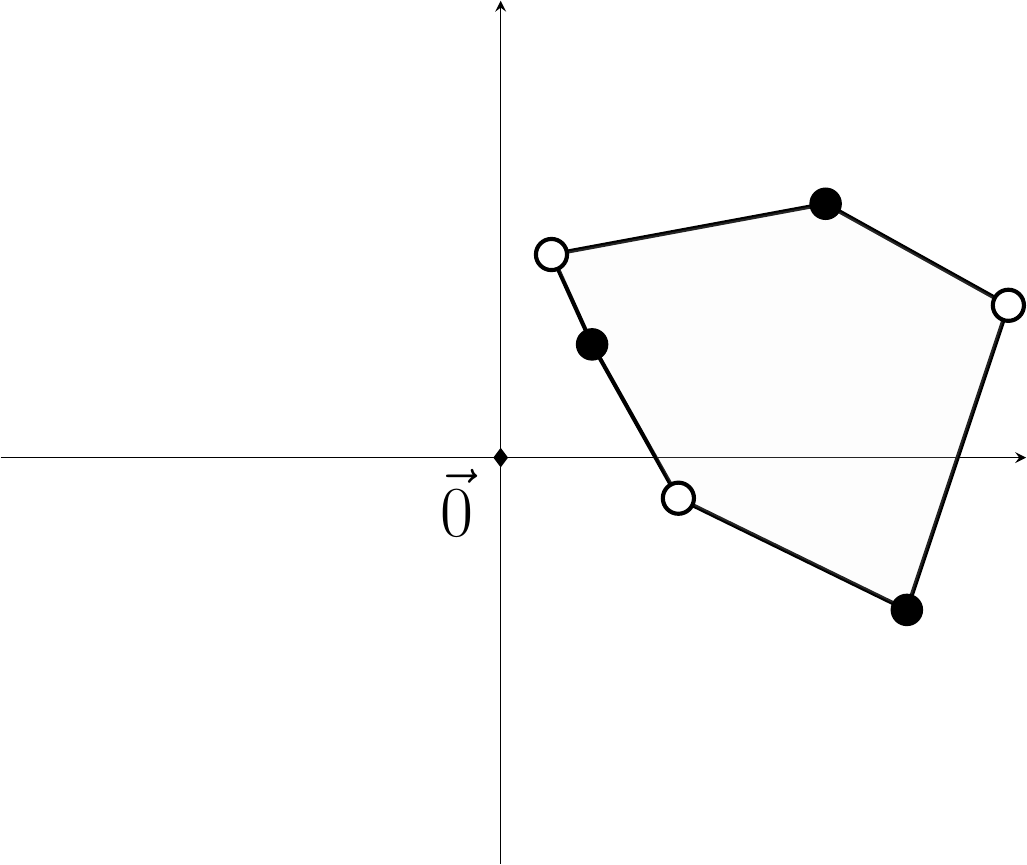}
	\end{array}$
    \end{center}
    \vspace{-0.5cm}
    \caption{Illustration of Lemma \ref{lem:zero}.
     {\it Left:} A separable sample $\{(x_i, y_i)\}_{i=1}^{n}$.
     {\it Middle:} The sample $\{ y_i \cdot x_i\}_{i=1}^{n}$.
     {\it Right:} The convex hull of $\{y_i \cdot x_i\}_{i=1}^{n}$
     next to the zero vector.
    }
    \label{fig:separability}
  \end{figure}

  \noindent  One can then deduce from the combination of Proposition \ref{prop:binary} 
  and Lemma \ref{lem:zero} that testing the existence of a consistent polynomial ranking rule
  can simply be performed by checking the emptiness of a specific
  polyhedron built from the sample as detailed in the next corollary.
  
  \begin{corollary}
  \label{coro:LP}
    Consider the same assumptions as in Proposition \ref{prop:binary}.
    Then, there exists a polynomial ranking rule of degree $k \geq 1$ 
    that perfectly ranks the sample
    if and only if the polyhedral set $\Omega^{\Phi_k}_t$ defined by
    \[
      \Omega^{\Phi_k}_t := \left\{\lambda \in \R^{t}: ~\normalfont{\text{M}^{\Phi_k}_{t}}
      \lambda^{\mathsf{T}} =\vec{0},
      ~\inner{\vec{1},\lambda} =1,~ \lambda \succeq \vec{0} \right\}
    \]
    is empty where $\normalfont{\text{M}^{\Phi_k}_{t}}
    =[C_1 \mid \cdots \mid C_t]$ is the
    $(\normalfont{\text{dim}}({\Phi_{\normalfont{k}}}),t)$-matrix
    with its $i$-th  column $C_i$ equal to
    $( \Phi_{\normalfont{k}}(X_{(i+1)}) 
    - \Phi_{\normalfont{k}}(X_{(i)}) )^\mathsf{T}$.
  \end{corollary}
  As a full implementation of the algorithm 
  can be derived at this point (see Figure \ref{fig:implementation} in Section
  \ref{sec:xps}
  for more details),
  a few comments are in order.

   \begin{remark}{\sc (Algorithmic aspects)}
    \label{rem:algo}
    Notice that, in practice, the problem of testing the emptiness
    of a polyhedral set admits a tractable solution. Indeed,
    it can be seen as the problem of
    determining if a particular linear program admits a 
    feasible point and can therefore be solved with the simplex algorithm.
    For further details on this topic,
    we refer to Chapter 11.4 in \cite{boyd2004convex} where
    practical examples as well as algorithmic solutions are discussed.
  \end{remark}

   \begin{remark}
   {\sc (Numerical complexity)} 
   In contrast, the numerical complexity
   of the proposed implementation can not  be tracked precisely
   due to the stochastic nature of the rejection method.
   Nonetheless, we point out that a simple union bound indicates that
   the complexity of generating the next evaluation point
   $X_{t+1}|\{X_i\}_{i=1}^t \sim \mathcal{U}(\X_t)$ 
   given a sample $\{(X_i, f(X_i))  \}_{i=1}^t$ 
   is upper bounded,
   with probability at least $1-\delta$,
   by the complexity of testing the emptiness of a polyhedron multiplied by 
   $\lceil {\ln(\delta)/\ln(1-\mu(\X_t)/\mu(\X ) )} \rceil$.
   But, we stress that the value of the ratio $\mu(\X_t)/\mu(\X)$ 
   which controls the upper bound depends on both
   the random evaluations previously made and
   the nested structure of the level sets of the unknown function
   and can therefore
   not be developed further.
   
  \end{remark}

\subsection{Convex ranking rules}

  We now consider the nonparametric sequence of convex ranking rules
  $\{\Rank_{\mathcal{C}_k}  \}_{k \in \mathbb{N}^{\star}}$.
  The equivalences provided below essentially rely on the fact that
  any bipartite ranking rule can be approximated
  by overlaying a finite sequence binary classifiers
  as previously shown
   in \cite{clemenccon2010overlaying}.
  We start with the one-dimensional case.

  \begin{proposition} {\sc (Overlaying classifiers)}
    \label{prop:binary_cvx}
    Set $d=1$ and
    assume that we have collected a sample 
    $\{(X_i, f(X_i)) \}_{i=1}^{t+1}$
    of $t+1$ function evaluations with distinct values.
    Then, there exists a convex ranking rule of degree $k \geq 1$ that perfectly 
    ranks the sample if and only if there a exists a sequence
    of classifiers $\{ h_i\}_{i=1}^{t+1}$
    of the form $h_i(x)= \sum_{m=1}^k  \indic{ l_{i,m} \leq x \leq u_{i,m}} $
    satisfying:
    \begin{enumerate}
      \item $h_i(X_{(j)}) =  \indic{ j \geq i }$, $\forall (i,j) \in \{1 \ldots t+1\}^2$;
      \item $h_1 \geq h_2 \geq \dots  \geq h_{t+1}$.
    \end{enumerate}
    where {\small$(1), (2) \dots (t+1)$} denote the indexes
    of the strictly increasing reordering of the sample.
  \end{proposition}
 In the specific case where $d >1$ and  $k=1$,
 we further argue that checking the existence of a consistent and finite collection
 of nested convex classifiers
 can be performed by determining the emptiness of a cascade of polyhedral sets.

  \begin{proposition}
    \label{prop:lp_cvx}
    Set any $d \in \mathbb{N}^{\star}$ and
    assume that we have collected a sample 
    $\{(X_i, f(X_i)) \}_{i=1}^{t+1}$ of $t+1$ function evaluations
    with distinct values. 
    Then, there exists a convex ranking rule of degree $k=1$ that
    perfectly ranks the sample
    if and only if for each $i =1, \ldots, t$, the polyhedral set $\Omega_{i}$ defined by
    \[
      \Omega_i := \left\{ \lambda \in \R^{i}: ~\normalfont{\text{M}}_i
      \lambda = X_{(t+1-i)}^{\mathsf{T}},~ \inner{\vec{1}, \lambda }=1,
      ~\lambda \succeq \vec{0} \right\}
    \]
    is empty where $~\normalfont{\text{M}}_i=[ C_1 \mid \cdots \mid C_i]$ is the $(d,i)$-matrix
    with its $j$-th column $C_j$ is equal to $X_{(t+2-j)}^{\mathsf{T}}$.
  \end{proposition}

\section{Numerical experiments}
\label{sec:xps}
  
  In this section, we compare the empirical performance of 
  the main algorithm of the paper to
  the existing state-of-the-art global optimization methods
  on real and synthetic problems.\\
  
  \noindent {\bf Algorithms.} 
  We compared {\sc AdaRankOpt} with five different types of
  algorithms, developed from various approaches of global optimization:
  \begin{itemize}
    \item \textbf{\textsc{BayesOpt}} (\cite{martinez2014bayesopt})
    is a Bayesian optimization algorithm.
    It uses a distribution over functions to build
    a surrogate model of the unknown function.
    The parameters controlling the distribution
    are estimated during the optimization process.
    \item \textbf{\textsc{CMA-ES}} (\cite{hansen2006cma}) is
    an evolutionary algorithm.
    At each iteration, the
    new evaluation points are sampled according to a multivariate
    normal distribution with a mean vector and a covariance matrix
    computed from the previous evaluations.
    \item \textbf{\textsc{CRS}} (\cite{kaelo2006some})
    is a variant of the Controlled Random Search of \cite{price1983global}
    which includes local mutations.
    It starts with a random population
    and randomly evolves these points by an heuristic rule.
    \item \textbf{\textsc{DIRECT}} (\cite{jones1993lipschitzian}) 
    is a Lipschitz optimization
    algorithm where the Lipschitz constant is unknown.
    It uses a deterministic splitting technique of
    the search space and it is therefore the only purely deterministic algorithm
    of the benchmark.
    
    \item \textbf{\textsc{MLSL}} (\cite{kan1987stochastic})
    is a multistart algorithm.
    It performs a series of local optimizations starting from
    points randomly chosen 
    by a clustering heuristic
    that helps to avoid repeated searches of the same local optima. 
  \end{itemize}

  For a fair comparison, the tuning parameters of the algorithms
  were all set to default and the {\sc AdaRankOpt} algorithm
  was used in all the experiments with the sequence of polynomial ranking rules
  and with a parameter $p$ fixed to $1/10$.
  The detailed implementation of the {\sc AdaRankOpt} algorithm used in the experiments
  can be found in Figure \ref{fig:implementation}.
  The source of the implementations of the remaining algorithms
  are also reported  in Table \ref{table:implement_competitor}.\\

  ~\smallskip
  
  \noindent {\bf Data sets.} We considered a series of nonconvex optimization problems 
  which involve real data sets and naturally arise 
  in the tuning of machine learning algorithms,
  and two series of artificial problems
  that are commonly met in standard global optimization benchmarks:
  
  \begin{itemize}
    \item[{1.}] We first studied the task of estimating
    the regularization parameter $\lambda$
      and the bandwidth $\sigma$ of a gaussian kernel ridge regression
      that minimize the empirical mean squared error of the predictions
      over a 10-fold cross validation.  We employed
    five data sets from the 
    UCI Machine Learning Repository (\cite{Lichman:2013}):
    {\it Auto-MPG}, {\it Breast Cancer Wisconsin (Prognostic)},
    {\it Concrete slump test}, {\it Housing}
    and {\it Yacht Hydrodynamics}.
     For each dataset, we only considered the real-valued attributes 
    which were centered and normalized so that
    $\sum_{i=1}^nX_i = 0$  and
     $\frac{1}{n}\sum_{i=1}^n X_i^2  = 1$ for all the attributes. 

    \item[{2.}] We then compared the algorithms on a series of 
    five bidimensional problems
    taken from \cite{jamil2013literature} and \cite{simulationlib}.
    The dimensionality of these problems allows 
    an easy visualization of the test functions and it can be seen
    that this series covers a wide variety of situations, including multimodal
    and non-linear functions as well as ill-conditioned and well-shaped functions.
    \item[{3.}]
    The last benchmark we used to assess the performance 
    of the algorithms consists of a set of five synthetic functions
    with a dimensionality varying from three to seven
    taken from \cite{finck2010real} and \cite{jamil2013literature}.
    Remark that,
    due to the the high dimensionality 
    of the input spaces,
    only few information is available 
    on the structure of the test functions of this series.
    \end{itemize}
    
     \noindent A complete description of the test functions of the benchmark can be found 
    in Table \ref{tab:func}.

\newpage
\vspace{-5em}
\RestyleAlgo{boxed}
\begin{figure}[!t]
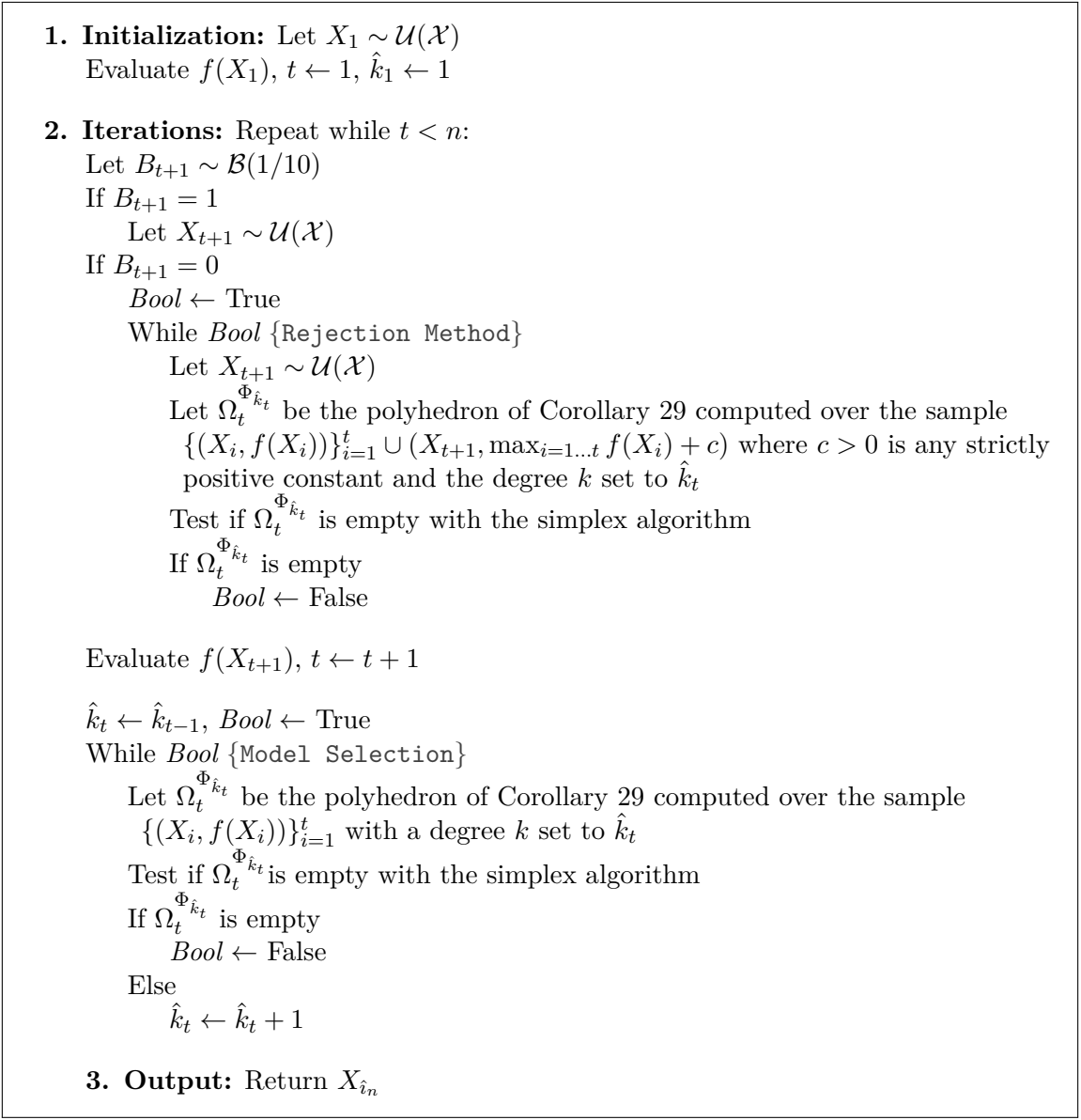

{
  \begin{algorithm}[H]
  \vspace{0.5em}
  \textbf{1. Initialization:} Let $X_1 \sim \mathcal{U}(\X)$\\
  \nonl \pushline  Evaluate $f(X_1)$, $t \leftarrow 1$,  $\hat{k}_1 \leftarrow 1$\\
  \vspace{1em}
  \nl \Indm \textbf{2. Iterations:} Repeat while $t <n$: \\
  \pushline \nonl Let $B_{t+1} \sim \mathcal{B}(1/10)$\\
  \nonl If  $B_{t+1} = 1$\\
  \pushline \nonl Let $X_{t+1} \sim \mathcal{U}(\X) $\\
  \nl \Indm  If  $B_{t+1} = 0$\\
  \pushline {\it Bool} $\leftarrow $ True \\
  \nonl While {\it Bool}~\textcolor{black!75}{\tt \{Rejection Method\}}\\
  \pushline Let $X_{t+1} \sim \mathcal{U}(\X) $\\
  Let $\Omega_{t}^{\Phi_{\hat{k}_t}}$ be the polyhedron of Corollary \ref{coro:LP}
  computed over the sample $\{(X_i, f(X_i))  \}_{i=1}^t
  \cup ( X_{t+1} ,\max_{i=1 \dots t}f(X_i) + c ) $ where $c>0$ is any strictly positive 
  constant and the degree $k$ set to $\hat{k}_t$\\
  Test if $\Omega_{t}^{\Phi_{\hat{k}_t}}$ is empty with the simplex algorithm\\
  \nonl If $\Omega_{t}^{\Phi_{\hat{k}_t}}$ is empty \\
  \nonl \pushline {\it Bool} $\leftarrow$ False \\
  \vspace{1em}
  \Indm \Indm \Indm Evaluate $f(X_{t+1})$, $t \leftarrow t +1$ \\
  \vspace{1em}
  $\hat{k}_t \leftarrow \hat{k}_{t-1} $, {\it Bool} $\leftarrow$ True\\
  While {\it Bool}~\textcolor{black!75}{\tt \{Model Selection\}} \\
  \pushline Let $\Omega_{t}^{\Phi_{\hat{k}_t}}$ be the polyhedron of Corollary \ref{coro:LP}
  computed over the sample $\{(X_i, f(X_i))  \}_{i=1}^t$
  with a degree $k$  set to $\hat{k}_t$ \\
  Test if $\Omega_{t}^{\Phi_{\hat{k}_t}}$is empty with the simplex algorithm\\
  If $\Omega_{t}^{\Phi_{\hat{k}_t}}$ is empty\\
  \pushline {\it Bool} $\leftarrow$ False \\
  \Indm Else \\
  \pushline $\hat{k}_t \leftarrow \hat{k}_{t} +1 $ \\
  \vspace{1em}
  \Indm \Indm \textbf{3. Output:} Return $X_{\hatin}$
  \vspace{0.5em}
  \end{algorithm}
   \vspace{-0.5em}
  \caption{Implementation of the {\sc AdaRankOpt} algorithm
	  with the sequence of polynomial ranking structures
	  and with a parameter $p$ set to 1/10.
	  }
  \label{fig:implementation}
 }
\end{figure}

\begin{table}[!h]
\centering
{\small
\begin{tabular}{lcc}
\toprule
{\bf Library} &  & {\bf Algorithm(s)} \\
\toprule
The CMA 1.1.06 package (\cite{CMAES_implementation}) &   &  CMA-ES \\
NLOpt Library (\cite{johnson2014nlopt}) &  & CRS, DIRECT, MLSL \\ 
BayesOpt Library (\cite{martinez2014bayesopt}) & ~~~~~~~~~~~~~~~~& BayesOpt \\
\bottomrule
\end{tabular}
\caption{Source of the implementations of the algorithms used for comparison.}
\label{table:implement_competitor}
}
\end{table}

\newpage
    
\begin{table}[!h]
\begin{center}
{ \footnotesize \tt
\vspace{-0.5em}
\begin{tabular}{@{}llcc@{}}
\toprule
 {\bf Problem} & {\bf Objective function}  & {\bf Domain}  & {\bf Local max.} \\ \midrule
  & & & \\
   {\bf Auto MPG} 
 &  
 &[-2,4]$\times$[-5,5]
 & - \\
 &\multirow{-2}{*}{-$\displaystyle \tt \frac{1}{10}\sum_{k=1}^{10} \sum_{i \in D_k}$($\tt \hat{f}_k(X_i) -Y_i $)$^{\tt 2}$  } 
 &  &  \\
  &  &  &  \\
 {\bf Breast Cancer~~} 
 & \multicolumn{1}{l}{where:} 
 &~\![-2,4]$\times$[-5,5]
 & - \\
 & \multicolumn{1}{l}{- $\tt \hat{f}_k \in$ $\tt\underset{f\in \mathcal{H}_{\sigma}}{argmin}$
 $\tt \frac{1}{n- |D_k|} \sum_{i \notin \mathcal{D}_k } (f(X_i)- Y_i)^2
 + \lambda \norm{f}_{\mathcal{H}_{\sigma}}$}
 &  &  \\
 &  
 \multicolumn{1}{l}{- the data set \{($\tt X_i, Y_i$)\}$_{\tt i=1}^n$ is split}
 &  &  \\
 {\bf Concrete} 
 & \multicolumn{1}{l}{~~into 10 folds $\tt D_1\dots D_{10}$}
 & [-2,4]$\times$[-5,5]
 & - \\
 & \multicolumn{1}{l}{- $\mathcal{H}_{\sigma}$ denotes the gaussian RKHS of}
 &
 &  \\
 &  
 \multicolumn{1}{l}{~~bandwidth $\sigma$}
 &  &  \\
  {\bf Yacht} 
 & \multicolumn{1}{l}{- $\tt \norm{f}_{\mathcal{H}_{\sigma}}$ is the corresponding norm}
 & [-2,4]$\times$[-5,5]
 &  - \\
 &
 &
 &  \\
  & 
  \multicolumn{1}{l}{- $\tt\sigma =10^{x_1}$}
  &  &  \\
 {\bf Housing} 
 & 
 & [-2,4]$\times$[-5,5]
 & -
 \\
 & \multicolumn{1}{l}{- $\tt\lambda = 10^{x_2}$}
 &  &  \\
 \cmidrule[0.2pt]{1-4}\\
  {\bf Branin-Hoo} 
 & 10(1\text{~\!-~\!}1/(8$\pi$))cos($\tt x_1$)\text{~\!+~\!}10 
 &~\![-5,10]$\times$[0,15]
 & 3 \\
 & +~\!\!($\tt x_2\text{~\!-~\!}5.1 x_1^2\text{/(}4\pi^2\text{)~\!\!+~\!\!}5x_1$/$\pi$\text{~\!-~\!}6)$^{\tt 2}$ 
 &  &  \\
 &  &  &  \\
 {\bf Himmelblau} 
 &  -~\!($\tt x_1^2\text{~\!+~\!}x_2$\text{~\!-~\!}11)$^{\tt 2}$
 &[-5,5]$\times$[-5,5]
 & 4 \\
 &\text{-~\!}($\tt x_1\text{~\!+~\!}x_2^2$\text{~\!-~\!}7)$^{\tt 2}$
 &  &  \\
  &  &  &  \\
 {\bf Styblinski} & 8$\tt x_2^2\text{~\!-~\!}0.5 x_2^4\text{~\!-~\!}2.5x_2$ 
 & [-5,5]$\times$[-5,5] 
 & 4 \\
 & $\tt\text{+~\!\!}8x_1^2\text{~\!-~\!}0.5x_1^4\text{~\!-~\!}2.5x_1$
 &
 &  \\
 &  &  &  \\
  {\bf Holder Table} 
 & |sin($\tt x_1$)|$\times$|cos($\tt x_2$)| 
 & [-10,10]$^{\tt 2}$
 &  36 \\
 & $ \times$exp(|1-($\tt x_1^2\text{~\!+~\!}x_2^2\text{)}^{1/2}\text{/}\pi$|) 
 &
 &  \\
  &  &  &  \\
 {\bf Levy N.13} 
 & -($\tt x_1$-1)$^2$(1+sin$^{\tt2}$(3$\tt\pi x_2$)) 
 & [-10,10]$^{\tt 2}$
 & $>$100
 \\
 & -sin$\tt^3$(3$\tt\pi x_1$)\text{~\!\!-~\!\!}($\tt x_2$-1)(1+sin$^{\tt2}$(2$\tt\pi x_2$))&  &  \\
 \cmidrule[0.2pt]{1-4}\\
 {\bf Rosenbrock} 
 & -$\tt \sum_{i=1}^3\text{(}x_i\text{~\!-~\!}1\text{)}^2$
 -~\!$\tt \sum_{i=1}^3 100\text{(}x_{i+1}\text{~\!+~\!}x_i^2 \text{)}^2$
 & [-2.048,2.048]$^{\tt3}$
 & - \\
  & 
  
  &  &  \\
  &  &  &  \\
    {\bf Mishra N.2} 
 &  -(6\text{~\!-~\!}$\tt
 \sum_{i=1}^5$0.5($\tt x_i +x_{i+1}$))$^{\tt 5 - \sum_{i=1}^5\text{0.5(}x_i+x_{i+1}\text{)} }$
 & [0,1]$^{\tt6
 }$
 & - \\
 &  
 &  &  \\
 &  &  &  \\
  {\bf Linear Slope }
 &  $\tt\sum_{i=1}^7 10^{\text{(}i-1\text{)/}6 } \text{(}x_i$-~\!\!5)
 &  [-5,5]$^{\tt7}$
 & 1 \\
 &  &  &  \\
 &  &  &  \\
  {\bf Deb N.1} 
 &  $\tt \frac{1}{5} \sum_{i=1}^5$sin$^{\tt6}$(5$\pi x_i$)
 & [-5,5]$^{\tt5}$
 & 36 \\
 &  &  &  \\
 &  &  &  \\
 {\bf Griewank N.4} 
 &  -1~\!-~\!$\tt\sum_{i=1}^4 x_i^2$/4000\!~+~\!$\tt\prod_{i=1}^4$cos($\tt x_i$~\!\!/~\!\!$\tt\sqrt{i}$)
 & [-300,600]$^{\tt4}$
 & >100 \\
 &  &  &  \\
 \bottomrule
\end{tabular}
}
\end{center}
\vspace{-1em}
\caption{Description of the test functions of the benchmark.
Dash symbols are used when a value can not be calculated.}
\label{tab:func}
\end{table}

\newpage
    
    \noindent {\bf Protocol and performances.} For each problem and each algorithm,
    we performed $K\!=\!\!100$ distinct runs with a budget of $n\!=\!\!1000$ 
    function evaluations.
    For each target parameter $t=$ 90\%, 95\% and 99\%, we have collected
    the stopping times corresponding to the number of evaluations 
    required by each method to reach the specified target 
     \[
      \tau_{k} := \min\{i=1,\dots, n:~
      f(X^{(k)}_{i}) \geq f_{\text{target}}(t)\} 
    \]
    where $\min\{ \emptyset \} = 1000$ by convention, $\{ f(X^{(k)}_{i})\}_{i=1}^n$ denotes 
    the evaluations made by a given method on the $k$-th run,
    $k \leq 100$ and the target value is set to 
    \[
      f_{\text{target}}(t) := \max_{x \in \X}f(x)  - 
      \left(\max_{x \in \X}f(x) -\int_{x \in \X} f(x)~\text{d}x/\mu(\X)
      \right) \times (1 - t).
    \]
    Note that the target is normalized to
    the average value of the function over the domain 
    to prevent the performance measures from being dependent of any 
    constant term in the unknown function. 
    In practice, the average value was estimated from a Monte Carlo sampling 
    of $10^6$ evaluations
    and the maximum of the function was estimated, for the real task problems,
    by taking the best value observed over all the sets of experiments.
    Based on these stopping times, 
    we then measured performance through a collection of indicators:
    \begin{itemize}
     \item[I)] Average and standard
     deviation of the number of evaluations required
     to reach the specified target:
     $\overline{\tau}_{K} =
     \frac{1}{K}\sum_{k=1}^K \tau_k$ and 
     $\widehat{\sigma}_{\tau} = ( \frac{1}{K} \sum_{k=1}^K (\tau_k -\overline{\tau}_K  )^2
     )^{1/2}$.
     \item[II)] Proportion of
     runs that reached the specified target in terms of function evaluations:
     $\forall i \leq n$, $\widehat{\P}_K(\tau \leq i) = \frac{1}{K} \sum_{k=1}^K
     \mathbb{I} \{ \tau_{k}  \leq i \}$.
     \item[III)] Number of runs for which a method has executed less (or more)
     evaluations to reach the target than {\sc AdaRankOpt}.
     Precisely, we have collected the following win/tie/loss indicators:
     $W= \sum_{k=1}^K \mathbb{I}\{ \tau_k < (1-0.1)\tau^{\text{ada}}_k \}$,
     $L =
      \sum_{k=1}^K \mathbb{I}\{ \tau_k> (1+0.1)\tau^{\text{ada}}_k \}$
     and $T = \text{K}- (W + L)$ 
     where $\tau^{\text{ada}}_k$ denotes
     the hitting times of the {\sc AdaRankOpt} algorithm.
    \end{itemize}

    These indicators capture the most important properties
    of global optimization algorithms, such as accuracy, stability 
    and velocity of convergence.\\

    \noindent {\bf Results and comments.} Results are collected 
    in a series of Tables \ref{table:num1}, \ref{table:num2}, \ref{table:num3}. 
    We also report the proportion
    of runs that reached the different targets in terms of
    function evaluations in Figures \ref{fig:num1}, \ref{fig:num2}, \ref{fig:num3}.
    Our main observations are the following: 
    \begin{itemize}
    
    \item[-]
    The proposed method displays---as one should expect---very competitive results 
    on test problems with estimated ranking rules of moderate complexity
    with regards to the sequence of ranking structures set as input
    (see, {\it e.g.}, {\it Breast Cancer},
    {\it Concrete}, {\it Housing}, {\it Himmelblau} or {\it Styblinski}).
    Moreover, experiments {\it Linear Slope} and {\it Mishra N.2} 
    also confirm that the algorithm can be robust against
    the dimensionality of the input space
    in the case of test functions with estimated ranking rule of low complexity.

    \item[-] In contrast
    the method stalls on test problems
    which do not admit an estimated ranking rule of moderate complexity
    (see, {\it e.g.}, {\it Deb N.1} and {\it Holder Table}).
    Indeed, the algorithm can not estimate efficiently the ranking rules induced
    by some classes of functions with a single sequence
    of ranking structures set as input.
    Considering at the same time multiple sequences of ranking structures
    might be a promising approach to address this issue,
    allowing the algorithm to adapt to wider varieties of shapes.
    
    \item[-]
    Finally, it can  be observed that
    in the case of test functions with strong global structure but many
    local optima, the algorithm reaches the 95\% target with few function evaluations
    but fails at getting to the 99\% target
    (see, {\it e.g.}, {\it Griewank N.4} or {\it Levy N.13}).
    Indeed, the algorithm starts moving toward the global optima
    by learning the global structure of the function
    but then considers ranking rules of a level of complexity higher than required
    when many local variations are met.
    As detailed in Remark \ref{rem:noisy},
    adding a noise parameter would allow the algorithm to be more robust
    against this type of local perturbations.
    \end{itemize}
    
    \noindent These empirical results aim at (i) providing numerical evidence
    that the main algorithm of the paper
    is competitive with the state-of-the-art methods 
    and effective on a wide range of optimization problems
    and (ii) identifying some limits of the ranking-based approach we developed that could
    be solved with further extensions.
    However, a complete and detailed empirical analysis of the merits
    and limitations of the algorithm with these extensions
    is beyond the scope of this paper and will be carried out in future work.

\begin{table}[!h]
\centering
\vspace{1em}
{\footnotesize {\tt
\begin{tabular}{@{}lccccc@{}}
\toprule
{\bf Problem} &
\multicolumn{1}{l}{{\bf Auto MPG}} 
&\multicolumn{1}{c}{{\bf Breast Cancer}}
&\multicolumn{1}{c}{{\bf Concrete}} 
& \multicolumn{1}{c}{{\bf Housing}} 
& 
{\bf Yacht}\\
\toprule
{\scriptsize  Target 90\%} & & & & & \\
{\bf AdaRank}  & 14.77 ($\pm$007) & ~6.14 ($\pm$003) & ~5.82 ($\pm$003) & ~{\selectfont\fontsize{8.7}{2}\fontfamily{pcr}\textbf{6~\!\!\!\!.64}} ($\pm$003) & 17.33 ($\pm$008)\\
{\bf BayesOpt} & {\selectfont\fontsize{8.7}{2}\fontfamily{pcr}\textbf{10~\!\!\!\!.\!84}} \!($\pm$003) & ~6.83 ($\pm$003) & ~6.40 ($\pm$004) & ~7.67 ($\pm$003) & 13.81 ($\pm$020)\\
{\bf CMA-ES}   & 29.27 ($\pm$024) & 11.10 ($\pm$009) & 10.41 ($\pm$008) & 12.84 ($\pm$012) & 29.61 ($\pm$025)\\
{\bf CRS }     & 28.73 ($\pm$014) & ~8.87 ($\pm$008) & 10.03 ($\pm$009) & 14.15 ($\pm$011) & 32.63 ($\pm$015)\\
{\bf DIRECT}   & 11.00 ($\pm$000) & ~{\selectfont\fontsize{8.7}{2}\fontfamily{pcr}\textbf{6~\!\!\!\!.\!00}} ($\pm$000) & ~6.00 ($\pm$000) & 11.00 ($\pm$000) & {\selectfont\fontsize{8.7}{2}\fontfamily{pcr}\textbf{11~\!\!\!\!.\!00}} ($\pm$000)\\
{\bf MLSL}     & 13.06 ($\pm$015) & ~6.59 ($\pm$004) & ~{\selectfont\fontsize{8.7}{2}\fontfamily{pcr}\textbf{3~\!\!\!\!.\!85}} ($\pm$004) & ~7.19 ($\pm$003) & 14.36 ($\pm$013)\\
\vspace{-0.3em} \\
{\scriptsize  Target 95\%} & & & & & \\
{\bf AdaRank}  & 17.14 ($\pm$008) & ~{\selectfont\fontsize{8.7}{2}\fontfamily{pcr}\textbf{6~\!\!\!\!.\!89}} ($\pm$004) & ~{\selectfont\fontsize{8.7}{2}\fontfamily{pcr}\textbf{6~\!\!\!\!.\!69}} ($\pm$003) & 12.25 ($\pm$004) & 23.45 ($\pm$012)\\
{\bf BayesOpt} & 12.20 ($\pm$006) & ~8.35 ($\pm$004) & ~7.94 ($\pm$004) & 14.10 ($\pm$022) & {\selectfont\fontsize{8.7}{2}\fontfamily{pcr}\textbf{15~\!\!\!\!.\!91}} ($\pm$021)\\
{\bf CMA-ES}   & 42.90 ($\pm$031) & 13.71 ($\pm$010) & 13.45 ($\pm$011) & 23.53 ($\pm$016) & 40.49 ($\pm$030)\\
{\bf CRS }     & 35.82 ($\pm$012) & 13.58 ($\pm$010) & 14.60 ($\pm$011) & 23.00 ($\pm$013) & 38.28 ($\pm$014)\\
{\bf DIRECT}   & {\selectfont\fontsize{8.7}{0.5}\fontfamily{pcr}\textbf{11~\!\!\!\!.\!00}} ($\pm$000) & 11.00 ($\pm$000) & 11.00 ($\pm$000) & 19.00 ($\pm$000) & 27.00 ($\pm$000)\\
{\bf MLSL}     & 14.97 ($\pm$015) & ~7.64 ($\pm$003) & ~7.31 ($\pm$004) & {\selectfont\fontsize{8.5}{2}\fontfamily{pcr}\textbf{11~\!\!\!\!.\!82}} ($\pm$007) & 16.25 ($\pm$013)\\
\vspace{-0.3em} \\
{\scriptsize  Target 99\%} & & & & & \\ 
{\bf AdaRank}  & 41.75 ($\pm$033) & 16.03 ($\pm$010) & 22.09 ($\pm$011) & 24.51 ($\pm$016) & 448.7 ($\pm$438)\\
{\bf BayesOpt} & {\selectfont\fontsize{8.7}{2}\fontfamily{pcr}\textbf{13~\!\!\!\!.\!97}} ($\pm$007) & 31.05 ($\pm$031) & 28.15 ($\pm$033) & 18.54 ($\pm$022) & {\selectfont\fontsize{8.7}{2}\fontfamily{pcr}\textbf{18~\!\!\!\!.\!84}} ($\pm$022)\\
{\bf CMA-ES}   & 73.74 ($\pm$049) & 35.13 ($\pm$035) & 46.31 ($\pm$029) & 62.14 ($\pm$085) & 70.87 ($\pm$049)\\
{\bf CRS }     & 48.48 ($\pm$016) & 34.84 ($\pm$034) & 36.55 ($\pm$014) & 44.09 ($\pm$014) & 52.89 ($\pm$018)\\
{\bf DIRECT}   & 47.00 ($\pm$000) & 27.00 ($\pm$027) & 37.00 ($\pm$000) & 41.00 ($\pm$000) & 49.00 ($\pm$000)\\
{\bf MLSL}     & 20.60 ($\pm$017) & {\selectfont\fontsize{8.7}{2}\fontfamily{pcr}\textbf{12~\!\!\!\!.\!84}} ($\pm$012) & {\selectfont\fontsize{9}{2}\fontfamily{pcr}\textbf{14~\!\!\!\!.\!73}} \!($\pm$010) & {\selectfont\fontsize{8.7}{2}\fontfamily{pcr}\textbf{16~\!\!\!\!.\!38}} ($\pm$010) & 21.43 ($\pm$014)\\
\vspace{-0.3em} \\
\cmidrule[0.5pt]{1-1} \vspace{-1.1em}\\
{\scriptsize  Target 90\%} & & & & & \\
{\bf BayesOpt} & 62/20/18 & 38/13/49 & 39/12/49 & 40/06/54 & 66/14/20\\
{\bf CMA-ES}   & 30/06/64 & 27/09/64 & 31/06/63 & 32/07/61 & 31/04/65\\
{\bf CRS }     & 17/06/77 & 42/06/52 & 39/03/58 & 28/05/67 & 13/07/80\\
{\bf DIRECT}   & 64/22/14 & 36/06/58 & 37/09/54 & 07/13/80 & 73/09/18\\
{\bf MLSL}     & 55/20/55 & 40/09/51 & 42/09/49 & 38/09/53 & 59/10/31\\
\vspace{-0.3em} \\
 {\scriptsize  Target 95\%}  &  &  &  &  &   \\
{\bf BayesOpt} & 68/17/15 & 29/14/57 & 32/12/56 & 54/15/31 & 74/09/17\\
{\bf CMA-ES}   & 17/01/87 & 21/10/69 & 29/02/69 & 25/02/73 & 26/04/70\\
{\bf CRS }     & 10/05/85 & 30/07/63 & 29/05/66 & 20/15/65 & 14/07/79\\
{\bf DIRECT}   & 78/12/10 & 10/13/77 & 07/13/80 & 02/07/91 & 21/20/59\\
{\bf MLSL}     & 52/22/26 & 34/10/56 & 34/09/57 & 52/18/30 & 68/11/21\\
\vspace{-0.3em} \\
 {\scriptsize  Target 99\%}  & & & & & \\
{\bf BayesOpt} & 87/09/04 & 52/12/36 & 47/11/42 & 72/09/19 & 92/03/05\\
{\bf CMA-ES}   & 24/05/71 & 13/08/79 & 19/10/71 & 12/03/85 & 65/02/33\\
{\bf CRS }     & 26/05/69 & 11/03/86 & 17/10/71 & 11/05/84 & 62/02/36\\
{\bf DIRECT}   & 25/03/69 & 02/07/91 & 06/09/85 & 05/04/91 & 58/04/38\\
{\bf MLSL}     & 66/14/20 & 61/12/27 & 71/07/22 & 07/08/21 & 88/05/07\\
\bottomrule
\end{tabular}
}
\vspace{-0.5em}
\caption{Results achieved on the cross-validation problems.
The top of the table displays 
the number of evaluations to reach the specified targets
(mean $\pm$ standard deviation).
In bold, the best result obtained for each target
in terms of average of evaluations.
The bottom of the table displays 
the number of win/tie/loss to {\sc AdaRankOpt}.}
\label{table:num1}
}
\end{table}

\begin{table}[!h]
\centering
{\footnotesize {\tt
\begin{tabular}{@{}lccccc@{}}
\toprule
{\bf Problem} &
\multicolumn{1}{l}{{\bf Branin-Hoo}} 
&\multicolumn{1}{c}{{\bf Himmelblau}}
&\multicolumn{1}{c}{{\bf Styblinski}} 
& \multicolumn{1}{c}{{\bf Holder Table}} 
& 
{\bf Levy N.13}\\
\toprule
{\scriptsize  Target 90\%} & & & & & \\
{\bf AdaRank}  & ~7.23 ($\pm$004) & 12.24 ($\pm$009) & ~27.5 ($\pm$010) & 170.8 ($\pm$185) & 13.10 ($\pm$012)\\
{\bf BayesOpt} & ~6.46 ($\pm$004) & 12.69 ($\pm$013) & ~79.9 ($\pm$079) & 410.0 ($\pm$417) & 10.37 ($\pm$006)\\
{\bf CMA-ES}   & 20.61 ($\pm$017) & 18.04 ($\pm$014) & 259.6 ($\pm$385) & ~{\selectfont\fontsize{8.7}{2}\fontfamily{pcr}\textbf{79~\!\!\!\!.\!9}} ($\pm$115) & 16.98 ($\pm$014)\\
{\bf CRS }     & ~8.77 ($\pm$007) & 13.41 ($\pm$013) & 160.3 ($\pm$297) & 307.9 ($\pm$422) & 17.89 ($\pm$016)\\
{\bf DIRECT}   & ~{\selectfont\fontsize{8.7}{2}\fontfamily{pcr}\textbf{4~\!\!\!\!.\!00}} ($\pm$000) & ~{\selectfont\fontsize{8.7}{2}\fontfamily{pcr}\textbf{2~\!\!\!\!.\!00}} ($\pm$000) & ~{\selectfont\fontsize{8.7}{2}\fontfamily{pcr}\textbf{20~\!\!\!\!.\!0}} ($\pm$000) & ~80.0 ($\pm$000) & ~{\selectfont\fontsize{8.7}{2}\fontfamily{pcr}\textbf{1~\!\!\!\!.\!00}} ($\pm$000)\\
{\bf MLSL}     & ~8.91 ($\pm$005) & ~7.60 ($\pm$005) & 116.4 ($\pm$090) & 305.0 ($\pm$379) & 35.57 ($\pm$035)\\
\vspace{-0.3em} \\ 
{\scriptsize  Target 95\%} & & & & & \\
{\bf AdaRank}  & ~{\selectfont\fontsize{8.7}{2}\fontfamily{pcr}\textbf{8~\!\!\!\!.\!79}} ($\pm$005) & 18.86 ($\pm$011) & ~34.5 ($\pm$011) & 285.4 ($\pm$276) & 19.67 ($\pm$022)\\
{\bf BayesOpt} & 10.40 ($\pm$004) & 20.77 ($\pm$020) & 150.3 ($\pm$146) & 417.8 ($\pm$410) & 14.64 ($\pm$006)\\
{\bf CMA-ES}   & 29.28 ($\pm$021) & 38.17 ($\pm$027) & 339.5 ($\pm$406) & 135.9 ($\pm$184) & 26.99 ($\pm$023)\\
{\bf CRS }     & 18.89 ($\pm$017) & 31.31 ($\pm$029) & 170.6 ($\pm$294) & 580.1 ($\pm$444) & 27.56 ($\pm$020)\\
{\bf DIRECT}   & 11.00 ($\pm$000) & 26.00 ($\pm$000) & ~{\selectfont\fontsize{8.7}{2}\fontfamily{pcr}\textbf{34~\!\!\!\!.\!0}} ($\pm$000) & ~{\selectfont\fontsize{8.7}{2}\fontfamily{pcr}\textbf{80~\!\!\!\!.\!0}} ($\pm$000) & ~{\selectfont\fontsize{8.7}{2}\fontfamily{pcr}\textbf{1~\!\!\!\!.\!00}} ($\pm$000)\\
{\bf MLSL}     & 14.53 ($\pm$017) & {\selectfont\fontsize{8.7}{2}\fontfamily{pcr}\textbf{10~\!\!\!\!.\!07}} ($\pm$005) & 118.0 ($\pm$090) & 315.7 ($\pm$384) & 43.10 ($\pm$160)\\
\vspace{-0.3em} \\
{\scriptsize  Target 99\%} & & & & & \\
{\bf AdaRank}  & 16.08 ($\pm$006) & 35.80 ($\pm$013) & ~58.3 ($\pm$023) & 808.6 ($\pm$301) & 184.2 ($\pm$230)\\
{\bf BayesOpt} & 14.99 ($\pm$005) & 32.19 ($\pm$023) & 602.5 ($\pm$376) & 422.0 ($\pm$407) & 37.17 ($\pm$028)\\
{\bf CMA-ES}   & 55.83 ($\pm$041) & 96.71 ($\pm$083) & 426.5 ($\pm$399) & 214.6 ($\pm$198) & 105.7 ($\pm$125)\\
{\bf CRS }     & 57.06 ($\pm$057) & 88.97 ($\pm$045) & 212.9 ($\pm$280) & 599.1 ($\pm$427) & 90.87 ($\pm$039)\\
{\bf DIRECT}   & {\selectfont\fontsize{8.7}{2}\fontfamily{pcr}\textbf{11~\!\!\!\!.\!00}} ($\pm$000) & 55.00 ($\pm$000) & ~{\selectfont\fontsize{8.7}{2}\fontfamily{pcr}\textbf{34~\!\!\!\!.\!0}} ($\pm$000) & ~{\selectfont\fontsize{8.7}{2}\fontfamily{pcr}\textbf{80~\!\!\!\!.\!0}} ($\pm$000) & {\selectfont\fontsize{8.7}{2}\fontfamily{pcr}\textbf{30~\!\!\!\!.\!00}} ($\pm$000)\\
{\bf MLSL}     & 61.79 ($\pm$177) & {\selectfont\fontsize{8.7}{2}\fontfamily{pcr}\textbf{15~\!\!\!\!.\!17}} ($\pm$005) & 121.2 ($\pm$090) & 321.7 ($\pm$382) & 67.41 ($\pm$201)\\
\vspace{-0.3em} \\
\cmidrule[0.5pt]{1-1} \vspace{-0.3em} \\
{\scriptsize  Target 90\%} & & & & & \\
{\bf BayesOpt} & 51/13/36 & 52/04/44 & 19/05/76 & 43/01/56 & 42/13/45\\
{\bf CMA-ES}   & 23/02/75 & 40/03/57 & 21/04/75 & 60/06/34 & 35/09/56\\
{\bf CRS }     & 44/08/48 & 45/03/52 & 23/05/72 & 47/03/50 & 44/03/53\\
{\bf DIRECT}   & 71/08/21 & 88/05/07 & 77/06/17 & 56/02/42 & 96/04/00\\
{\bf MLSL}     & 40/11/49 & 59/07/34 & 27/00/73 & 49/02/49 & 57/14/29\\
\vspace{-0.3em} \\
 {\scriptsize  Target 95\%}  &  &  &  &  &   \\
{\bf BayesOpt} & 35/15/50 & 55/05/40 & 14/05/81 & 44/05/51 & 39/15/46\\
{\bf CMA-ES}   & 16/02/82 & 25/06/69 & 11/03/86 & 62/03/35 & 32/07/61\\
{\bf CRS }     & 33/06/61 & 34/10/56 & 17/06/77 & 31/04/65 & 33/04/63\\
{\bf DIRECT}   & 28/27/45 & 25/15/60 & 26/43/31 & 70/02/28 & 99/01/00\\
{\bf MLSL}     & 37/11/52 & 72/02/26 & 29/00/71 & 52/02/46 & 56/18/26\\
\vspace{-0.3em} \\
 {\scriptsize  Target 99\%}  & & & & & \\
{\bf BayesOpt} & 48/10/42 & 62/02/36 & 03/13/94 & 64/20/16 & 61/07/32\\
{\bf CMA-ES}   & 08/05/87 & 14/06/80 & 09/13/88 & 87/05/08 & 50/03/47\\
{\bf CRS }     & 09/04/87 & 12/14/84 & 08/13/86 & 48/30/22 & 48/04/48\\
{\bf DIRECT}   & 79/10/11 & 02/13/15 & 79/13/07 & 95/00/05 & 59/01/40\\
{\bf MLSL}     & 29/14/57 & 87/05/08 & 31/13/66 & 77/11/12 & 69/08/23\\

\bottomrule
\end{tabular}
}
\caption{Results achieved on the first series of synthetic problems.
The top of the table displays 
the number of evaluations to reach the specified targets
(mean $\pm$ standard deviation).
In bold, the best result obtained  for each target
in terms of average of evaluations.
The bottom of the table displays 
the number of win/tie/loss to {\sc AdaRankOpt}.}
\label{table:num2}
}
\end{table}

\begin{table}[!h]
\centering
{\footnotesize {\tt
\begin{tabular}{@{}lccccc@{}}
\toprule
{\bf Problem} &
\multicolumn{1}{l}{{\bf Rosenbrock}} 
&\multicolumn{1}{c}{{\bf Mishra N.2}}
&\multicolumn{1}{c}{{\bf Linear Slope}} 
& \multicolumn{1}{c}{{\bf Deb N.1}} 
& 
{\bf Griewank N.4}\\
\toprule
{\scriptsize  Target 90\%} & & & & & \\
{\bf AdaRank}  & 10.53 ($\pm$009) & ~{\selectfont\fontsize{8.7}{2}\fontfamily{pcr}\textbf{4~\!\!\!\!.\!84}} ($\pm$003) & 54.60 ($\pm$009) & 950.0 ($\pm$180) & 35.87 ($\pm$016)\\
{\bf BayesOpt} & 11.97 ($\pm$008) & ~5.56 ($\pm$003) & 319.2 ($\pm$406) & 814.7 ($\pm$276) & {\selectfont\fontsize{8.7}{2}\fontfamily{pcr}\textbf{27~\!\!\!\!.\!67}} ($\pm$020)\\
{\bf CMA-ES}   & 16.30 ($\pm$012) & ~5.00 ($\pm$004) & 213.1 ($\pm$105) & 930.1 ($\pm$166) & 66.79 ($\pm$043)\\
{\bf CRS }     & 15.08 ($\pm$014) & ~5.10 ($\pm$005) & 368.1 ($\pm$239) & 980.7 ($\pm$108) & 76.70 ($\pm$044)\\
{\bf DIRECT}   & 10.00 ($\pm$000) & 16.00 ($\pm$000) & 390.0 ($\pm$000) & ~1000 ($\pm$000) & 103.0 ($\pm$000)\\
{\bf MLSL}     & ~{\selectfont\fontsize{8.7}{2}\fontfamily{pcr}\textbf{8~\!\!\!\!.\!82}} ($\pm$005) & ~5.51 ($\pm$003) & {\selectfont\fontsize{8.7}{2}\fontfamily{pcr}\textbf{27~\!\!\!\!.\!48}} ($\pm$036) & {\selectfont\fontsize{8.7}{2}\fontfamily{pcr}\textbf{198~\!\!\!\!.\!0}} ($\pm$326) & 218.1 ($\pm$394)\\
\vspace{-0.3em} \\
{\scriptsize  Target 95\%} & & & & & \\
{\bf AdaRank}  & 14.92 ($\pm$014) & ~7.89 ($\pm$004) & 76.15 ($\pm$015) & 991.8 ($\pm$091) & 185.0 ($\pm$274)\\
{\bf BayesOpt} & 17.39 ($\pm$014) & ~7.76 ($\pm$003) & 467.0 ($\pm$455) & 949.1 ($\pm$153) & {\selectfont\fontsize{8.7}{2}\fontfamily{pcr}\textbf{46~\!\!\!\!.\!43}} ($\pm$027)\\
{\bf CMA-ES}   & 22.09 ($\pm$015) & 10.12 ($\pm$008) & 279.9 ($\pm$100) & 952.0 ($\pm$127) & 138.3 ($\pm$109)\\
{\bf CRS }     & 21.95 ($\pm$018) & 10.38 ($\pm$009) & 553.5 ($\pm$319) & 997.1 ($\pm$038) & 136.3 ($\pm$047)\\
{\bf DIRECT}   & {\selectfont\fontsize{8.7}{2}\fontfamily{pcr}\textbf{10~\!\!\!\!.\!00}} ($\pm$000) & 36.00 ($\pm$000) & 512.0 ($\pm$000) & ~1000 ($\pm$000) & 130.0 ($\pm$000)\\
{\bf MLSL}     & 10.01 ($\pm$006) & ~{\selectfont\fontsize{8.7}{2}\fontfamily{pcr}\textbf{7~\!\!\!\!.\!40}} ($\pm$004) & {\selectfont\fontsize{8.7}{2}\fontfamily{pcr}\textbf{37~\!\!\!\!.\!74}} ($\pm$057) & {\selectfont\fontsize{8.7}{2}\fontfamily{pcr}\textbf{215~\!\!\!\!.\!8}} ($\pm$328) & 282.2 ($\pm$421)\\
\vspace{-0.3em} \\
{\scriptsize  Target 99\%} & & & & & \\
{\bf AdaRank}  & 33.62 ($\pm$029) & 19.33 ($\pm$005) & 127.5 ($\pm$032) & ~1000 ($\pm$000) & ~1000 ($\pm$000)\\
{\bf BayesOpt} & 27.71 ($\pm$027) & 22.63 ($\pm$019) & 468.3 ($\pm$468) & ~1000 ($\pm$000) & 422.7 ($\pm$360)\\
{\bf CMA-ES}   & 43.59 ($\pm$043) & 35.75 ($\pm$021) & 380.0 ($\pm$106) & 962.1 ($\pm$106) & {\selectfont\fontsize{8.7}{2}\fontfamily{pcr}\textbf{267~\!\!\!\!.\!5}} ($\pm$102)\\
{\bf CRS }     & 43.58 ($\pm$043) & 93.78 ($\pm$037) & 612.9 ($\pm$322) & ~1000 ($\pm$000) & 424.8 ($\pm$089)\\
{\bf DIRECT}   & 24.00 ($\pm$000) & 98.00 ($\pm$000) & 910.0 ($\pm$000) & ~1000 ($\pm$000) & 908.0 ($\pm$000)\\
{\bf MLSL}     & {\selectfont\fontsize{8.7}{2}\fontfamily{pcr}\textbf{19~\!\!\!\!.\!72}} ($\pm$051) & {\selectfont\fontsize{8.7}{2}\fontfamily{pcr}\textbf{13~\!\!\!\!.\!66}} ($\pm$007) & ~{\selectfont\fontsize{8.7}{2}\fontfamily{pcr}\textbf{50~\!\!\!\!.\!5}} ($\pm$080) & {\selectfont\fontsize{8.7}{2}\fontfamily{pcr}\textbf{256~\!\!\!\!.\!7}} ($\pm$334) & 451.1 ($\pm$445)\\
\vspace{-0.3em} \\
\cmidrule[0.5pt]{1-1} \vspace{-0.3em} \\
{\scriptsize  Target 90\%} & & & & & \\
{\bf BayesOpt} & 40/08/52 & 36/08/56 & 50/04/46 & 32/59/09 & 64/05/31\\
{\bf CMA-ES}   & 33/05/62 & 45/14/41 & 01/00/99 & 15/78/07 & 25/08/67\\
{\bf CRS }     & 38/03/59 & 52/06/42 & 00/0/100 & 04/86/10 & 27/02/71\\
{\bf DIRECT}   & 34/05/61 & 00/0/100 & 00/0/100 & 00/90/10 & 00/01/99\\
{\bf MLSL}     & 49/05/46 & 35/10/55 & 92/01/07 & 87/12/01 & 68/01/31\\
\vspace{-0.3em} \\
 {\scriptsize  Target 95\%}  &  &  &  &  &   \\
{\bf BayesOpt} & 37/08/55 & 48/08/44 & 42/01/57 & 11/88/01 & 70/07/23\\
{\bf CMA-ES}   & 33/01/66 & 42/07/51 & 00/0/100 & 15/84/01 & 32/03/65\\
{\bf CRS }     & 33/06/61 & 42/04/54 & 00/01/99 & 01/98/01 & 29/05/66\\
{\bf DIRECT}   & 48/03/49 & 00/0/100 & 00/0/100 & 00/99/01 & 23/01/76\\
{\bf MLSL}     & 56/06/38 & 46/11/43 & 90/02/08 & 87/13/00 & 68/00/32\\
\vspace{-0.3em} \\
 {\scriptsize  Target 99\%}  & & & & & \\
{\bf BayesOpt} & 56/04/40 & 58/11/31 & 43/01/56 & 0/100/00 & 75/25/00\\
{\bf CMA-ES}   & 36/06/58 & 18/10/72 & 00/01/99 & 12/88/00 & 99/01/00\\
{\bf CRS }     & 32/03/65 & 04/00/96 & 00/03/97 & 00/100/0 & 100/0/00\\
{\bf DIRECT}   & 36/15/49 & 00/0/100 & 00/0/100 & 00/100/0 & 00/100/0\\
{\bf MLSL}     & 79/01/18 & 81/07/12 & 91/01/08 & 85/15/00 & 63/37/00\\
\bottomrule
\end{tabular}
}
\caption{Results achieved on the second series of synthetic problems.
The top of the table displays 
the number of evaluations to reach the specified targets
(mean $\pm$ standard deviation).
In bold, the best result obtained for each target
in terms of average of evaluations.
The bottom of the table displays 
the number of win/tie/loss to {\sc AdaRankOpt}.}
\label{table:num3}
}
\end{table}

  \begin{figure}[!h]
  \label{fig:comparisonCV}
    {
    \begin{center}
    $\begin{array}{ccc}
      \includegraphics[height=40mm, page=1]{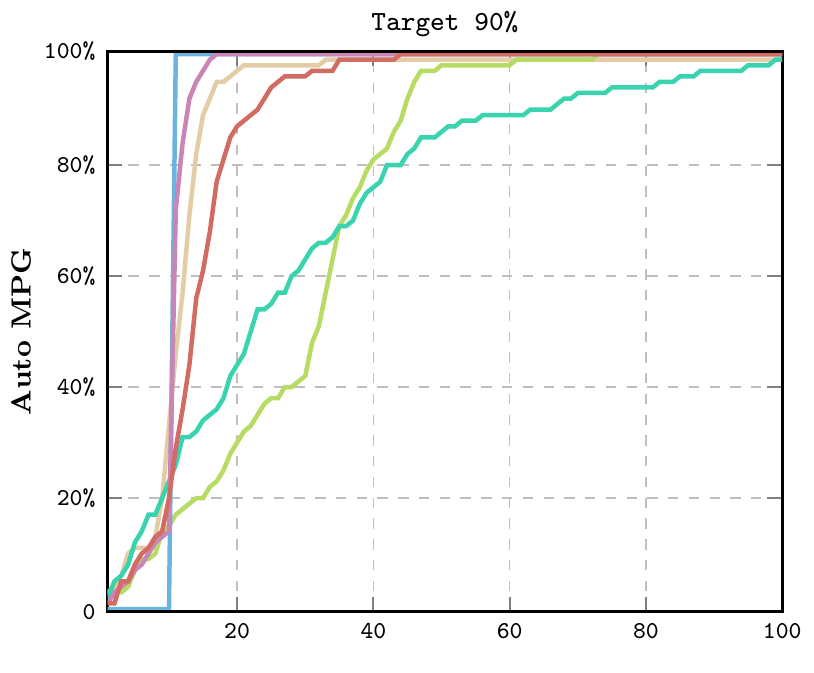}
      & \hspace{-1em} \includegraphics[height=40mm, page=2]{AutoMPG.pdf}
      & \hspace{-1em} \includegraphics[height=40mm, page=3]{AutoMPG.pdf}
      \vspace{-0.8em} \\
      \includegraphics[height=40mm, page=1]{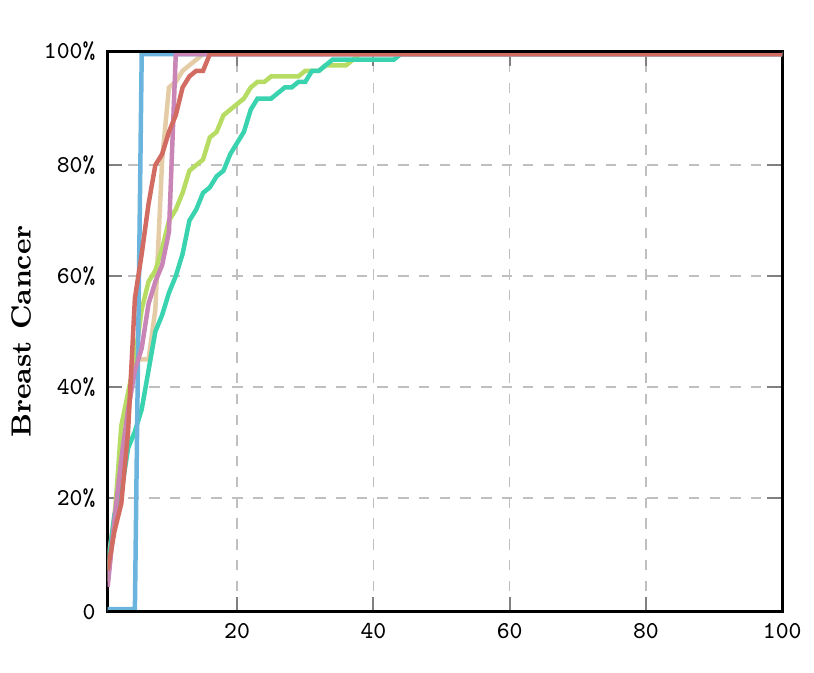}
      & \hspace{-1em} \includegraphics[height=40mm, page=2]{BreastCancer.pdf}
      & \hspace{-1em} \includegraphics[height=40mm, page=3]{BreastCancer.pdf}
      \vspace{-0.8em} \\
      \includegraphics[height=40mm, page=1]{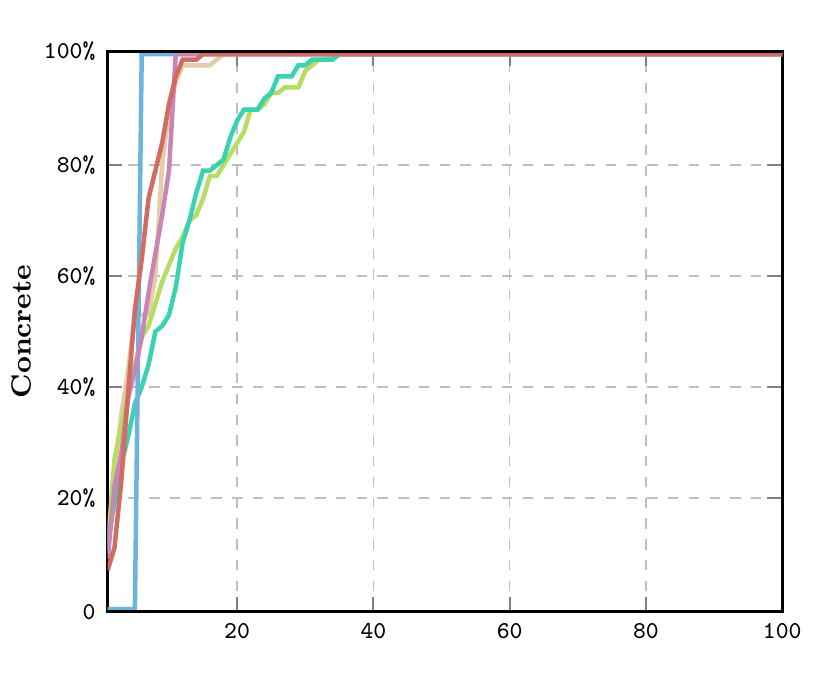}
      & \hspace{-1em} \includegraphics[height=40mm, page=2]{Concrete.pdf}
      & \hspace{-1em} \includegraphics[height=40mm, page=3]{Concrete.pdf}
      \vspace{-0.8em} \\
      \includegraphics[height=40mm, page=1]{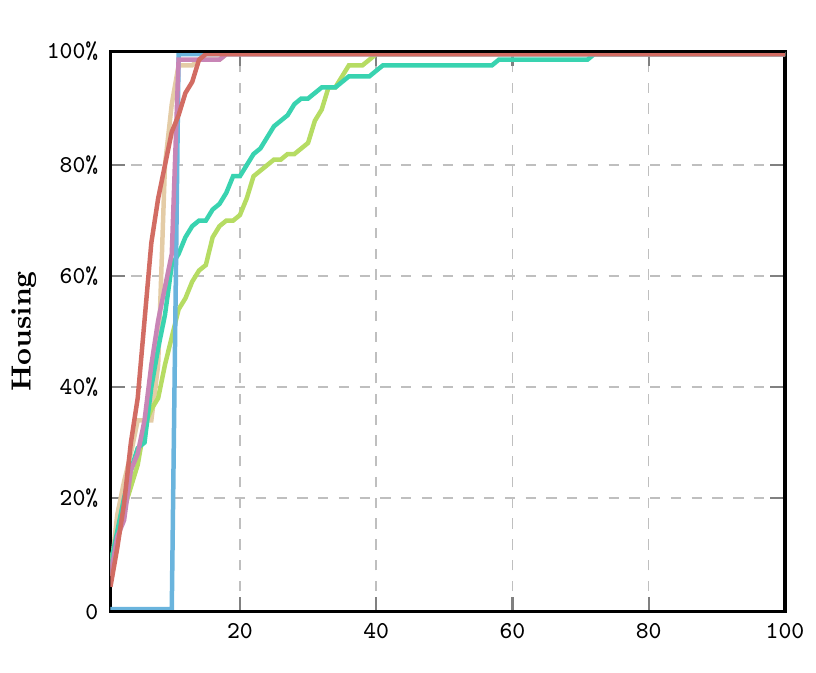}
      & \hspace{-1em} \includegraphics[height=40mm, page=2]{Housing.pdf}
      & \hspace{-1em} \includegraphics[height=40mm, page=3]{Housing.pdf}
      \vspace{-0.8em} \\
      \includegraphics[height=40mm, page=1]{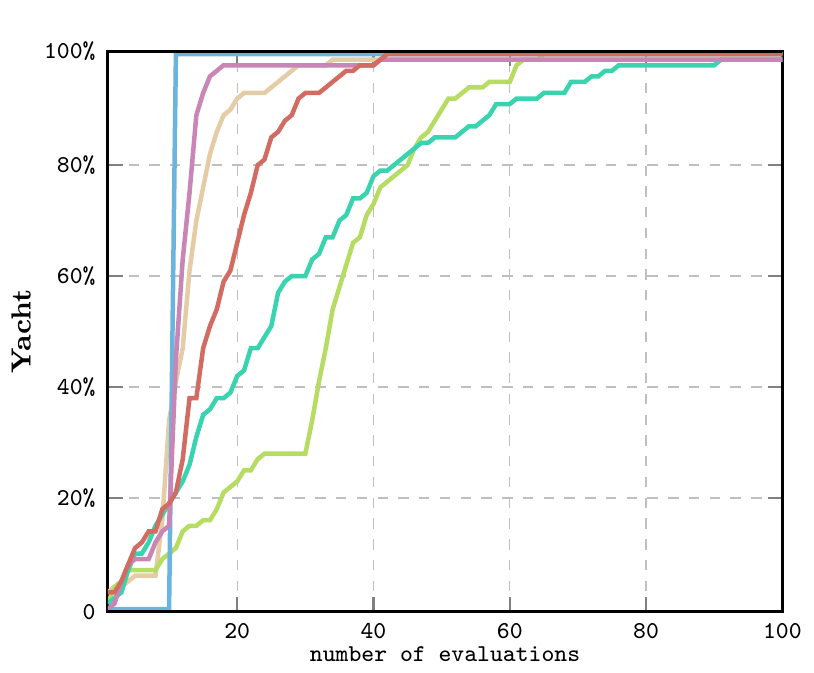}
      & \hspace{-1em} \includegraphics[height=40mm, page=2]{Yacht.pdf}
      & \hspace{-1em} \includegraphics[height=40mm, page=3]{Yacht.pdf}
      \\
      \end{array}$
    \end{center}
    }
  \vspace{-2em}
    \caption{
    Proportion of runs that reached the targets 90\%, 95\% and 99\% 
    in terms of function evaluations on each of the cross-validation problems.
    }
    \label{fig:num1}
  \end{figure}

  \begin{figure}[!h]
  \label{fig:comparisonFirst}
    {
    \begin{center}
    $\begin{array}{ccc}
      \includegraphics[height=40mm, page=1]{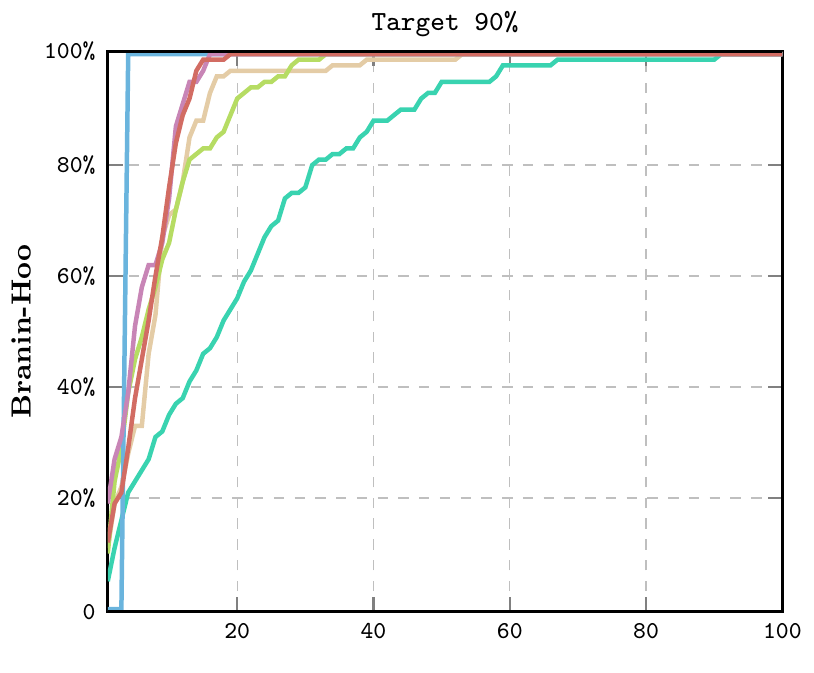}
      & \hspace{-1em} \includegraphics[height=40mm, page=2]{Branin.pdf}
      & \hspace{-1em} \includegraphics[height=40mm, page=3]{Branin.pdf}
      \vspace{-0.8em} \\
      \includegraphics[height=40mm, page=1]{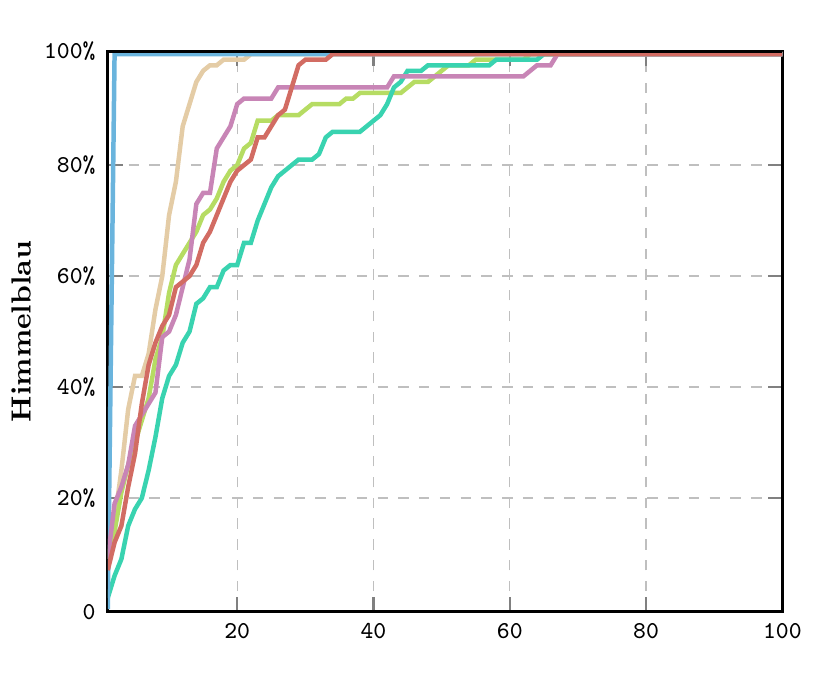}
      & \hspace{-1em} \includegraphics[height=40mm, page=2]{Himmelblau.pdf}
      & \hspace{-1em} \includegraphics[height=40mm, page=3]{Himmelblau.pdf}
      \vspace{-0.8em} \\
      \includegraphics[height=40mm, page=1]{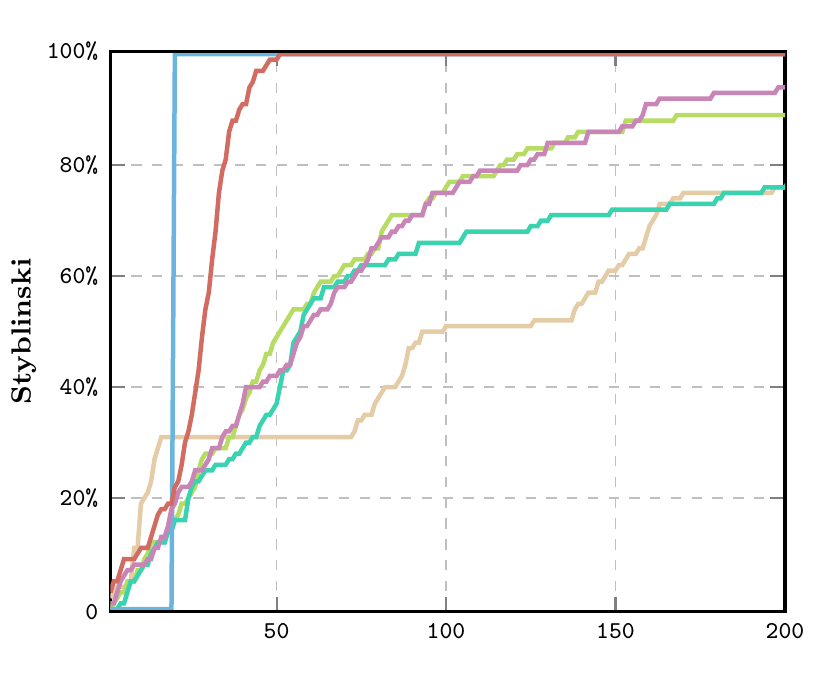}
      & \hspace{-1em} \includegraphics[height=40mm, page=2]{Styblinski.pdf}
      & \hspace{-1em} \includegraphics[height=40mm, page=3]{Styblinski.pdf}
      \vspace{-0.8em} \\
      \includegraphics[height=40mm, page=1]{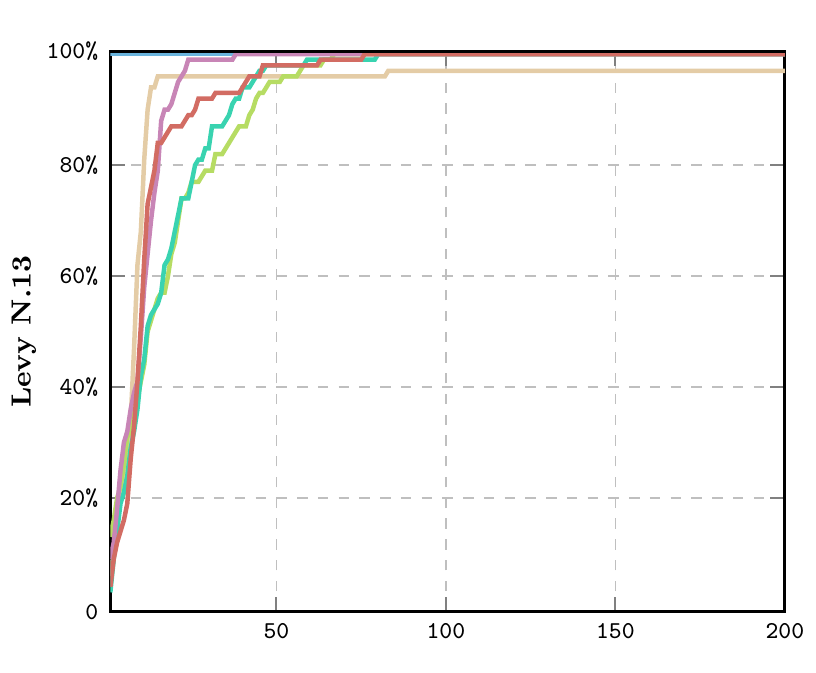}
      & \hspace{-1em} \includegraphics[height=40mm, page=2]{Levy.pdf}
      & \hspace{-1em} \includegraphics[height=40mm, page=3]{Levy.pdf}
      \vspace{-0.8em} \\
      \includegraphics[height=40mm, page=1]{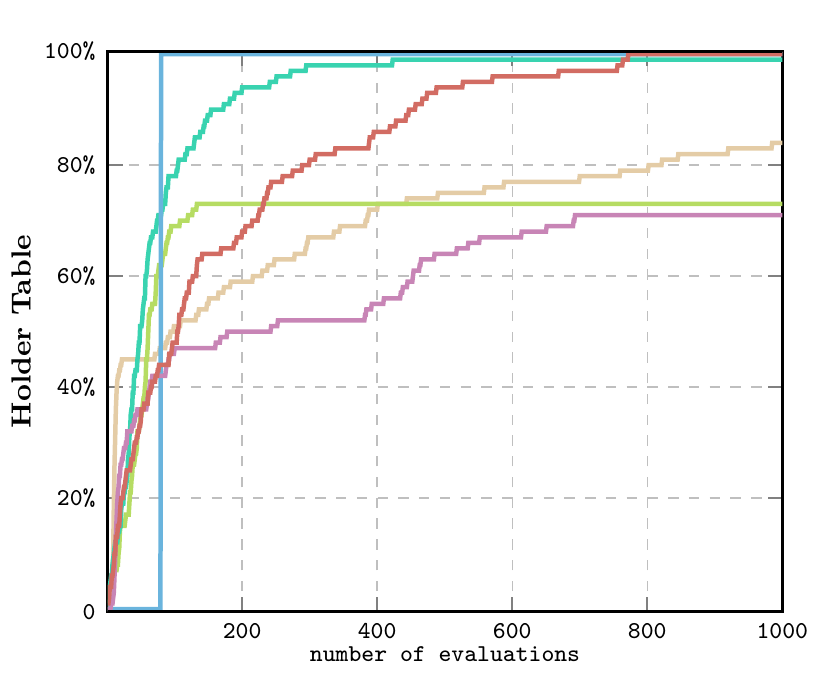}
      & \hspace{-1em} \includegraphics[height=40mm, page=2]{HolderTable.pdf}
      & \hspace{-1em} \includegraphics[height=40mm, page=3]{HolderTable.pdf}
      \\
      \end{array}$
    \end{center}
    }

  \vspace{-2em}
    \caption{
    Proportion of runs that 
    reached the targets 90\%, 95\% and 99\% 
    in terms of function evaluations on each problem of the 
    first series of synthetic functions.
    }
    \label{fig:num2}
  \end{figure}

  \begin{figure}[!h]
  \label{fig:comparisonFirst}
    {
    \begin{center}
    $\begin{array}{ccc}
      \includegraphics[height=40mm, page=1]{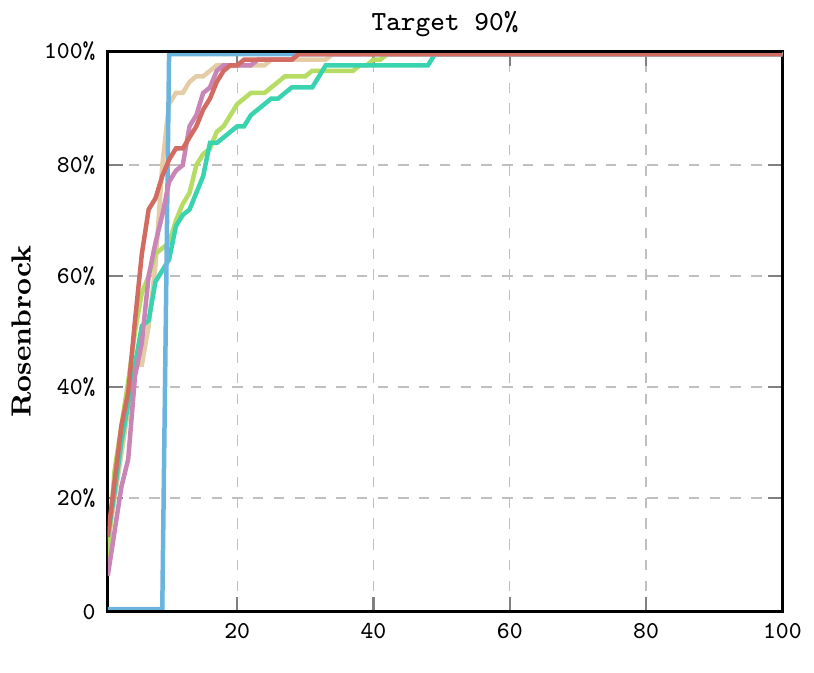}
      & \hspace{-1em} \includegraphics[height=40mm, page=2]{Rosenbrock.pdf}
      & \hspace{-1em} \includegraphics[height=40mm, page=3]{Rosenbrock.pdf}
      \vspace{-0.8em} \\
      \includegraphics[height=40mm, page=1]{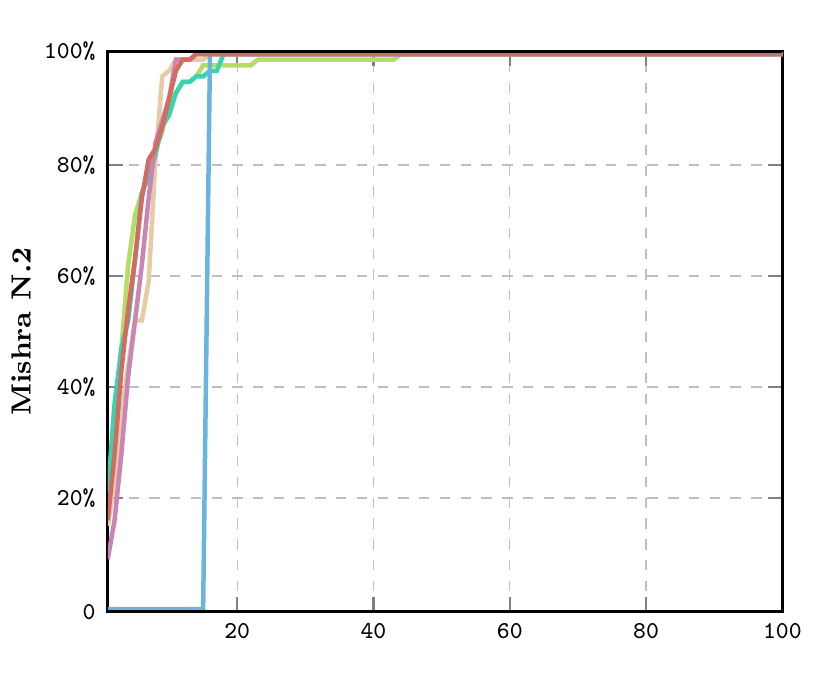}
      & \hspace{-1em} \includegraphics[height=40mm, page=2]{Mishra.pdf}
      & \hspace{-1em} \includegraphics[height=40mm, page=3]{Mishra.pdf}
      \vspace{-0.8em} \\
      \includegraphics[height=40mm, page=1]{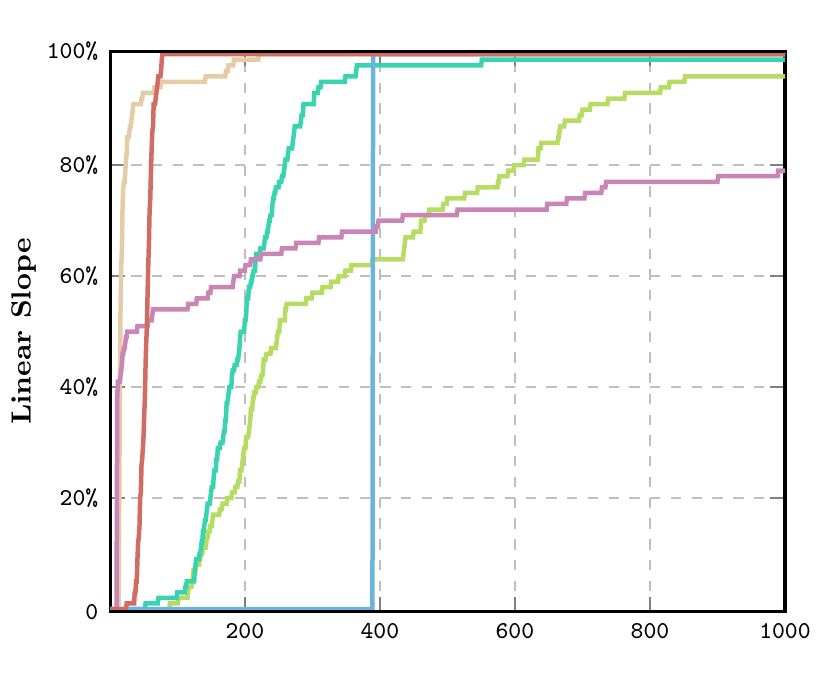}
      & \hspace{-1em} \includegraphics[height=40mm, page=2]{LinearSlope.pdf}
      & \hspace{-1em} \includegraphics[height=40mm, page=3]{LinearSlope.pdf}
      \vspace{-0.8em} \\
      \includegraphics[height=40mm, page=1]{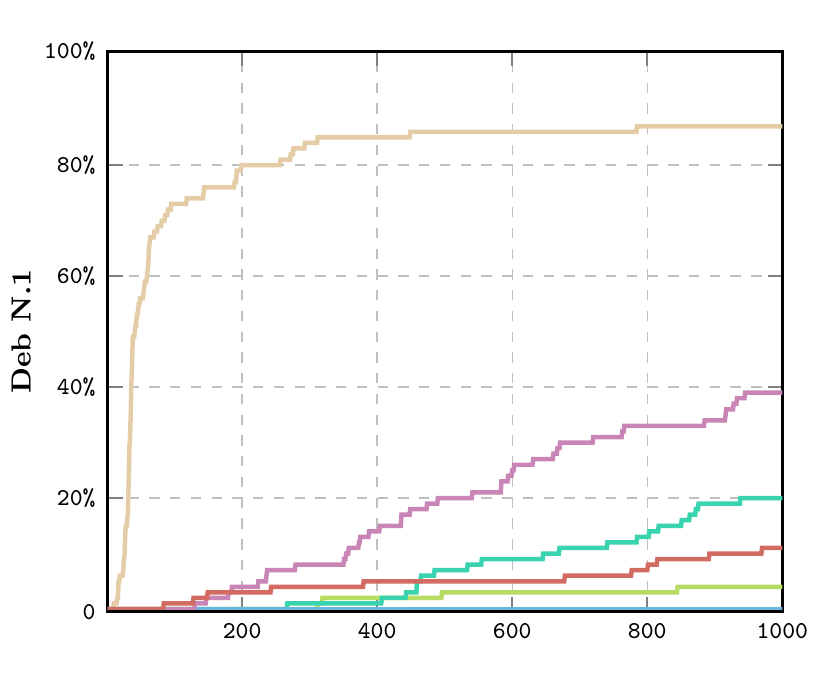}
      & \hspace{-1em} \includegraphics[height=40mm, page=2]{Deb.pdf}
      & \hspace{-1em} \includegraphics[height=40mm, page=3]{Deb.pdf}
      \vspace{-0.8em} \\
      \includegraphics[height=40mm, page=1]{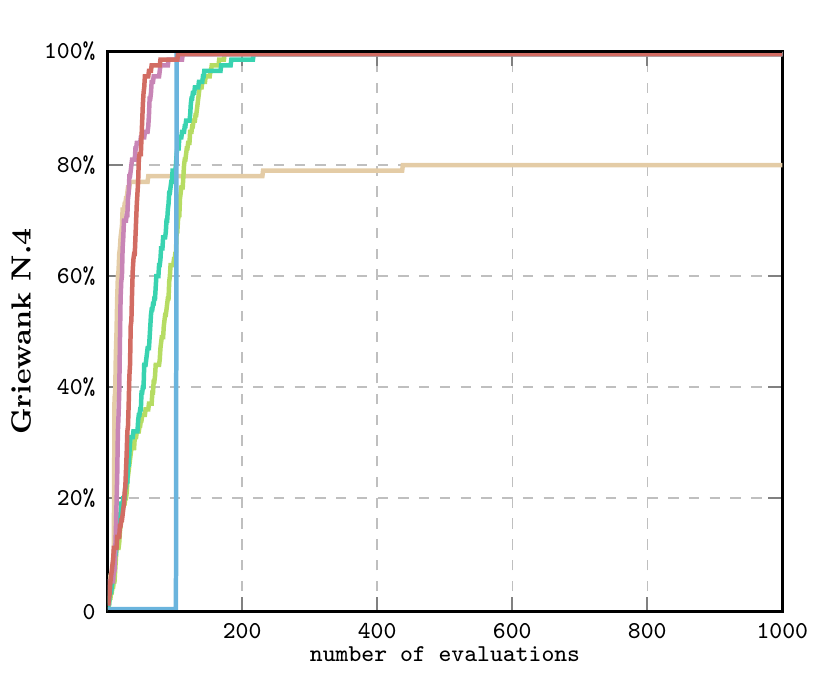}
      & \hspace{-1em} \includegraphics[height=40mm, page=2]{Griewank.pdf}
      & \hspace{-1em} \includegraphics[height=40mm, page=3]{Griewank.pdf}
      \\
      \end{array}$
    \end{center}
    }
  \vspace{-2em}
    \caption{
    Proportion of runs that reached 
    the targets 90\%, 95\% and 99\% 
    in terms of function evaluations on each problem of the 
    second series of synthetic functions.
    }
    \label{fig:num3}
  \end{figure}


  \section{Conclusion and future work}
  
  The major contribution of the paper was to
  show how to apply the principles of bipartite ranking to the 
  global optimization problem. We introduced two novel global optimization strategies 
  based on a sequential estimation of the ranking rule induced by the unknown function:
  {\sc RankOpt} which requires the knowledge of a ranking structure
  containing the induced ranking rule of the unknown function
  and its adaptive version {\sc AdaRankOpt}
  which performs model selection.
  A theoretical analysis of the algorithms is provided and 
  empirical results based on synthetic and real
  problems have also been obtained, 
  demonstrating the competitiveness of the adaptive version of the algorithm
  with regards to the existing state-of-the-art global optimization methods.
  Several questions are left open for future work. 
  First, our theoretical analysis suggest that the characterization of real-valued 
  functions given a particular
  ranking structure could be refined in order to identify
  the classes of functions providing the exponentially 
  decreasing loss obtained in the lower bound.
  Second,
  our numerical studies also suggest that
  the empirical performance of the algorithm could be improved by
  (i) investigating the choice of the sequence of ranking structures set as input
  and (ii) 
  allowing the algorithm do deal with noisy evaluations.
  Last, determining whether using a more aggressive sampling strategy
  would lead to better empirical results without deteriorating the convergence properties
  of the algorithm is left as a completely open question.
  
%

\newpage
\appendix

\section{Proof of Proposition \ref{prop:rankingequivalence}}

We develop here the proof for the equivalence class of real-valued functions
sharing the same induced ranking stated in Proposition \ref{prop:rankingequivalence}.\\

\noindent {\bf Proof of proposition \ref{prop:rankingequivalence}}
  \noindent $(\Leftarrow) \!$ Assume that there exists a strictly increasing function
  $\psi : \R \rightarrow \R$
  such that
  $h = \psi \circ f$.
  Since $\psi$ is strictly increasing, it directly follows that $\forall (x,x')\in\X^2$,
  \[
    r_{h}(x,x')=\sgn{\psi \circ f(x) - \psi \circ f(x')}= \sgn{f(x) - f(x') }=r_{f}(x,x').
  \]

  ~\\
  \noindent $(\Rightarrow) \!$ Assume now that $\forall (x,x')\in\X^2$,
  $r_f(x,x')=r_{h}(x,x')$.
  First, note that 
  if $\forall (x,x')\in\X^2$ $r_{f}(x,x')=r_h(x,x')=0$,
  both $f=c_1$ and $h=c_2$ are constant over $\X$
  and then $h= \psi \circ f$ where $\psi: x \mapsto x + (c_2-c_1)$
  is a strictly increasing function.
  We now consider the case where $f$ is not constant over $\X$
  and 
  we start to show that there exists a strictly increasing function
  $\psi : \R \rightarrow \R$ such that
  $f= \psi \circ M$ where $M: \X \rightarrow [0,1]$ is defined for all $x\in\X$ by
  \[
    M: x \mapsto \int_{x' \in \X} \indic{r_f(x,x')<0} ~\text{d}x'
    = \mu( \{x' \in \X: f(x') < f(x)\}  ).
  \]
  To properly define $\psi$,
  we first need to ensure that the function $f$ is constant over
  the iso-level set
  $M^{-1}(y)= \{x \in \X: M(x)=y \}$ for all $y \in \text{Im}(M)$.
  To do so, fix any $y \in \text{Im}(M)$,
  pick any $(x_1,x_2)\in M^{-1}(y) \times M^{-1}(y)$
  and assume by contradiction and without loss of generality
  that $f(x_1) < f(x_2)$.
  As the ranking rules $r_f$ and $r_h$ are assumed to be equal over $\X \times \X$,
  we have that
  (i) $h(x_1) < h(x_2)$
  and (ii) $M(x_i) = \mu( \{x': h(x') < h(x_i) \})$, $i \in \{1,2 \}$.
  Hence putting (i) and (ii)  altogether
  with the continuity of $h$ leads us to the contradiction
  $$
    M(x_1)= \mu( \{x' \in \X: h(x')<h(x_1) \})<\mu( \{x' \in \X: h(x') < h(x_2) \})=M(x_2)
  $$
  and we deduce that $f$ is constant over any iso-level set of $M$.
  Now, denoting by $f(M^{-1}(y))$ the unique value of $f$ over $M^{-1}(y)$,
  we are ready to introduce the restriction 
  of the function $\psi$ over Im$(M)$ defined by
  \[
    \psi_{\textrm{Im}(M)}: y \in \textrm{Im}(M) \mapsto f(M^{-1}(y)) \in \R.
  \]
  As $\forall x \in \X$, $\psi_{\textrm{Im}(M)}(M(x) )=f(x)$,
  it follows from the continuity of $h$ that
  $\forall y_1<y_2 \in \textrm{Im}(M) \times \textrm{Im}(M)$, 
  $\psi_{\textrm{Im}(M)}(y_1) < \psi_{\textrm{Im}(M)}(y_2)$. Hence
  $\psi_{\textrm{Im}(M)}$ is strictly increasing over Im$(f)$ and
  one can then write that $f= \psi \circ M$ where $\psi : \R \rightarrow \R$
  is any strictly increasing extension of the function
  $\psi_{\textrm{Im}(M)}$  over $\R$.
  In addition, it can easily be shown 
  by reproducing the same steps as previously with the function $h$
  that there also exists a strictly increasing function
  $\psi': \R \rightarrow \R$ such that $h= \psi' \circ M$.
  Hence, the desired result follows by
  writing that $h= \psi' \circ M = (\psi' \circ \psi^{-1} )\circ f$
  where  $\psi' \circ \psi^{-1}: \R \rightarrow \R$ is a strictly increasing function.
\hfill\(\Box\)

\section{Analysis of the  \textbf{\textsc{RankOpt}} algorithm}

In this section, we develop the full proofs of Proposition \ref{th:fasterprs},
Corollary \ref{coro:consistencyrank}, Theorem \ref{th:upperbound},
Proposition \ref{prop:slowerpas} and Theorem \ref{th:lowerbound}.

\subsection{Generic results and technical lemmas}

 We start by casting a simple property (Proposition \ref{prop:chain}) and
 two technical lemmas (Lemma \ref{lem:zab}
 and Lemma \ref{lem:inclusion})
 that will be used throughout the proofs.

\begin{proposition} {\sc (RankOpt process)}
\label{prop:chain}
  Consider that the assumptions of Proposition \ref{prop:slowerpas} are fulfilled.
  Then, the sequence of evaluation points $\{ X_i\}_{i=1}^n$ generated by 
  the \textsl{\textsc{RankOpt}} algorithm after $n$ iterations,
  that will be denoted in the sequel by 
  $\{ X_i\}_{i=1}^n \sim$ \textsl{\textsc{RankOpt}}$(n,$ $f,$ $\X,$ $\Rank)$,
  is distributed as follows:
  \begin{equation*}
    \begin{cases}
      X_1 \sim \mathcal{U}(\X) \\
      X_{t+1}|~\! \{X_i \}_{i=1}^t \sim \mathcal{U}(\X_t)
      \textrm{~~~~~~~~~} \forall t \in \{1 \ldots n-1 \}
    \end{cases}
  \end{equation*}
  where at each step $t\geq 1$
  the sampling area $\X_t := \{x \in \X: 
  \exists r \in \Rank_t~\text{such that}~r(x,X_{\hatit})\geq 0 \}$
  satisfies 
  \[
   \{x \in \X: f(x) \geq f(X_{\hatit}) \} 
   \subseteq \X_t \subseteq \X
  \]
  where $\hatit \in \arg\max_{i=1\dots t}f(X_i) $.
\end{proposition}

\begin{proof}
  The first part of the proposition 
  is a direct consequence of 
  the definition of the algorithm.
  Noticing that $\X_t$ is a subset of $\X$ gives
  the second inclusion of the second part of the proposition.
  To prove the first inclusion fix any $t \geq 1$,
  pick any $x \in \X $ satisfying $f(x) \geq f(X_{\hatit})$ and
  observe that $r_f(x, X_{\hatit})= \sgn{f(x) - f(X_{\hatit})} \geq 0$.
  As $L_{t}(r_f) =0$ by definition,
  it necessarily follows that $r_f \in \Rank_t$.
  Hence there exists $r=r_f \in \Rank_t$ 
  such that $r(x, X_{\hatit}) \geq 0$
  and we deduce that $\{ x \in \X: f(x) \geq f(X_{\hatit}) \}
  \subseteq \X_t $ which concludes the proof.
\end{proof}

\noindent The next lemmas (Lemma \ref{lem:zab}
 and Lemma \ref{lem:inclusion})
will be important in order to state the nonasymptotic performance
of the algorithm ({\it i.e.}~Theorems \ref{th:upperbound} and \ref{th:lowerbound}).

\begin{lemma} $\normalfont{(}$From \cite{zabinsky1992pure},
see Appendix Section$\normalfont{)}$
\label{lem:zab}
  Let $\X \subset \R^d$ be any compact and convex set
  with non-empty interior.
  Then, for any $x^{\star} \in \X$ and
  any $r \in (0, \diam{\X} )$, we have that
  \[
    \frac{\mu( B(x^{\star}, r) \cap \X )}{\mu( \X)}
    \geq \left( \frac{r}{\diam{\X}} \right)^{d}.
  \]
\end{lemma}

\begin{proof}
  Introduce the similarity transformation  $S: \R^d \rightarrow \R^d$
  defined by
  \[
    S: x \mapsto x^{\star} + \frac{r}{\diam{\X}} (x-x^{\star})
  \]
  and let $S(\X)=\{S(x) : x \in \X \}$ be the image of $\X$ by $S$.
  Since $x^{\star} \in \X$ and
  $\max_{x \in \X}\norm{x-x^{\star}}_2 \leq \diam{\X}$ by definition,
  it follows from the convexity of $\X$ that $S(\X) \subseteq B(x^{\star},r)\cap \X$
  which directly implies that $\mu( B(x^{\star},r)\cap \X ) \geq \mu( S(\X))$.
  Moreover, as $S$ is a similarity transformation conserves
  the ratios of the volumes before/after transformation, we have that
  \[
    \frac{\mu( B(x^{\star}, r) \cap \X )}{\mu( \X)}
    \geq \frac{\mu(S(\X))}{\mu(\X)}
    = \frac{\mu( S( B(x^{\star}, \diam{\X}) ) )}{ \mu( B(x^{\star}, \diam{\X}) )}
    = \frac{ \mu( B(x^{\star},r ) )   }{ \mu( B(x^{\star},\diam{\X}) ) }
  \]
  which combined with the fact that
  $\forall r \geq 0$,
  $\mu(B(x^{\star},r))=\pi^{d/2}r^{d}/\Gamma(d/2+1)$
  where $\Gamma(\cdot)$ stands
  for the standard gamma function
  gives the result.
\end{proof}

\noindent The next lemma will be useful in order to control the volume of level sets
of a function with $(c_{\alpha}, \alpha)$-regular level sets.

\begin{lemma}
\label{lem:inclusion}
  Let $\X \subset \R^d$ be any compact and convex set and let
  $f\in \mathcal{C}^0(\X, \R)$ be any continuous function
  with ($c_{\alpha}, \alpha$)-regular level sets (Condition \ref{cond:levelset}).
  Then, for any radius $r \in (0, \max_{x \in \X} \norm{x^{\star}-x}_2)$,
  we have that
  \[
    \X \cap B(x^{\star}, (r/c_{\alpha})^{1+\alpha})
    \subseteq \{x \in \X: f(x) \geq \textstyle{\min_{x_r \in \mathcal{S}_r }f(x_r)} \}
    \subseteq  B(x^{\star}, c_{\alpha} \cdot r^{1/(1+\alpha)})
  \]
  where
  $\mathcal{S}_r= \{x \in \X: \norm{x^{\star}-x}_2=r \} $.
\end{lemma}

\begin{proof}\sloppy
  We start with the second inclusion.
  First, we show that for all $ y \in [ \min_{x \in \mathcal{S}_r}f(x), f(x^{\star})]$,
  there exists $x_y \in f^{-1}(y):= \{x \in \X: f(x)=y \} $
  such that $\norm{x^{\star}-x_y}\leq r$.
  Consider any $y \in [ \min_{x \in \mathcal{S}_r}f(x), f(x^{\star})]$,
  pick any $x_r \in \arg\min_{x \in \mathcal{S}_r}f(x)$
  and introduce the function $F: [0,1] \rightarrow \R$ defined by
  \[
    F: \lambda \mapsto f( (1-\lambda)x^{\star} + \lambda x_r),
  \]
  which returns the value of the function $f$ over the segment $[x^{\star}, x_r]$.
  As (i) $F$ is continuous, (ii) $F(0)=f(x^{\star})$, 
  (iii) $F(1)=\min_{x \in \mathcal{S}_r}f(x)$
  and (iv) $y \in [F(1), F(0)]$,
  it follows from the intermediate value theorem that
  there exists $\lambda_y\in [0,1]$ such that $F_{x_r}(\lambda_y)=y$.
  Hence there exists $x_y=\lambda_y x^{\star} +(1-\lambda_y)x_r \in f^{-1}(y)$
  such that $\rVert x^{\star}- x_y \lVert_2 \leq  \rVert x^{\star}-x_r \lVert_2=r$.
  Keeping in mind the previous statement,
  we may now prove the second inclusion.
  Assume by contradiction that there exists $x_y' \in f^{-1}(y)$ 
  such that $\rVert x^{\star}-x_y' \lVert_2>c_{\alpha}r^{1/(1+\alpha)}$.
  Then, it directly  from the definition of the maximum that
  $ \max_{x \in f^{-1}(y) }\lVert x^{\star}-x \rVert_2
  \geq \lVert x^{\star}-x_y' \rVert_2 > c_{\alpha}  r^{1/(1+\alpha)}$.
  However, since
  $ c_{\alpha} \cdot \min_{x \in f^{-1}(y)}\rVert x^{\star} - x \rVert^{1/(1+\alpha)}_2
  \leq c_{\alpha} \cdot \lVert x^{\star} - x_y \rVert^{1/(1+\alpha)}_2
  \leq c_{\alpha} \cdot  r^{1/(1+\alpha)}$ by definition of the minimum,
  we get the following contradiction by combining the previous statements with
  the regularity of the level set of the function:
  \[
    \max_{x \in f^{-1}(y) }\norm{x^{\star}-x}_2  \leq
    c_{\alpha} \cdot \min_{x \in f^{-1}(y)}\norm{x^{\star} - x}^{1/(1+\alpha)}_2
    < \max_{x \in f^{-1}(y) }\norm{x^{\star}-x}_2.
  \]
  As the previous contradiction holds for any
  $y \in [\min_{x \in S_r }f(x_r), f(x^{\star})]$,
  we deduce that  $\{x \in \X: f(x) \geq  \min_{x \in S_r }f(x_r) \}
  \subseteq  B(x^{\star}, c_{\alpha} \cdot r^{1/(1+\alpha)})$ which proves the second inclusion.
  
  ~\\
  \noindent We use similar arguments to prove the first inclusion.
  Suppose by contradiction
  that there exists $x' \in \X \cap B(x^{\star},(r/c_{\alpha})^{1+\alpha})$
  such that $f(x) < f(x_r)$ and
  introduce the function
  $F: \lambda \in [0,1] \mapsto f(  (1-\lambda)x^{\star} + \lambda x' )$.
  First, we know from the intermediate value theorem that there exists
  $x_r' \in f^{-1}(f(x_r))$ such that~$
  \rVert x^{\star} - x_r' \lVert_2 < (r/c_{\alpha})^{1+\alpha} $.
  Hence we have that
  $ c_{\alpha} \cdot \min_{x\in f^{-1}(f(x_r))}\lVert x^{\star}-x \rVert_2^{1/(1+\alpha)}
  \leq  c_{\alpha} \lVert x^{\star}-x_r' \rVert_2^{1/(1+\alpha)}
  < r$. However, as
  $\max_{x\in f^{-1}(f(x_r))} \lVert x^{\star}-x \rVert_2
  \geq \lVert x^{\star}-x_r \rVert_2 = r$,
  we get a similar contradiction as the one obtained previously
  which proves that the first inclusion.
\end{proof}

\subsection{Consistency and upper bound}

In this subsection, we provide the proofs of Proposition \ref{th:fasterprs},
Corollary \ref{coro:consistencyrank} and Theorem \ref{th:upperbound}.\\

\noindent {\bf Proof of Proposition \ref{th:fasterprs}.}
  The statement is proved by induction. Since
  $X_1 \sim \mathcal{U}(\X)$, the result directly holds for $n=1$.
  Assume now that the statement holds for a given $n \in \mathbb{N}^{\star}$ and
  let $\{X_i \}_{i=1}^{n+1}\sim$ {\sc RankOpt}$(n+1, f, \X, \Rank)$.
  As the result also trivially holds whenever $y  \notin \text{Im}(f)$,
  consider any $y \in \text{Im}(f) $
  and let $\X_{y}=\{ x \in \X: f(x) \geq y\}$ 
  be the corresponding level set.
  We start with the following decomposition:
  \begin{equation}
  \label{eq:prop11}
    \P\left(\max_{i =1\dots n+1}f(X_i) \geq y \right)
    = \P\left(\max_{i=1\dots n}f(X_i) \geq y \right)
    + \P\left(  \left\{ \max_{i=1\dots n}f(X_i) <y \right\}  \medcap \{X_{n+1} \in \X_y \} \right).
  \end{equation}

  \noindent 
  From Proposition \ref{prop:chain}, we know that
  $X_{n+1}|~\!\{X_i \}_{i=1}^n \sim \mathcal{U}(\X_n)$ 
  where the sampling area $\X_n$ 
  has a strictly positive Lebesgue measure whenever $\{\max_{i=1\dots n}f(X_i)$ $<y\}$.
  Hence conditioning upon $\{X_i \}_{i=1}^n $ gives that
  \begin{align*}
    \P\left( \left\{ \max_{i=1\dots n}f(X_i) <y \right\}
    \medcap \{X_{n+1} \in \X_y \}   \right)
    &=  \esp{ \indic{ \max_{i=1\dots n}f(X_i) <y }
    \cdot \P( X_{n+1} \in \X_y  |~\!\{X_i \}_{i=1}^n )} &  \\
    &= \mathbb{E} \left[   \indic{ \max_{i=1\dots n}f(X_i) <y  }
    \cdot
    \frac{ \mu( \X_n \cap \X_y )  }{\mu(\X_n)}  \right].
  \end{align*} 
  From Proposition \ref{prop:chain} again, we also know that 
  the sampling area $\X_n$ satisfies 
  $\X_y \subseteq \X_n$ and  $\X_n \subseteq \X$
  whenever $\{ \max_{i=1\dots n}f(X_i) < y \}$.
  Therefore
    \begin{align*}
    \P\left( \left\{ \max_{i=1\dots n}f(X_i) <y  \right\} 
    \medcap \{X_{n+1} \in \X_y \} \right)
    & \geq \frac{ \mu(  \X_y )  }{\mu(\X)}
    \cdot \esp{ \indic{ \max_{i=1\dots n}f(X_i) <y } } \\
    &= \frac{ \mu(  \X_y )  }{\mu(\X)} \cdot
    \left( 1 - \P\left(\max_{i=1\dots n}f(X_i) \geq y \right) \right).
  \end{align*}
  Finally, successively
  plugging the previous inequality into (\ref{eq:prop11}) and
  applying the induction assumption gives us that
      \begin{align*}
    \P\left(\max_{i=1\dots n+1}f(X_i) \geq y \right) 
    & \geq \P\left(\max_{i=1\dots n}f(X_i) \geq y \right)
    +
    \frac{\mu(\X_y)}{\mu(\X)} \cdot
    \left( 1 - \P\left(\max_{i=1\dots n}f(X_i) \geq y \right) \right) \\
    &\geq \P\left(\max_{i=1\dots n}f(X'_i) \geq y \right)
    +
    \frac{\mu(\X_y)}{\mu(\X)} \cdot
    \left( 1 - \P\left(\max_{i=1\dots n}f(X'_i) \geq y \right) \right)
  \end{align*}
  where $\{X_i' \}_{i=1}^{n+1} \iid \mathcal{U}(\X)$
  and the desired result follows by noticing that the right 
  hand term of the previous inequality 
  is equal to $\P( \max_{i=1\dots n+1}f(X'_i) \geq y )$.
\hfill\(\Box\)

~\\ 
\noindent Equipped with Proposition \ref{th:fasterprs}, we may now
easily prove the consistency property of the algorithm.\\
~\\
{\bf Proof of Corollary \ref{coro:consistencyrank}.}
  Pick any $\varepsilon >0$ and let $\X_{f^{\star}- \varepsilon}=
  \{x \in \X: f(x) \geq \max_{x \in \X}f(x) - \varepsilon \}$
  be the corresponding level set.
  By Proposition \ref{th:fasterprs}, we have that $\forall n \in \mathbb{N}^{\star}$,
  \[
  \P\left( f(X_{\hatin})< \max_{x \in \X}f(x) -\varepsilon \right)
    \leq \P\left(\max_{i=1 \ldots n} f(X'_i) < \max_{x \in \X}f(x)- \varepsilon \right)
    \text{~where~}\{X'_i \}_{i=1}^n \!\iid \mathcal{U}(\X).
  \]
  Therefore, using the fact that 
  $0<\mu(\X_{f^{\star}-\varepsilon})/\mu(\X) \leq 1$ by Condition \ref{cond:id}, 
  we directly get that
  \[
    \P\left( f(X_{\hatin})< \max_{x \in \X}f(x) -\varepsilon \right)~\!
    \leq \P\left(X'_1 \notin \X_{f^{\star}-\varepsilon} \right)^n
    = \left(1 - \frac{\mu( \X_{f^{\star}-\varepsilon} )}{\mu(\X)} \right)^n
    \underset{n \rightarrow \infty}{\longrightarrow} 0
  \]
  which proves the result.
\hfill\(\Box\)

~\\
\noindent We now turn to the proof of the upper bound.\\

\noindent {\bf Proof of Theorem \ref{th:upperbound}.}
  Note first that since
  $r_f \in \Rank \subseteq \Rank_{\infty}$
  is a continuous ranking rule, 
  we know from Proposition \ref{prop:rankingequivalence} that
  there exists a continuous function $h \in \mathcal{C}^{0}(\X, \R)$ 
  which shares the same ranking rule with $f$. 
  One can then consider, without loss of generality, that $f\in \mathcal{C}^{0}(\X, \R)$
  as all the arguments used in the proof only use function comparisons.
  Additionally, since the result trivially holds whenever
  the upper bound of the theorem, denoted here by $r_{\delta,n}$, satisfies
  $r_{\delta, n}$ $\geq$ $\max_{x \in \X} \norm{x-x^{\star}}_2$,
  we consider that
  $r_{\delta, n} < \max_{x \in \X} \norm{x-x^{\star}}_2$ which  
  also implies by the level set assumption
  that $\ln(1/\delta)<n$.
  Last, we  also set some notations, set
  $\mathcal{S}_{\delta, n} = 
  \{x \in \X: \norm{x^{\star} -x}_2 = ( r_{\delta, n}/c_{\alpha} )^{1+\alpha} \}$
  and let
  $R_{\delta,n} = ( (r_{\delta,n}/c_{\alpha} )^{1+\alpha}
  /c_{\alpha}  )^{1+\alpha}$.
  Equipped with these notations, we may now prove the result.
  By Lemma \ref{lem:inclusion}, we have that
  \[
   \P(\norm{X_{\hatin} - x^{\star}}_2 \leq r_{\delta, n})
   = \P( X_{\hatin} \in B(x^{\star},r_{\delta, n})) 
   \geq \P\left( f(X_{\hatin})
    \geq \textstyle{\min_{x \in \mathcal{S}_{\delta, n}}} f(x)  \right)
  \]
   which together with Proposition \ref{th:fasterprs} gives that
  \[
   \P(\norm{X_{\hatin} - x^{\star}}_2 \leq r_{\delta, n}) 
   \geq 
   \P\left( \textstyle \max_{i =1 \ldots n } f(X'_i)
    \geq \textstyle{\min_{x \in \mathcal{S}_{\delta, n}}} f(x)  \right)
  \]
  where $\{X'_i \}_{i=1}^n \iid \mathcal{U}(\X)$.
  Therefore, successively
  using independence and the second inclusion of Lemma \ref{lem:inclusion}
  gives that
  \[
    \P(\norm{X_{\hatin} - x^{\star}}_2 \leq r_{\delta, n})
    \geq \P \left( \bigcup_{i=1}^n 
    \{X'_i \in \X \cap B(x^{\star}, R_{\delta, n}) \}  \right)
    = 1- \left(1 - \frac{\mu(\X \cap B(x^{\star}, R_{\delta, n}) )}
    {\mu(\X)} \right)^n.
  \]
  Finally, as $R_{\delta, n}$ was defined so that Lemma \ref{lem:zab} ensures that
  $$
   \frac{\mu(\X \cap B(x^{\star}, R_{\delta,n}))}{\mu(\X)}
  \geq \left( \frac{R_{\delta, n}}{\diam{\X}}\right)^d
  = \frac{\ln(1/\delta)}{n},
  $$
  it follows that
  \[
    \P(\norm{X_{\hatin} - x^{\star}}_2 \leq r_{\delta, n})
    \geq 1-\left(1-  \frac{\ln(1/\delta)}{n} \right)^n
  \]
  which combined with the elementary inequality $1-x \leq e^{-x}$ gives the result.
\hfill\(\Box\)

\subsection{Lower bound}

In order to prove Theorem \ref{th:lowerbound},
we start by developing the full proof for Proposition \ref{prop:slowerpas}
and we provide two technical lemmas (Lemma \ref{lem: fasteruniform}
and Lemma \ref{prop:concentration})
that will used in the proof of the lower bound.

~\\
  {\bf Proof of Proposition \ref{prop:slowerpas}.}
    Again, the result is proved by induction.
    Since $X_1$ and $X_1^{\star}$ are both uniformly distributed over $\X$,
    the result directly holds for $n=1$.
    Assume now that the statement holds for a given $n \in \mathbb{N}^{\star}$
    and let $\{X_i \}_{i=1}^{n+1}\sim$ {\sc RankOpt}$(n+1,f,\X,\Rank)$.
    As the result also trivially holds whenever $y \notin  \text{Im}(f)$,
    consider any $y \in \text{Im}(f)$ and let
    $\X_y=\{x \in \X: f(x) \geq y \}$ be the corresponding level set.
    We start with a similar decomposition as the one used in 
    the proof of Proposition \ref{th:fasterprs}:
    \begin{align*}
      \P\left( \max_{i=1\dots n+1}f(X_i) \geq y \right)
      & = \mathbb{E}\left[ \indic{\max_{i=1\dots n}f(X_i)\geq y }
      + \frac{ \mu (\X_y \cap \X_n) }{\mu(\X_n)}
      \cdot \indic{\max_{i=1\dots n}f(X_i) < y }  \right].
    \end{align*}
    Observe now that if $\mu(\X_y) = 0$,
    then $\P\left( \max_{i=1\dots n+1}f(X_i) \geq y \right)$
    = $\P(X_1 \in \X_y ) = 0$ and the result directly holds.
    We thus consider the 
    case where $\mu(\X_y)>0$ and 
    we set some additional notations to clarify the proof:
    let $f(X_{\hatin}) = \max_{i=1\dots n}f(X_i)$
    and let $\X_{f(X_{\hatin})}=\{x \in \X: f(x) \geq  f(X_{\hatin})\}$.
    From Proposition \ref{prop:chain}, we know that on the event 
    $\{ f(X_{\hatin}) < y  \}$ the sampling area $\X_n$ satisfies both
    $\X_{f(X_{\hatin})} \subseteq  \X_n$ and  
    $\X_y  \subseteq \X_{f(X_{\hatin})}$.
    Therefore we have that
    \begin{align*}
      \P\left( \max_{i=1\dots n+1}f(X_i) \geq y \right)
      & \leq \esp{ \indic{f(X_{\hatin})\geq y }
      + \frac{ \mu (\X_y) }{\mu(\X_{f(X_{\hatin})})}
      \cdot \indic{f(X_{\hatin}) < y }  }
    \end{align*}
    which combined with the fact that for any random variable $X\in[0,1]$,
    $\esp{X}=\int_{0}^{1} \P(X \geq t)~ \text{d}t$ gives that
    \begin{align}
    \label{eq:slowerprs}
      \P\left( \max_{i=1\dots n+1}f(X_i) \geq y \right)
      & \leq
      \int_{0}^1 \P \left(
          \indic{f(X_{\hatin}) \geq y }
      + \frac{ \mu (\X_y) }{\mu(\X_{f(X_{\hatin})})}
      \cdot \indic{f(X_{\hatin}) < y }
      \geq t \right)~\text{d}t.
    \end{align}
    Now, observe that since the volume of the sampling area 
    always satisfies 
    $\mu(\X_{f(X_{\hatin}) }) \leq \mu(\X_n) \leq \mu(\X)$ by Proposition \ref{prop:chain},
    then (i)  the probability under the integral in (\ref{eq:slowerprs}) is equal to 1 
    whenever $t \leq \mu(\X_y)/\mu(\X)$ 
    and (ii) for any $ t > \mu(\X_y)/\mu(\X)$, the following events are equivalent:
    $$
    \left\{
    \indic{f(X_{\hatin}) \geq y }
	  + \frac{ \mu (\X_y) }{\mu(\X_{f(X_{\hatin})})}
	  \cdot \indic{f(X_{\hatin}) < y } \geq t
    \right\}
    =
    \left\{ \mu(\X_{f(X_{\hatin})}) \leq \frac{  \mu(\X_y)}{ t} \right\}.
    $$
    Therefore plugging the inequalities obtained in 
    (i) and (ii) into (\ref{eq:slowerprs}) gives us that
    \begin{align}
    \label{eq:slowerprs2}
    \P\left( \max_{i=1\dots n+1}f(X_i) \geq y \right)
    \leq \frac{\mu(\X_y)}{ \mu(\X) }
    + \int_{ \frac{\mu(\X_y)}{ \mu(\X) }}^1
    \P\left( \mu(\X_{f(X_{\hatin})}) \leq \frac{  \mu(\X_y)}{ t} \right)  \text{d}t.
    \end{align}
    We now turn to the analysis 
    of the probability under the integral in (\ref{eq:slowerprs2}).
    By successively using the continuity of the ranking rule induced by the unknown function
    and applying the induction assumption,
    we obtain 
    for any $t \in (\mu(\X_y) / \mu(\X),1)$
    the following bound:
      \begin{align}
      \label{eq:slowerprs3}
      \P\left( \mu(\X_{f(X_{\hatin})}) \leq  \frac{\mu(\X_y)}{t}  \right)
      & = \P\left( f(X_{\hatin})
      \geq \min\left\{y' \in \textrm{Im}(f): 
      \mu(\X_{y'}) \leq \frac{\mu(\X_y)}{t}  \right\}  \right) \nonumber \\
      & \leq \P\left( f(X^{\star}_n) 
      \geq \min\left\{y' \in \textrm{Im}(f): 
      \mu(\X_{y'}) \leq \frac{\mu(\X_y)}{t}  \right\}  \right) \nonumber \\
      & = \P\left( \mu(\X_{f(X^{\star}_n)}) \leq  \frac{\mu(\X_y)}{t}  \right)
    \end{align}
    where
    $\{X^{\star}_i \}_{i=1}^{n}$ 
    is a sequence of $n$ random variables
    distributed as Pure Adaptive Search indexed by $f$ over $\X$ and 
    $\X_{f(X^{\star}_n)} =\{x \in \X: f(x) \geq f(X^{\star}_n) \} $.
    Therefore, plugging (\ref{eq:slowerprs3}) into (\ref{eq:slowerprs2}) gives that
  \begin{equation*}
    \P\left( \max_{i=1\dots n+1}f(X_i) \geq y \right)
    \leq \frac{\mu(\X_y)}{ \mu(\X) }
    + \int_{ \frac{\mu(\X_y)}{ \mu(\X) }}^1
    \P\left( \mu(\X_{f(X^{\star}_n)}) \leq  \frac{\mu(\X_y)}{t}  \right) \text{d}t
    \end{equation*}
    and the desired result follows by noticing that
    the right hand term of the previous equation
    is equal to $\P( f(X^{\star}_{n+1}) \geq y )$ (which can be easily shown
    by reproducing the same steps as previously with 
    a sequence of $n+1$ random variables
    distributed as a Pure Adaptive Search).
  \hfill\(\Box\)

  ~\\
  \noindent The next lemma will be used in the proof of Theorem \ref{th:lowerbound}
  to control the volume of the level set of the highest value observed by a
  {\sc Pure Adaptive Search}.

  \begin{lemma}
  \label{lem: fasteruniform}
  Let $\X \subset \R^d$ be any compact and convex set with non-empty interior,
  let $f:\X \to \R$ be any function such that $r_f \in \Rank_{\infty}$
  and let $\{ X_i^{\star}\}_{i=1}^n$
  be a sequence of $n$ random variables distributed
  as a \textsl{\textsc{Pure Adaptive Search}} indexed by $f$ over $\X$.
  Then, for any $u \in [0,1]$, we have that
    \[
      \P\left(   \frac{\mu(\X^{\star}_n)}{\mu(\X)} \leq u \right)
      \leq \P \left( \prod_{i=1}^n U_i \leq u  \right)
    \]
    where $\X^{\star}_n := \{x \in \X: f(x) \geq f(X^{\star}_n) \}$  and
    $\{U_i \}_{i=1}^n\!\iid \mathcal{U}([0,1])$.
  \end{lemma}

\begin{proof} 
  Observe first that if
  $u^{\star}= \mu( \{x \in \X: f(x) \geq \max_{x \in \X}f(x) \}  ) /\mu(\X)
  >0$, then the result trivially holds for all  $u < u^{\star}$ 
  and $n \geq 1$.
  For simplicity, we thus consider that $u^{\star}=0$ and
  we set some notations:
  $\forall u \in[0,1]$, let
  $y_u := \min\{y \in \textrm{Im}(f):
  \mu(\{x \in \X: f(x) \geq y \}) \leq u \cdot \mu(\X) \}$
  and let $\X_{y_u}=\{x \in \X: f(x) \geq y_u \}$ be the corresponding level set.
  Keeping in mind that $\mu(\X_{y_u}) \leq u \cdot \mu(\X)$
  for all $u \in [0,1]$,
  we may now prove the result by induction.

  ~\\
   Set $n=1$, pick any $u \in [0,1]$
   and let $U_1 \sim \mathcal{U}([0,1])$.
   Since $X_1^{\star}\sim \mathcal{U}(\X)$ and
   $\P(U_1 \leq u) = u$,
   it directly follows that
  \[
    \P\left(  \frac{ \mu(\X^{\star}_{1}) }{ \mu(\X) }   \leq u \right)
    = \P( X^{\star}_1 \in  \X_{y_u})
    = \frac{\mu(\X_{y_u} )}{\mu(\X)} \leq u
    = \P(U_1 \leq  u )
  \]
  which proves the result for $n=1$.
  Assume now that the statement holds for a given
  $n \in \mathbb{N}^{\star}$.
  Fix any $u \in [0,1]$ and
  let $ \{ X_{i}^{\star} \}_{i=1}^{n+1}$
  be a sequence of $n+1$ random variables distributed as Pure Adaptive Search
  indexed by $f$ over $\X$.
  From definition \ref{def:PAS}, we know that
  $X_{n+1}^{\star}|~\! X^{\star}_n \sim \mathcal{U}(\X^{\star}_{n})$
  where $\X_n^{\star} = \{ x \in \X: f(x) \geq f(X_n^{\star})\}$.
  Therefore,  conditioning upon $X_{n}^{\star}$ gives that
  \[
    \P\left( \frac{\mu(\X^{\star}_{n+1})}{\mu(\X)} \leq u    \right)
    =\mathbb{E} \left[ \frac{ \mu( \X_{y_u} \cap \X^{\star}_n ) }
    {\mu(\X^{\star}_n)} \cdot \indic{ \mu(\X_n^{\star}) > \mu( \X_{y_u}) }   
    + \indic{ \mu(\X_n^{\star}) \leq  \mu( \X_{y_u}) }   \right].
  \]
  Since the level sets of the unknown function form a nested sequence,
  we know that the following events are equivalent
  $ \{ \mu(\X_n^{\star}) > \mu( \X_{y_u}) \} =\{\X^{\star}_n \subset \X_{y_u}\}$.
  Hence, using the convention $1/0 = +\infty$, we obtain that
  \[
    \P\left( \frac{\mu(\X^{\star}_{n+1})}{\mu(\X)} \leq u \right)
    = \mathbb{E} \left[ \min \left(1,  \frac{\mu(\X_{y_u})}
    { \mu(\X^{\star}_n)}  \right)  \right].
  \]
  Now, since for any random variable $U_{n+1} \sim \mathcal{U}([0,1])$
  independent of $Y \in [0,1]$,
  $\P(U_{n+1} \leq Y ~\!|~\! Y) = Y$, 
  we have that
   \begin{align*}
      \P\left( \frac{\mu(\X^{\star}_{n+1})}{\mu(\X)} \leq u \right)
      & = \mathbb{E} \left[ \P \left( U_{n+1} \leq
    \min \left(1,  \frac{\mu(\X_{y_u})}
    { \mu(\X^{\star}_n)}  \right)
    ~|~\!\mu(\X^{\star}_n) \right)   \right].
    \end{align*}
    Therefore using the independence of $U_{n+1}$ and $\{X_i^{\star} \}_{i=1}^n$ gives that
    \begin{align}
    \label{eq:lem_sloweruniform}
    \P\left( \frac{\mu(\X^{\star}_{n+1})}{\mu(\X)} \leq u \right)
      &= \P \left(  U_{n+1} \cdot \frac{\mu(\X^{\star}_n)}{\mu(\X)}
    \leq \frac{\mu(\X_{y_u})}{\mu(\X)}   \right).
    \end{align}
  Finally, successively
  using the fact that  $\mu(\X_{y_u}) \leq u \cdot \mu(\X)$ and
  plugging the induction assumption into (\ref{eq:lem_sloweruniform})
  gives that
  \[
    \P\left( \frac{\mu(\X^{\star}_{n+1})}{\mu(\X)} \leq u    \right)
    \leq \P \left( U_{n+1} \cdot \frac{\mu(\X^{\star}_n)}{\mu(\X)} \leq u \right)
    \leq \P \left( \prod_{i=1}^{n+1} U_i \leq u \right)
  \]
  where $\{U_i \}_{i=1}^{n+1} \iid \mathcal{U}([0,1])$
  and the proof is complete.
\end{proof}

\noindent 
  The concentration inequality provided in the next lemma
  will be important in order to control the volume of the level set of 
  the highest value observed by a {\sc Pure Adaptive Search}.

\begin{lemma}
\label{prop:concentration}
  Let $\{ U_i\}_{i=1}^n$ be a sequence of $n$ independent
  copies of $U\sim \mathcal{U}([0,1])$.
  Then, for any $\delta \in (0,1)$, we have that
  $
    \P \left( \prod_{i=1}^n U_i
    < \delta \cdot e^{-n-\sqrt{ 2 n \ln(1/\delta) }} \right) < \delta.
  $
\end{lemma}

\begin{proof} Taking the logarithm on both sides gives that
  $ \prod_{i=1}^n U_i < \delta \cdot e^{-n-\sqrt{ 2 n \ln(1/\delta) }}$
  if and only if
  $\sum_{i=1}^n - \ln(U_i)  > n+  \sqrt{2 n \ln(1/\delta)} +\ln(1/\delta)$.
  As $U_i \sim \mathcal{U}([0,1])$ for $ i \leq n$,
  we have that $-\ln(U_i)\sim \textrm{Exp}(1)$ which combined with 
  independence gives that
  $\sum_{i=1}^n -\ln(U_i) \sim$ Gamma$(n,1)$.
  Therefore, the desired result follows from the application of a standard
  concentration inequality for sub-gamma random variables
  (see  Chapter 2.4 in \cite{boucheron2013concentration}).
\end{proof}

\noindent 
Equipped with Proposition \ref{prop:slowerpas}, Lemma \ref{lem: fasteruniform}
and Lemma \ref{prop:concentration},
we may now prove the lower bound.

~\\
{\bf Proof of Theorem \ref{th:lowerbound}.}
  As in the proof of Theorem \ref{th:upperbound},
  we may consider without loss of generality that
  $f\in \mathcal{C}^{0}(\X, \R)$.
  Now, fix any $\delta \in (0,1)$, let $ r_{\delta, n}$ be the lower bound of the theorem,
  set $\mathcal{S}_{\delta, n }= \{ x \in \X: \norm{x^{\star} - x }_2 
  = c_{\alpha} r_{\delta, n}^{1/(1+\alpha)} \}$ and let
  $
  R_{\delta, n}
  = \rad{\X} \delta^{1/d} \exp(-(n+\sqrt{2n \ln{(1/\delta)}})/d).
  $ 
  From the first inclusion of Lemma \ref{lem:inclusion}, we have that
  $$
  \P( \norm{X_{\hatin} - x^{\star}}_2 \leq r_{\delta, n} )
   = \P( X_{\hatin} \in B(x^{\star}, r_{\delta, n}) \cap \X )
    \leq \P\left( f(X_{\hatin}) \geq
    \min_{x\in \mathcal{S}_{\delta, n}}f(x) \right)
  $$
  which together with Proposition \ref{prop:slowerpas} gives that
  $$
    \P( \norm{X_{\hatin} - x^{\star}}_2 \leq r_{\delta, n} )
   \leq \P\left( f(X^{\star}_{n}) \geq
    \min_{x\in \mathcal{S}_{\delta, n}}f(x) \right)
  $$
  where $\{X_i^{\star} \}_{i=1}^n$
  is a sequence of $n$ random variables distributed as {\sc Pure Adaptive Search}
  indexed by $f$ over $\X$.
  Now, observing that if $\X^{\star}_n=\{x \in \X: f(x) \geq f(X^{\star}_n) \}$
  denotes the level set of $f(X^{\star}_n)$, then
  the following events are equivalent:
  \[
   \left\{f(X^{\star}_n) \geq \min_{x\in \mathcal{S}_{\delta, n}}f(x) \right\} = 
  \left\{ \mu(\X_n^{\star}) 
  \leq \mu\left(\left\{x \in \X: f(x) 
  \geq \min_{x\in \mathcal{S}_{\delta, n}}f(x) \right\}\right) \right\},
  \]
  we obtain by applying the second inclusion of Lemma \ref{lem:inclusion} that
  \begin{align*}
    \P( \norm{X_{\hatin} - x^{\star}}_2 \leq r_{\delta, n} )
    &\leq \P\left(  \mu(\X_n^{\star})  \leq
    \mu\left(\left\{ x\in \X:  f(x) \geq  \min_{x\in S_{\delta, n}}f(x)\right\}
    \right) \right)\\
    &\leq \P\left(  \frac{\mu(\X_n^{\star})}{\mu(\X)}
    \leq \frac{\mu(B(x^{\star},R_{\delta,n}))}{\mu(\X)}  \right).
  \end{align*}
  As $\rad{\X}>0$ is assumed to be finite,
  we know that there exists $x \in \X$  such that $B(x, \rad{\X}) \subseteq \X$
  which implies that
  $\mu(\X) \geq \mu(B(x,\rad{\X})) = \pi^{d/2} \rad{\X}^d/\Gamma(d/2+1)$.
  Hence, we deduce that
  \[
    \P( \norm{X_{\hatin} - x^{\star}}_2 \leq r_{\delta, n} )
    \leq  \P\left(   \frac{\mu(\X_n^{\star})}{\mu(\X)}
    \leq \frac{\mu(B(x^{\star}, R_{\delta,n})) }{\mu(B(x, \rad{\X}))} \right)
    = \P\left(  \frac{\mu(\X_n^{\star})}{\mu(\X)}
    \leq \left( \frac{R_{\delta, n}}{\rad{\X}} \right)^d \right).
  \]
  Finally, as $R_{\delta,n}$ was defined so that
  $(R_{\delta,n}/\rad{\X} )^d$ =$ \delta $ $\cdot \exp {(-n-\sqrt{2n \ln(1/\delta)})}$,
  we obtain that
  \begin{align*}
    \P( \norm{X_{\hatin} - x^{\star}}_2 \leq r_{\delta, n} )
    & \leq \P \left( \prod_{i=1}^n U_i \leq \delta \cdot e^{-n-\sqrt{2n \ln(1/\delta)}} \right)
  \end{align*}
  by using Lemma \ref{lem: fasteruniform} 
  where $\{U_i \}_{i=1}^n \iid \mathcal{U}([0,1])$
  and the desired result naturally follows from 
  the application of the 
  concentration inequality of Lemma \ref{prop:concentration}.
\hfill\(\Box\)

\section{Analysis of the \textbf{\textsc{AdaRankOpt}} algorithm}

We develop here the proofs of Proposition \ref{prop:cons_adarank},
Proposition \ref{prop:model} and Theorem \ref{th: upper_ada}.
For convenience, we start recalling the definition of the sequence of 
evaluation points generated by the algorithm.

\begin{definition} {\sc (AdaRankOpt process)}
\label{prop:adarankopt}
  Pick any $p \in (0,1)$, 
  let $\{ \Rank_k\}_{k \in \mathbb{N}^{\star}}$ be any
  sequence of nested ranking structures,
  let $\X \subset \R^d$ be any compact and convex set
  with non-empty interior
  and let
  $f: \X \rightarrow \R$ be any function such that $r_f \in \Rank_{\infty}$.
  We say that a sequence $\{X_i \}_{i=1}^n$ is distributed 
  as a {\sc AdaRankOpt}$(n,$ $f$, $\X$, $p$, $\{ \Rank_k\}_{k \in \mathbb{N}^{\star}})$
  process if it follows the process defined by:
  \begin{equation*}
    \begin{cases}
      X_1 \sim \mathcal{U}(\X) \\
      X_{t+1}| ~B_{t+1}, \{X_i \}_{i=1}^t \sim B_{t+1}\cdot \mathcal{U}(\X)
      +(1-B_{t+1}) \cdot \mathcal{U}(\X_t) \textrm{~~~~~~~~~} \forall t \in \{1 \ldots n-1 \}
    \end{cases}
  \end{equation*}
  where at each step $t \geq 1$,
  $B_{t+1}$ is a Bernoulli  random variable of parameter $p$
  independent of $\{(X_i , B_i)\}_{i=1}^t$,
  $\X_t := \{x \in \X: \exists r \in \Rank_{\widehat{k}_t},~s.t.~r(x,X_{\hatit})\geq 0 \}$
  and $\hatit$ and $\widehat{k}_t$ are defined as in the algorithm.
\end{definition}

\noindent The proof of the consistency property of the algorithm
is straightforward.\\

\noindent {\bf Proof of Proposition \ref{prop:cons_adarank}.}
  Pick any $\varepsilon > 0$ and let
  $\X_{f^{\star}-\varepsilon}
  = \{ x \in \X: f(x) \geq \max_{x \in \X}f(x) - \varepsilon\}$
  be the corresponding level set.
  Since $p \in (0,1)$ and
  $0< \mu(\X_{f^{\star} - \epsilon}) / \mu(\X) \leq 1$ 
  by Condition \ref{cond:id},
  it is sufficient to show that the following holds true to prove the result:
  \begin{align}
  \label{eq:consada}
  \forall n \in \mathbb{N}^{\star},~
    \P\left( f(X_{\hatin}) <
     \max_{x \in \X}f(x) -\varepsilon \right)
    &\leq \left( 1 - p \cdot\frac{\mu(\X_{f^{\star}-\varepsilon})}{\mu(\X) } \right)^n.
  \end{align}
  We prove (\ref{eq:consada}) by induction.
  As $\hatiun=1$, $X_1 \sim \mathcal{U}(\X)$ and $0<p\leq 1$, we directly get that
  \begin{align*}
    \P\left( f(X_{\hatiun}) <
     \max_{x \in \X}f(x) -\varepsilon \right)
    &= \P(X_1 \notin \X_{f^{\star}-\varepsilon})
    = \left( 1- \frac{\mu(\X_{f^{\star}-\varepsilon})}{\mu(\X) } \right)
    \leq \left( 1 - p \cdot\frac{\mu(\X_{f^{\star}-\varepsilon})}{\mu(\X) } \right)
  \end{align*}
  which proves the result for $n=1$.
  Assume now that (\ref{eq:consada}) holds for a given $n \in \mathbb{N}^{\star}$,
  let $\{X_i \}_{i=1}^{n+1} \sim$
  {\sc AdaRankOpt}$(n,$ $f$, $\X$, $p$, $\{ \Rank_k\}_{k \in \mathbb{N}^{\star}})$
  and consider following decomposition:
  \begin{align*}
    \P\left(  \max_{i=1 \dots n+1}f(X_i) <
    \max_{x \in \X}f(x) -\varepsilon \right)
    &= \mathbb{E}\left[  \P(X_{n+1}\notin \X_{f^{\star}-\varepsilon} |~\!\{ X_i \}_{i=1}^n)
    \cdot \indic{ \bigcap_{i=1}^n \{X_i \notin \X_{f^{\star}-\varepsilon} \} }   \right].
  \end{align*}
  As Definition \ref{prop:adarankopt} guarantees that
  $\forall (x_1, \dots, x_n ) \in \X^n$,
    \begin{align*}
\P\left( X_{n+1} \notin \X_{f^{\star} - \epsilon} |~\! \bigcup_{i=1}^n \{ X_i = x_i \} \right)
     \leq  ~~~~~~~~~~~~~~~~~~~~~~~~~~~~~~~~~
     ~~~~~~~~~~~~~~~~~~~~~~~~~~~~~~~~~
     ~~~~~&
     ~~~~~~~~~~~~~~~~~~~~~~~~~~~~~~~~~\\
     1 -  \P\left( X_{n+1}  \in \X_{f^{\star} - \epsilon} 
     |~\!B_{n+1}=1, \bigcup_{i=1}^n \{ X_i = x_i \} \right) 
     \times \P\left( B_{n+1}=1~\!|~\!\bigcup_{i=1}^n \{ X_i = x_i \} \right)\\
    = 1 - p \cdot \frac{\mu( \X_{f^{\star} -\varepsilon  } )}{ \mu(\X)}
  \end{align*}
  it follows that
  \begin{align*}
   \P\left( \max_{i=1 \dots n+1}f(X_i) <  \max_{x \in \X}f(x) -\varepsilon \right)
    & \leq \left( 1 - p \cdot\frac{\mu(\X_{f^{\star}-\varepsilon})}{\mu(\X) } \right)
    \times \P\left( \bigcap_{i=1}^n \{ X_i \notin \X_{f^{\star}-\varepsilon} \} \right)
  \end{align*}
  which combined with the induction assumption proves (\ref{eq:consada}).
\hfill\(\Box\)

~\\
We now prove the stopping time upper bound.\\

\noindent {\bf Proof of Proposition \ref{prop:model}.}
  Let $\{(X_i,B_i)\}_{i \in \mathbb{N}^{\star}}$ be
  the sequence of random variables defined in the {\sc AdaRankOpt} algorithm.
  Fix any $\delta \in (0,1)$ and set
  $n'_{\delta} = \floor{p\cdot n_{\delta}  -\sqrt{ n_{\delta} \log(2/\delta)/2}}$
  where $n_{\delta}= \floor{10 \cdot (K+ \log(2/\delta))/(p
  \cdot \inf_{r\in \Rank_{N^{\star}-1}}L(r)^2)}$ 
  denotes the integer part of the upper bound.
  First, observe that since 
  $
   \mathcal{R}_1 \subset \mathcal{R}_2 \subset \dots \subset \mathcal{R}_{\infty}
  $
  forms a nested sequence,
  \[
  \min_{r \in \Rank_1}L_{n_{\delta}}(r) \leq \min_{r \in \Rank_2}L_{n_{\delta}}(r)
  \leq \dots \leq \min_{r \in \Rank_{k^{\star} \! -1}}L_{n_{\delta}}(r)
  \]
  where $L_{n_{\delta}}$ denotes the empirical ranking loss
  taken over the first $n_{\delta}$ samples  $\{X_i \}_{i=1}^{n_{\delta}}$.
  One might then start with the following decomposition:
  \begin{align}
  \label{eq:stop1}
    \P(\tau_{k^{\star}} \leq n_{\delta})
    &=\P \left( \min_{r \in \Rank_{k^{\star} \! -1  }}L_{n_{\delta}}(r) >0  \right) 
    \nonumber \\
    & \geq \P \left( \min_{r \in \Rank_{k^{\star} \! -1  }}L_{n_{\delta}}(r) >0  ~\!\mid~\!
     \sum_{i=1}^{n_{\delta}} B_i \geq n_{\delta}' \right)
     \times \P\left( \sum_{i=1}^{n_{\delta}} B_i \geq n_{\delta}'  \right).
  \end{align}
  We now focus on the first term of the right hand side of (\ref{eq:stop1})
  and we start to
  lower bound the empirical risk by only keeping the first $n_{\delta}'$
  i.i.d.~exploratory samples:
  \begin{equation}
  \label{eq:model1}
    L_{n_{\delta}}(r)
    \geq \frac{2}{n_{\delta}(n_{\delta}-1)} \sum_{1 \leq i < j \leq n_{\delta}}
    \indic{ r(X_i,X_j) \neq  r_f(X_i,X_j) }
    \times \indic{(i,j)\in I^2}
  \end{equation}
  where $I=\{i \leq n_{\delta}: B_i=1 ~\textrm{and}~
  \sum_{j=1}^i B_i \leq n_{\delta}' \}$.
  By definition \ref{prop:adarankopt}, we know 
  that conditioned upon $|I|$, the sequence
  $\{X_i \}_{i \in I}| |I|$ is a sequence
  of $|I|$ independent random variables uniformly distributed over $\X$.
  Therefore, conditioned upon the event $\{ \sum_{i=1}^{n_{\delta}} B_i \geq n'_{\delta} \}
  =\{|I|=n_{\delta}' \}$,
  the right hand term of (\ref{eq:model1}) has the same distribution as
  \[
    \frac{2}{n_{\delta}(n_{\delta}-1)}
    \sum_{1 \leq i < j \leq n_{\delta'}} \indic{ r(X'_i,X'_j) \neq  r_f(X'_i,X'_j) }
    \propto L_{n_{\delta}'}(r)
  \]
  where the sequence $\{X'_i \}_{i=1}^{n_{\delta}'} \iid \mathcal{U}(\X)$ 
  is independent of
  $\{B_i \}_{i=1}^{n_{\delta}}$.
  Hence, we have that
  \begin{equation*}
    \P \left( \min_{r \in \Rank_{k^{\star} \! -1  }}L_{n_{\delta}}(r) >0  ~\!\mid~\!
     \sum_{i=1}^{n_{\delta}} B_i \geq n_{\delta}' \right)
    \geq \P \left( \min_{r \in \Rank_{k^{\star}-1}} L_{n_{\delta}'}(r) >0 \right)
  \end{equation*}
  where $L_{n_{\delta}'}$  
  denotes the empirical ranking loss is taken over $\{X'_i \}_{i=1}^{n_{\delta}'}$.
  Now, by slightly adapting the generalization bound on bipartite ranking rules
  of \cite{clemenccon2008ranking}
  ({\it i.e.,~}see the proof of Corollary 3 in Section 3 therein) we obtain
  that with probability at least $1- \delta/2$,
    \[
      \sup_{ r \in \Rank_{k^{\star}\!-1}  }
    \abs{L_{n'_{\delta}}(r) - L(r)} \leq 2 \esp{R_{n'_{\delta}}(\Rank_{k^{\star}\!-1})}
    +2 \sqrt{\frac{\log(2/\delta)}{n'_{\delta}-1}},
    \]
   which combined with the fact that 
   $\esp{R_{n'_{\delta}}(\Rank_{k^{\star}\!-1})} \leq \sqrt{K/n}$ gives that
    \[
    \min_{r \in \Rank_{k^{\star}\!-1}} L_{n'_{\delta}}(r)
    \geq \inf_{r \in \Rank_{k^{\star}\!-1}}L(r) - 2\sqrt{ \frac{K}{n'_{\delta}} }
    -2  \sqrt{ \frac{\log(2/\delta)}{n'_{\delta}-1} }.
  \]
  Finally, as $n_{\delta}'$ and $n_{\delta}$
  were defined (with express purpose) so that (i)
  the right hand term of the previous inequality is strictly positive
  and so that (ii) Hoeffding's inequality ensures that
    \[
  \P\left(\sum_{i=1}^{n_{\delta}} B_i \geq n_{\delta}' \right) \geq 1-\delta/2,
  \]
  we deduce from (\ref{eq:stop1}) that
  $
   \P( \tau_{k^{\star}} \leq n_{\delta})
   \geq (1 -\delta/2)^2 \geq 1- \delta
  $
  and the proof is complete.
\hfill\(\Box\)

~\\
Theorem \ref{th: upper_ada} is obtained by combining
the upper bounds of Proposition \ref{prop:model} and Theorem \ref{th:upperbound}.\\
~\\
{\bf Proof of Theorem \ref{th: upper_ada}.}
  Fix $\delta \in (0,1)$, let
  $n_{\delta/2} = \floor{  10(K + \ln(4/\delta))/
  ( p \cdot \inf_{r \in \Rank_{k^{\star}-1}} L(r)^2 )  }$
  be the integer part of the upper bound of Proposition \ref{prop:model}
  (set with probability $1-\delta/2$)
  and let $r_{\delta/2,n}$ be the upper bound of the Theorem \ref{th:upperbound}
  (also set with probability $1-\delta/2$).
  We use the following decomposition:
    \begin{equation}
    \label{eq:upper_ada}
    \P( \norm{X_{\hatin}-x^{\star} }_2  \leq r_{\delta/2, n}) \geq
    \P(  \norm{X_{\hatin}-x^{\star} }_2  \leq r_{\delta/2, n}
    \!~|~\! \tau_{k^{\star}} < n_{\delta/2} )
    \times \P( \tau_{k^{\star}} < n_{\delta/2}).
  \end{equation}
  First, as on the event $\{ \tau_{k^{\star}} < n_{\delta/2} \}
  = \bigcap_{t \geq n_{\delta/2}} \{ \widehat{k}_t = k^{\star} \}$
  the smallest ranking structure $\Rank_{k^{\star}}$ containing the true ranking rule
  $r_f$ is identified 
  for all $t \geq n_{\delta/2}$,
  one can easily check that 
  \[
  \P(  \norm{X_{\hatin}-x^{\star} }_2  \leq r_{\delta/2, n}
  \!~|~\! \tau_{k^{\star}} \leq r_{\delta, n})
  \geq 1 -\delta/2
  \]
  by  reproducing the same steps as in Theorem \ref{th:upperbound}'s proof
  with the last $n- n_{\delta/2}$ samples.
  Second, as Proposition \ref{prop:model} also
  guarantees that $\P(  \tau_{k^{\star}} < n_{\delta/2} )\geq 1- \delta/2$,
  we then obtain from (\ref{eq:upper_ada}) that
  $\P( \norm{X_{\hatin}-x^{\star} }_2  \leq r_{\delta/2, n})
  \geq (1- \delta/2)^2 \geq 1 - \delta$. Hence, for all $n> n_{\delta/2}$,
  we have with probability at least $1-\delta$,
  \begin{align*}
   \norm{X_{\hatin} - x^{\star}}_2 & \leq C_1 \cdot \left( \frac{\ln(2/\delta)}{n - n_{\delta/2}}
   \right)^{\frac{1}{d(1+\alpha)^2}} \\
   & = C_1 \cdot \left( 1+ \frac{n_{\delta/2}}{n-n_{\delta/2}}\right)^{\frac{1}{d(1+\alpha)^2}}  \cdot \left( \frac{\ln(2/\delta)}{n}
   \right)^{\frac{1}{d(1+\alpha)^2}} \\
   & \leq C_1 \cdot \left( \frac{11(K+\ln(4/\delta)) }{p\inf_{r \in \Rank_{k^{\star}-1}}L(r)^2 }
   \right) \cdot \left( \frac{\ln(2/\delta)}{n} \right)^{\frac{1}{d(1+\alpha)^2}} 
  \end{align*}
  and the proof is complete by noticing that the right hand term 
  of the previous inequality is superior or equal to $\diam{\X}$ whenever $n \leq n_{\delta/2}$.
\hfill\(\Box\)

\section{Proofs of the implementation details}

We state here the proofs of Proposition \ref{prop:binary},
Lemma \ref{lem:zero}, Corollary \ref{coro:LP}, 
Proposition \ref{prop:binary_cvx} and Proposition \ref{prop:lp_cvx}.

\subsection{Polynomial and sinusoidal ranking rules}

We start with the proofs of Proposition \ref{prop:binary},
Lemma \ref{lem:zero} and Proposition \ref{prop:lp_cvx}.

~\\
{\bf Proof of Proposition \ref{prop:binary}}
  $(\Rightarrow)$ Assume that there exists $r \in \Rank_{ \mathcal{P}(k)}$
  such that $L_{t+1}(r)=0$. By definition of $\Rank_{ \mathcal{P}(k)}$,
  we know that there exists a polynomial function $f_r$ of degree $k$
  such that $\forall (x,x')\in \X^2$, $r(x,x')= \sgn{f(x) -f(x')}$.
  Moreover, as $f_r \in \mathcal{P}_k(\X, \R)$, we also know that
  there exists $(\omega_r,c_r) \in \R^{\normalfont{\text{dim}}({\Phi_k})} \times \R$
  such that $\forall x \in \R$, $f_r(x) = \inner{\omega_r, \Phi_k(x)}+c_r $.
  Therefore, putting the previous statements altogether
  with the fact that $L_{t+1}(r)=0$
  gives that
  $\forall i \leq t$,
  \[
   1 = r(X_{(i+1)},X_{(i)})
   = \sgn{f_r(X_{(i+1)}) -f_r(X_{(i)})}
   = \sgn{\inner{\omega_r, \Phi_k(X_{(i+1)})- \Phi_k(X_{(i)})}}
  \]
  and we thus deduce that there exists 
  $\omega=\omega_r \in \R^{\normalfont{\text{dim}}({\Phi_k})}$
  such that $\forall i \leq t$,
  $\inner{\omega,\Phi_k(X_{(i+1)}) - \Phi_k(X_{(i)}  ) } > 0$
  
  ~\\
  $( \Leftarrow)$ Assume now that there exists 
  $\omega \in \R^{\normalfont{\text{dim}}({\Phi_k})}$
  such that $\forall i \in \{1 \ldots t \}$,
  $\inner{\omega,\Phi_k(X_{(i+1)}) - \Phi_k(X_{(i)}  ) } > 0$
  and introduce the polynomial function  of degree $k$
  defined by
  $f_{\omega}: x \mapsto \inner{\omega, \Phi_k(x)  } +c$
  where $c\geq0$ is any arbitrary constant.
  Now, if $r_{f_{\omega}}$ denotes the polynomial ranking rule induced
  by $f_{\omega}$,
  we obtain from the first assumption that $\forall i \leq  t$,
  \[
   r_{f_{\omega}}(X_{(i+1)}, X_{(i)})
   = \sgn{f_{\omega}(X_{(i+1)}) -  f_{\omega}(X_{(i)}) }
   = \sgn{ \inner{\omega, \Phi_k(X_{i+1})- \Phi_k(X_{(i)})} }
   = 1.
  \]
  Hence $L_{t+1}(r)=0$ and we deduce that there exists
  $r = r_{f_w} \in \Rank_{\mathcal{P}_k}$ such that $L_{t+1}(r)=0$.
\hfill\(\Box\)

~\\
{\bf Proof of Lemma \ref{lem:zero}.}
  Observe first that for all $i \leq t$,
  $Y_i \cdot \inner{\omega, X_i}> 0 \Leftrightarrow 
  \inner{\omega, Y_i \cdot X_i}>0$.
  One can then consider without loss of generality that 
  $Y_i=1$ for all $i\leq t$, by replacing $X_i$ with $Y_i \cdot X_i$.\\

  \noindent $(\Rightarrow)$ Assume that there exists $\omega \in \R^d$ such that
  $\forall i \in \{1 \ldots t \}$, $\inner{\omega,X_i}>0$. 
  If $\vec{0} \in \textsc{CH}\{X_i \}_{i=1}^t$,
  this would mean
  that there exists $(\lambda_1, \ldots , \lambda_t ) \in \R^{t}$
  such that (i) $\vec{0} = \sum_{i=1}^t \lambda_i \cdot X_i$, (ii)
  $\sum_{i=1}^n \lambda_i=1$ and (iii) $\lambda_i\geq 0$, $i = 1 \dots t$
  and it would give us to the following contradiction:
  \[
    0= \inner{\omega, \vec{0}} = \sum_{i=1}^t \lambda_i \cdot \inner{\omega, X_i}>0.
  \]
  Hence $\vec{0} \notin \textsc{CH}\{X_i \}_{i=1}^t$.\\
  
  \noindent $(\Leftarrow)$ Assume now that $\vec{0}
  \notin \textsc{CH}\{X_i \}_{i=1}^t$.
  Since $t$ and $d$ are finite,
  $\textsc{CH}\{X_i \}_{i=1}^t$ is a closed, compact and convex set and thus
  $\min_{x \in \textsc{CH}\{X_i \}_{i=1}^t} \lVert x \rVert_2=d_{\min}$
  exists and the condition
  $\vec{0} \notin \textsc{CH}\{X_i \}_{i=1}^t$ implies that $d_{\min}>0$.
  Now, let $x_{d} \in \textsc{CH}\{X_i \}_{i=1}^t$ be
  the (unique) point of the convex hull
  which satisfies $\lVert x_{d} \rVert_2= d_{\min}$.
  We now prove by contradiction that $\forall x \in \textsc{CH}\{X_i \}_{i=1}^t$,
  $\inner{x,x_d}\geq {d^2_{\min}}$.
  Suppose that there exists $x \in \textsc{CH}\{X_i \}_{i=1}^t$
  such that $\inner{x,x_d}< {d_{\min}}^2$.
  First, we know from the convexity of the convex hull that
  the whole line
  $L= (x, x_{d})$ also belongs to $\textsc{CH}\{X_i \}_{i=1}^t$.
  However,  since $\lVert x_{d} \rVert_2={d_{\min}}$ and
  $\inner{x,x_d}<\lVert x_d \rVert^2$,
  the line $L$ is not tangent to the ball $B(\vec{0}, {d_{\min}} )$ and intersects it.
  Therefore, we deduce that there necessarily exists
  $x' \in L \cap B(\vec{0},{d_{\min}})$ such that $\lVert x' \rVert_2 < {d_{\min}}$.
  Nonetheless, as $x' \in L \subseteq \textsc{CH}\{X_i \}_{i=1}^t$ also 
  belongs to the convex hull,
  we obtain the following contradiction:
  \[
    \min_{x \in \textsc{CH}\{X_i \}_{i=1}^t} \norm{x}_2 \leq \lVert x' \rVert_2
    <  d_{\min} =\min_{x \in \textsc{CH}\{X_i \}_{i=1}^t} \norm{x}_2
  \]
  and we deduce that $\forall x \in \textsc{CH}\{X_i \}_{i=1}^t$,
  $\inner{x_d, x}\geq d_{\min}> 0$. Finally, as
  $\{X_i \}_{i=1}^t \in \textsc{CH}\{X_i \}_{i=1}^t$, 
  it directly follows that there exists $\omega=x_d \in \R^d$
  such that $\forall i \in \{ 1 \ldots t\}$, $\inner{\omega, X_i}>0$
  and the proof is complete.
\hfill\(\Box\)

~\\
Corollary \ref{coro:LP} is obtained by combining
Proposition \ref{prop:binary} with Lemma \ref{lem:zero}.

~\\
{\bf Proof of Corollary \ref{coro:LP}} From Proposition \ref{prop:binary},
we have the following equivalence: 
  \begin{align*}
    \min_{r \in \Rank_{\mathcal{P}_k}  }L_{t+1}(r) = 0 ~~
    &   \Leftrightarrow
    ~~\exists \omega \in \R^{\normalfont{\text{dim}}({\phi_k})} \textrm{~s.t.~}
    \inner{\omega, \Phi_k(X_{(i+1)}) - \Phi_k(X_{(i)}) }>0,
    ~\forall i \in \{ 1 \ldots t\}
  \end{align*}
  which combined with Lemma \ref{lem:zero} gives
    \begin{align*}
    \min_{r \in \Rank_{\mathcal{P}_k}  }L_{t+1}(r) = 0 ~~
    & \Leftrightarrow
    ~~ \vec{0} \notin \textsc{CH}\{ ( \Phi_k(X_{(i+1)}) - \Phi_k(X_{(i)}) ) \}_{i=1}^t.
  \end{align*}
  \sloppy
  In addition, we know from the vertex representation of convex hulls 
  that
  $\vec{0} \notin \textsc{CH}\{ (\Phi_k(X_{(i+1)}) - \Phi_k(X_{(i)}) ) \}_{i=1}^t$
  if and only if there does not exist any
  $\lambda=(\lambda_1, \ldots, \lambda_t) \in \R^t$ such that
  (i) $\sum_{i=1}^t \lambda_i$ $( \Phi_k(X_{(i+1)})$ $- \Phi_k(X_{(i)}) ) = \vec{0}$,
  (ii) $\sum_{i=1}^t \lambda_i=1$ and (iii) 
  $\lambda_i \geq 0$, $i=1,\dots, t$
  and therefore putting those constraints (i), (ii) and (iii)
  into matrix form leads us to the desired equivalence:
  \begin{align*}
  \min_{r \in \Rank_{\mathcal{P}_k}  }L_{t+1}(r) = 0 ~~
  & \Leftrightarrow
  ~~\left\{\lambda \in \R^{t}: ~\normalfont{\text{M}_{t}^{\Phi_k}} 
      \lambda^{\mathsf{T}} =\vec{0},
      ~\inner{\vec{1},\lambda} =1,~ \lambda \succeq \vec{0} \right\} = \emptyset
  \end{align*}
  where $\normalfont{\text{M}^{\Phi_k}_{k}}$ is the matrix defined in the corollary.
\hfill\(\Box\)

\subsection{Convex ranking rules}

In this subsection, we provide the proofs for Proposition \ref{prop:binary_cvx} 
and Proposition \ref{prop:lp_cvx}.

~\\
{\bf Proof of Proposition \ref{prop:binary_cvx}}
  $(\Rightarrow)$ Assume that there exists $r \in \Rank_{\mathcal{C}_k}$
  such that $L_{t+1}(r) =0$ and 
  let $\{h_i \}_{i=1}^{t+1}$  be the sequence of classifiers 
  defined $\forall i\leq t+1$ by $h_i(x)= \indic{ r(x,X_{(i)} ) \geq 0  }$.
  First, we know  from the definition of $\Rank_{\mathcal{C}_k}$ that
  all the classifiers are of the form
  $h_i(x) = \sum_{m=1}^k \indic{  l_{i,m} \leq  x \leq u_{i,m} }$.
  Second, since $L_{n+1}(r) =0$, it directly follows that
  $\forall (i,j) \in \{1, \dots, t+1 \}^2$, $h_i(X_{(j)})= \indic{j \geq i}$.
  Finally, as $r$ is transitive and $\forall i \leq t$,
  $r(X_{(i+1)}, X_{(i)})=1$, we have that 
  $h_1 \geq h_2 \geq \dots \geq h_{t+1}$.\\

  \noindent $(\Leftarrow)$ Assume now that there exists 
  a sequence of classifiers $\{ h_i \}_{i=1}^{t+1}$
  of the form $h_i(x)= \sum_{m=1}^k \indic{ l_{i,m} \leq x \leq u_{i,m} } $
   satisfying:
  (i) $h_1 \geq h_2 \geq  \dots \geq h_{t+1}$ and
  (ii) $\forall (i,j) \in \{1 \ldots t+1 \}^2$,  $h_i(X_{(j)})= \indic{j \geq i}$.
  Define the step function $f_{\text{step}}(x)= \sum_{i=1}^{t+1} h_i(x)$ and
  observe that $L_{t+1}(r_{f_{\text{step}}}) =0$
  since
  $\forall (j,k) \in \{1 \dots t+1 \}^2$,
  \begin{align*}
    r_{f_{\text{step}}}(X_{(j)}, X_{(k)})
    = \text{sgn}\left( \sum_{i=1}^{t+1}  \indic{ j \geq i } -\indic{ k \geq i } \right)
     = \sgn{j -k}
     = r_f(X_{(j)}, X_{(k)}),
  \end{align*}
  To prove the result, we will simply construct a continuous approximation
  of the function $f_{\text{step}}$ which (i) induces a ranking rule which
  perfectly ranks the sample and (ii) admits level sets which are unions of
  at most $k$ convex set.
  Let $\hat{f}_{\epsilon}:\X \to \R$ be the continuous function defined by
  $\hat{f}_{\epsilon}(x) = 
  \sum_{i=1}^{t+1} \sum_{m=1}^k \hat{\mathds{1}}_{\epsilon, l_{i,m}, u_{i,m}}(x) $
   where $\forall l \leq u$,
  \begin{equation*}
    \hat{\mathds{1}}_{\epsilon, l, u}(x)=
    \begin{cases}
      1 \ \   & \   \ \ \text{if}\  \ \ x \in [l,u]\\
	 1 - \frac{l-x}{\epsilon}  & \ \ \   \text{if}\ \   \ x \in [l-\epsilon,l[ \\
	1 - \frac{x-u}{\epsilon} & \ \ \   \text{if}\ \  \ x \in ]u,u + \epsilon] \\
	0 & \ \ \ \text{otherwise}.
    \end{cases}
  \end{equation*}
  Observe now that $\forall \epsilon<  \min \{
  \abs{ x_1 -x_2}: x_1 \neq x_2 \in \{X_{(i)} \}_{i=1}^{t+1}
  \cup \{ l_{i,m}\}_{i = 1 \ldots t+1}^{m=1 \ldots k} \cup
  \{ u_{i,m}\}_{i = 1 \ldots t+1}^{m=1 \ldots k}   \}$ and 
  $\forall i\leq  t$, we have that
  $\hat{f}_{\epsilon}(X_{(i)}) = f_{\text{step}}(X_{(i)})$.
  Hence $L_{t+1}(r_{\hat{f}_{\epsilon}})=L_{t+1}(r_{f_{\text{step}}})=0$
  which proves (i).
  Moreover, as for any 
  $\epsilon <
  \min\{ \abs{x_1 - x_2}:x_1 \neq x_2 \in \{l_{i,m} \}_{i = 1 \ldots t+1}^{m=1 \ldots k}
  \cup \{ u_{i,m }\}_{i = 1 \ldots t+1}^{m=1 \ldots k} \}/2$,
  the level sets of $\hat{f}_{\epsilon}$ are by construction
  a union of at most $k$ segments (convex sets) and (ii) holds true.
  We  then deduce from (i) and (ii) that for $\epsilon$ small enough
  there exists
  $r = r_{\hat{f}_{\epsilon}} \in \Rank_{\mathcal{C}_k  } $
  such that $L_{t+1}(r)=0$.
  \hfill\(\Box\)

 ~\\
\noindent The next lemma will be used in the proof of Proposition \ref{prop:lp_cvx}.

\begin{lemma}
\label{lem:epsiball}
  Let $\X \subset \R^d$ be any compact and convex set
  and define for any $\epsilon>0$ the $\epsilon$-ball of $\X$
  as $B(\X, \epsilon)=\{ x \in \R^d :
  \min_{x' \in \X} \lVert x-x'  \rVert_2 \leq \epsilon \}$.
  Then, for any $\epsilon>0$,
  the $\epsilon$-ball of $\X$
  is also a convex set.
\end{lemma}

\begin{proof}
  Pick any $(b_1,b_2) \in B(\X, \epsilon)^2$.
  By definition of $B(\X, \epsilon)$,
  we know that there exists $(x_1, \epsilon_1) \in \X \times \R^d$
  such that $b_1 = x_1 + \epsilon_1$
  and $\rVert \epsilon_1 \lVert_2 \leq \epsilon$
  (resp.~$b_2 = x_2 + \epsilon_2$ where $x_2 \in \X$ and
  $\rVert \epsilon_2 \lVert_2 \leq \epsilon$).
  Then, by convexity of $\X$, we have  that $\forall \lambda \in [0,1]$,
  \[
    (1-\lambda) b_1 + \lambda b_2 ~=~
    \underbrace{\lambda x_1 +(1-\lambda) x_2}_{\in \X}
    ~+~ \underbrace{\lambda \epsilon_1 +(1-\lambda) \epsilon_2}_{ \norm{\cdot }_2\leq \epsilon}.
  \]
  Hence $(1-\lambda) b_1 + \lambda b_2  \in B(\X, \epsilon)$ and we deduce that
  $B(\X, \epsilon)$ is a convex set.
\end{proof}

\noindent Equipped with Lemma \ref{lem:epsiball}, 
we may now prove Proposition \ref{prop:lp_cvx}.

~\\
{\bf Proof of Proposition \ref{prop:lp_cvx}}
  $(\Rightarrow)$ Assume that there exists $r \in \Rank_{\mathcal{C}_1}$ such that
  $L_{t+1}(r)=0 $.
  Observe first that since
  $\forall j \neq k \leq t+1$,
  $r(X_{(j)}, X_{(k)}) =2\indic{ j > k }  -1$,
  we have that $\forall k \leq t$,
  \begin{enumerate}
   \item[i)] $\{X_{(i)} \}_{i=k+1}^{t+1} \in \{x \in \X: r(x, X_{(k+1)})\geq 0 \}$;
   \item[ii)] $X_{(k)} \notin \{x \in \X: r(x, X_{(k+1)})\geq 0 \}$.
  \end{enumerate}
  Now, pick any $k \in \{1, \dots, t \}$ and notice that,
  by definition of $\Rank_{\mathcal{C}_1}$,
  the level set  $\{x \in \X: r(x,X_{(k+1)}) \geq 0 \}$ is also a convex set.
  However, since $\textsc{CH}\{X_{(i)} \}_{i=k+1}^{t+1}$
  is the smallest convex set which contains
  $\{X_{(i)} \}_{i=k+1}^{t+1}$, we deduce from (i)~that
  $\textsc{CH}\{X_{(i)} \}_{i=k+1}^{t+1} \subseteq \{x \in \X: r(x, X_{(k+1)})\geq 0 \}$.
  Therefore, combining the previous statement with (ii)~gives
  that $\forall k \leq t$,
  \begin{equation*}
   X_{(k)} \notin \textsc{CH}\{X_{(i)} \}_{i=k+1}^{t+1}.
  \end{equation*}
  Finally, using the vertex representation of convex hulls we know that
  $X_{(t+1-i)} \notin \textsc{CH}\{X_{(i)} \}_{i=t+2-i}^{t+1}$
  if and only if there does not any $\lambda = (\lambda_1, \dots, \lambda_{i} )\in \R^i$
  such that (i) $\sum_{j=1}^j \lambda_j \cdot X_{(t+2-j)} = X_{(k)}$,
  (ii) $\sum_{j=1}^i \lambda_j = 1$ and (iii) $\lambda_j\geq 0$, $j=1\dots i$
  and putting those constraints (i), (ii) and (iii) into matrix form
  gives the result.\\

  \noindent $(\Leftarrow)$ Assume now that the cascade of polyhedrons is empty.
  First, we point out that it can easily check
  by reproducing the same steps as in the first part
  of the proof that
  $\forall k \leq t$, $X_{(k)} \notin \textsc{CH}\{X_{(i)} \}_{i=k+1}^{t+1}$,
  which implies that
  \begin{align}
  \label{eq:final}
    \textsc{CH}\{X_{(t+1)} \} \subset \textsc{CH}\{X_{(i)} \}_{i=t}^{t+1}
    \subset \dots  \subset \textsc{CH}\{X_{(i)} \}_{i=1}^{t+1}.
  \end{align}
  Now define the step function
  $f_{\text{step}}:x \in \X \mapsto 
  \sum_{i=1}^{t+1} \mathbb{I} \{ x \in \textsc{CH}\{X_{(j)} \}_{j=i}^{t+1} \}$
  and observe that $L_{t+1}(r_{f_{\text{step}}}) = 0$  by (\ref{eq:final}).
  To prove the result,
  we will simply construct a continuous approximation
  of the function $f_{\text{step}}$ which (i) induces a ranking rule that
  perfectly ranks the sample and (ii) has convex level sets.
  Let $\hat{f}_{\epsilon}$ be the continuous
  function defined by $\hat{f}_{\epsilon}(x) = \sum_{i=1}^{t+1} \phi_{i, \epsilon}(x) $
  where $\forall i \leq t+1 $,
    \begin{equation*}
    \phi_{i, \epsilon}(x)=
    \begin{cases}
  1 - \frac{d(x,B( \textsc{CH}\{X_{(j)} \}_{j=i}^{t+1} , 2(t+1-i) \epsilon))}{\epsilon}
  & \text{if} \ d(x,B( \textsc{CH}\{X_{(j)} \}_{j=i}^{t+1} , 2(t+1-i) \epsilon) \leq \epsilon \\
        0 \ \   \   \ \ & \text{otherwise.}
    \end{cases}
  \end{equation*}
  \sloppy
  Observe now that for any $\epsilon < \min_{i=1 \ldots t} d(X_{(i)},
  \textsc{CH}\{X_{(j)} \}_{j=i+1}^{t+1} )  / (2t+2)$,  
  we have that
  $\forall i \leq t+1$,
  $\hat{f}_{\epsilon}(X_{(i)}) = f_{\text{step}}(X_{(i)})$. Hence
  $L_{t+1}(r_{\hat{f}_{\epsilon}})= L_{t+1}(r_f)= 0$,
  which proves (i).
  Moreover, we know from Lemma \ref{lem:epsiball}
  that for any $\epsilon < \min_{i=1 \ldots t} d(X_{(i)},
  \textsc{CH}\{X_{(j)} \}_{j=i+1}^{t+1} )  / (2t+2)$
  and any $x \in \X$, the level set
  $\{x' \in \X: \hat{f}_{\epsilon}(x') \geq \hat{f}_{\epsilon}(x) \}$ is a convex set.
  Hence (ii) holds true and
  we then deduce from (i) and (ii)
  that for $\epsilon$ small enough
  there exists $r = r_{\hat{f}_{\epsilon}} \in \Rank_{ \mathcal{C}_1}$
  such that $L_{t+1}(r_{\hat{f}_{\epsilon}})=0$ and the proof is complete.
\hfill\(\Box\)



\end{document}